\documentclass[letterpaper]{article} 
\usepackage{aaai23}  
\usepackage{times}  
\usepackage{helvet}  
\usepackage{courier}  
\usepackage[hyphens]{url}  
\usepackage{graphicx} 
\usepackage{natbib}  
\usepackage{caption} 
\DeclareCaptionStyle{ruled}{labelfont=normalfont,labelsep=colon,strut=off} 
\usepackage{algorithm}
\usepackage[noend]{algorithmic}
\usepackage{eqparbox}

\usepackage{newfloat}
\usepackage{listings}
\DeclareCaptionStyle{ruled}{labelfont=normalfont,labelsep=colon,strut=off} 
\floatstyle{ruled}
\newfloat{listing}{tb}{lst}{}
\floatname{listing}{Listing}
\nocopyright

\title{Eliminating The Impossible, Whatever Remains Must Be True \\[5pt] \large{On Extracting and Applying Background Knowledge In The Context Of Formal Explanations}}

\author {
	Jinqiang Yu\textsuperscript{\rm 1,\rm 4},
  Alexey Ignatiev\textsuperscript{\rm 1},
  Peter J. Stuckey\textsuperscript{\rm 1,\rm 4},
	Nina Narodytska\textsuperscript{\rm 2},
  Joao Marques-Silva\textsuperscript{\rm 3}
}
\affiliations {
  \textsuperscript{\rm 1} Monash University, Melbourne, Australia \\
  \textsuperscript{\rm 2} VMWare Research, Palo Alto, USA \\
  \textsuperscript{\rm 3} IRIT, CNRS, Toulouse, France \\
  \textsuperscript{\rm 4} ARC Training Centre in OPTIMA, Melbourne, Australia \\
  jinqiang.yu@monash.edu, alexey.ignatiev@monash.edu, peter.stuckey@monash.edu, \\ nnarodytska@vmware.com, joao.marques-silva@irit.fr
}

\usepackage{bibentry}

\usepackage{adjustbox}
\usepackage{appendix}
\usepackage{amsmath}
\usepackage{amssymb}
\usepackage{amsfonts}
\usepackage{array}
\usepackage[english]{babel}
\usepackage{booktabs,siunitx}
\usepackage{comment}
\usepackage{enumitem}
\usepackage{graphicx}
\usepackage[utf8]{inputenc}
\usepackage{multirow}
\usepackage{setspace}
\usepackage{stmaryrd}
\usepackage{subcaption} 
\usepackage{tablefootnote}
\usepackage{tikz}
\usetikzlibrary{arrows,decorations,fit,positioning,shapes}
\usepackage{url}
\usepackage{empheq,fancybox}
\usepackage{xcolor}
\usepackage{xspace}
\usepackage{caption}
\usepackage{setspace}
\usepackage{subcaption}
\usepackage{tabularx}
\usepackage{amsthm}

\newtheorem{proposition}{Proposition}

\newtheorem{remark}{Remark}
\newtheorem{example}{Example}

\newcommand{\fml}[1]{{\mathcal{#1}}}

\newcommand{\tn}[1]{\textnormal{#1}}
\newcommand{\mbf}[1]{\ensuremath\mathbf{#1}}
\newcommand{\mbb}[1]{\ensuremath\mathbb{#1}}

\DeclareMathOperator*{\limply}{\rightarrow}

\def\squareforqed{\hbox{\rlap{$\sqcap$}$\sqcup$}}
\def\qed{\ifmmode\squareforqed\else{\unskip\nobreak\hfil
\penalty50\hskip1em\null\nobreak\hfil\squareforqed
\parfillskip=0pt\finalhyphendemerits=0\endgraf}\fi}

\newcommand{\pnote}[1]{$\llbracket$\textcolor{tred3}{peter}:~~{\sl\textcolor{tdgray1}{#1}}$\rrbracket$}

\newcommand{\jynote}[1]{$\llbracket$\textcolor{tred3}{jinqiang}:~~{\sl\textcolor{tdgray1}{#1}}$\rrbracket$}

\addto\extrasenglish{%

}

\graphicspath{{plots/}}
\DeclareGraphicsExtensions{.pdf}
\newcommand{\ignore}[1]{}
\newcommand{\frmeq}[1]{\begin{empheq}[box={\fboxsep=1.5pt\doublebox}]{flalign*}#1\end{empheq}}
\DeclareOldFontCommand{\sc}{\normalfont\scshape}{\@nomath\sc}
\newcommand{\myparagraph}[1]{\textbf{#1}}
\newcommand{\Status}{Status}
\begin{document}

\maketitle
\begin{abstract}
  The rise of AI methods to make predictions and decisions has led to
  a pressing need for more explainable artificial intelligence (XAI)
  methods.
  One common approach for XAI is to produce a post-hoc explanation,
  explaining why a black box ML model made a certain prediction.
  Formal approaches to post-hoc explanations provide succinct reasons
  for \emph{why} a prediction was made, as well as \emph{why not}
  another prediction was made.
  But these approaches assume that features are independent and
  uniformly distributed.
  While this means that ``why'' explanations are correct, they may be
  longer than required.
  It also means the ``why not'' explanations may be suspect as the
  counterexamples they rely on may not be meaningful.
  In this paper, we show how one can apply background knowledge to
  give more succinct ``why'' formal explanations, that are presumably
  easier to interpret by humans, and give more accurate ``why not''
  explanations.
  In addition, we show how to use existing rule induction
  techniques to efficiently extract background information from a
  dataset, and also how to report which background information was
  used to make an explanation, allowing a human to examine it if they
  doubt the correctness of the explanation.
\end{abstract}

%
\definecolor{tyellow1}{HTML}{FCE94F}
\definecolor{tyellow2}{HTML}{EDD400}
\definecolor{tyellow3}{HTML}{C4A000}
%
\definecolor{torange1}{HTML}{FCAF3E}
\definecolor{torange2}{HTML}{F57900}
\definecolor{torange3}{HTML}{C35C00}
%
\definecolor{tbrown1}{HTML}{E9B96E}
\definecolor{tbrown2}{HTML}{C17D11}
\definecolor{tbrown3}{HTML}{8F5902}
%
\definecolor{tgreen1}{HTML}{8AE234}
\definecolor{tgreen2}{HTML}{73D216}
\definecolor{tgreen3}{HTML}{4E9A06}
%
\definecolor{tblue1}{HTML}{729FCF}
\definecolor{tblue2}{HTML}{3465A4}
\definecolor{tblue3}{HTML}{204A87}
%
\definecolor{tpurple1}{HTML}{AD7FA8}
\definecolor{tpurple2}{HTML}{75507B}
\definecolor{tpurple3}{HTML}{5C3566}
%
\definecolor{tred1}{HTML}{EF2929}
\definecolor{tred2}{HTML}{CC0000}
\definecolor{tred3}{HTML}{A40000}
%
\definecolor{tlgray1}{HTML}{EEEEEC}
\definecolor{tlgray2}{HTML}{D3D7CF}
\definecolor{tlgray3}{HTML}{BABDB6}
%
\definecolor{tdgray1}{HTML}{888A85}
\definecolor{tdgray2}{HTML}{555753}
\definecolor{tdgray3}{HTML}{2E3436}

\section{Introduction} \label{sec:intro}

Recent years have witnessed rapid advances in Artificial Intelligence
(AI) and Machine Learning (ML) algorithms revolutionizing all aspects
of human lives~\cite{bengio-nature15,taward18}. 
An ever growing range of practical applications of AI and ML, on the
one hand, and a number of critical issues observed in modern AI
systems (e.g. decision bias~\cite{propublica16} and
brittleness~\cite{szegedy-iclr14}), 
on the other
hand, gave rise to the quickly advancing area of theory and practice
of Explainable AI (XAI).

Several major approaches to XAI have been proposed in the recent past.
Besides tackling XAI through computing \emph{interpretable} ML models
directly~\cite{rudin-natmi19}, 
or through the use of
interpretable models for approximating complex \emph{black-box} ML
models~\cite{guestrin-kdd16}, the most prominent approach to XAI is to
compute \emph{post-hoc explanations} to ML predictions on
demand~\cite{lundberg-nips17,guestrin-aaai18}.
%
Prior work distinguishes post-hoc (\emph{abductive}) explanations answering
a \emph{``why?''} question and (\emph{contrastive}) explanations targeting a \emph{``why
  not?''} question~\cite{miller-aij19}.
%
Heuristic approaches to post-hoc explainability
are known to suffer from a number of fundamental explanation quality
issues~\cite{nsmims-sat19,inms-corr19,lukasiewicz-corr19,ignatiev-ijcai20},
including the existence of out-of-distribution
attacks~\cite{lakkaraju-aies20a}.
A promising alternative is formal
explainability where explanations are computed as prime implicants of
the decision function associated with ML
predictions~\cite{darwiche-ijcai18}. Formal explanations have also
been related with abductive reasoning~\cite{inms-aaai19,inms-nips19}.

Although provably correct and minimal, formal explanations have a few
limitations.
%
In order to provide provable correctness
guarantees that a subset of features is sufficient for an ML
prediction, formal approaches have to take into account the complete
feature space assuming that the features are independent and uniformly
distributed~\cite{kutyniok-jair21}.
This makes a formal reasoner check all the combinations of feature
values, including those that realistically can \emph{never appear} in
practice.
This issue is caused by the inability of modern (both formal and
heuristic\footnote{The lack of background knowledge support pertains
to heuristic approaches as well by making them
error-prone~\cite{lakkaraju-aies20a,snimmsv-aaai22}.}) explanation approaches to
account for background knowledge associated with the problem domain of
the target dataset.
It results in formal explanations being unnecessarily long, which
makes them hard for a human decision maker to interpret.
%

Motivated by this limitation, our work focuses on computing both
abductive and contrastive formal explanations making use of background
knowledge, and makes the following contributions:
%
First, given a training dataset, an efficient generic approach to
extracting background knowledge in the form of highly accurate
\emph{if-then} rules is proposed.
Following recent work on using constraints in compilation-based formal
explainability~\cite{rubin-aaai22}, accurate background knowledge is
argued to be the key to good quality explanations.
The approach builds on a recent formal method for learning decision
sets~\cite{ilsms-aaai21} and is able to extract reasonably short rules
representing relations between various features of the target dataset.
Also, as our approach is designed to enumerate 100\% accurate rules,
its performance is shown to be on par with a modern implementation of
the well-known Apriori and Eclat association rule mining
algorithms~\cite{as-vlbd94,zaki-kdd97}.
Second, a novel approach to computing formal explanations taking into
account background knowledge is proposed, \emph{independent} of the nature of
the background knowledge; the only requirement imposed is that the
knowledge must be represented as a conjunction of constraints.
Third, we prove theoretically that the use of background knowledge
positively affects the quality of both abductive and contrastive
explanations, thus, helping to build trust in the underlying AI
systems.
Fourth, we develop an effective way to discover which background
knowledge rules are used in extracting an explanation.
This enables a human decision maker to examine whether or not the
rules used are meaningful, which further facilitates human trust in
the explanations computed.
Fifth and finally, motivated by the results
of~\cite{inms-corr19,ignatiev-ijcai20}, we argue that background
knowledge helps one assess the correctness of heuristic ML
explainers~\cite{guestrin-kdd16,lundberg-nips17,guestrin-aaai18} more
accurately since it blocks impossible combinations of feature values.
Namely, we show that the estimated correctness of SHAP, LIME, and
Anchor may improve significantly when background knowledge is
available.


\section{Preliminaries} \label{sec:prelim}

\subsubsection{SAT, MaxSAT, and SMT.}

Definitions standard in \emph{propositional satisfiability} (SAT) and
\emph{maximum satisfiability} (MaxSAT) solving are
assumed~\cite{sat-handbook21}.
SAT and MaxSAT formulas are assumed to be propositional.
A propositional formula $\varphi$ is considered to be in
\emph{conjunctive normal form} (CNF) if it is a conjunction (logical
\emph{``and''}) of clauses, where a \emph{clause} is a disjunction
(logical \emph{``or''}) of literals, and a \emph{literal} is either a
Boolean variable $b$ or its \emph{negation} $\neg b$.
Whenever convenient, a clause is treated as a set of literals.
A \emph{truth assignment} $\mu$ is a mapping from the set of variables
in $\varphi$ to $\{0, 1\}$.
A clause is \emph{satisfied} by truth assignment $\mu$ if one of its
literals is assigned value $1$ by $\mu$; otherwise, the clause is said
to be \emph{falsified}.
If all clauses of formula $\varphi$ are satisfied by assignment $\mu$
then $\mu$ also satisfies $\varphi$; otherwise, $\varphi$ is falsified by
$\mu$.
%
%
A formula $\varphi$ is said to be \emph{satisfiable} if there is an
assignment $\mu$ that satisfies $\varphi$; otherwise, $\varphi$ is
\emph{unsatisfiable}.
%

In the context of unsatisfiable formulas, the maximum satisfiability
problem is to find a truth assignment that maximizes the number of
satisfied clauses.
%
%
Hereinafter, we will make use of a variant of MaxSAT called Partial
(Unweighted) MaxSAT~\cite[Chapters~23~and~24]{sat-handbook21}.
The formula $\varphi$ in Partial (Unweighted) MaxSAT is a conjunction
of \emph{hard} clauses $\fml{H}$, which must be satisfied, and
\emph{soft} clauses $\fml{S}$, which represent a preference to satisfy
them, i.e.\ $\varphi = \fml{H} \wedge \fml{S}$.
The Partial Unweighted MaxSAT problem aims at finding a truth
assignment that satisfies all the hard clauses while maximizing the
total number of satisfied soft clauses.
%
%
%


Note that we consider a family of ML classifiers such that their
decision making process can be represented logically as a
propositional formula.
%
This is needed for applying formal reasoning about ML model behavior,
as well as for representing background knowledge extracted.
Finally, a logical representation of boosted tree models will require
us to apply an extension of propositional logic to decidable fragments
of first-order logic (FOL).
Namely, we will assume the use of \emph{satisfiability modulo
theories} (SMT) in the theory of linear arithmetic over reals, i.e.\
the concept of a clause will be lifted to \emph{linear constraints}
over real variables.
Optimization problems for SMT
can be defined analogously to MaxSAT.

\begin{figure*}[t!]
	\begin{subfigure}[b]{\textwidth}
		\centering
		\begin{minipage}{0.72\textwidth}
        \frmeq{
	\begin{array}{lllcl}
		\tn{R$_{0}$:} & \tn{IF}     & \text{Education = Dropout} & \tn{THEN} & \text{Target $< 50$k} \\
		\tn{R$_{1}$:} & \tn{ELSE IF} & \text{Occupation = Service}& \tn{THEN} & \text{Target $< 50$k} \\
		\tn{R$_{2}$:} & \tn{ELSE IF} & \text{\Status{} = Married}  \land \text{ Relationship = Husband}& \tn{THEN} & \text{Target $\geq 50$k} \\
		\tn{R$_{3}$:} & \tn{ELSE IF} & \text{\Status{} = Married}  \land \text{ Relationship = Wife}& \tn{THEN} & \text{Target $\geq 50$k} \\
            \tn{R$_{\tn{\textsc{def}}}$:} & \tn{ELSE} &  & \tn{THEN} &  \text{Target $< 50$k} \\
          \end{array}
        }
		\end{minipage}
	\caption{Decision list.}
    \label{fig:dl}
	\end{subfigure}
\begin{subfigure}[b]{\textwidth}
  \begin{center}
    \scalebox{0.98}{
      \begin{minipage}{0.99\textwidth}
%
%
%

\tikzstyle{box} = [draw=black!90, thick, rectangle, rounded corners,
                     inner sep=10pt, inner ysep=20pt, dotted
                  ]
\tikzstyle{title} = [draw=black!90, fill=black!5, semithick, top color=white,
                     bottom color = black!5, text=black!90, rectangle,
                     font=\small, inner sep=2pt, minimum height=1.3em,
                     top color=tyellow2!27, bottom color=tyellow2!27
                    ]
\tikzstyle{feature} = [rectangle,font=\scriptsize,rounded corners=1mm,thick,%
                       draw=black!80, top color=tblue2!20,bottom color=tblue2!25,%
                       draw, minimum height=1.1em, text centered,%
                       inner sep=2pt%
                      ]
\tikzstyle{pscore} = [rectangle,font=\scriptsize,rounded corners=1mm,thick,%
                     draw=black!80, top color=tgreen3!20,bottom color=tgreen3!27,%
                     draw, minimum height=1.1em, text centered,%
                     inner sep=2pt%
                    ]
\tikzstyle{nscore} = [rectangle,font=\scriptsize,rounded corners=1mm,thick,%
                     draw=black!80, top color=tred2!20,bottom color=tred2!25,%
                     draw, minimum height=1.1em, text centered,%
                     inner sep=2pt%
                    ]

\begin{adjustbox}{center}
\setlength{\tabcolsep}{3pt}
\def\arraystretch{3}
\begin{tabular}{ccc}
    \begin{tikzpicture}[node distance = 4.0em, auto]
        \node [box] (box) {%
        \begin{minipage}[t!]{0.33\textwidth}
            \vspace{0.9cm}\hspace{1.5cm}
        \end{minipage}
        };
        \node[title] at (box.north) {$\text{T}_\text{1}$ ($\geq 50$k)};

        \node [feature] (feat1) at (-0.38, 0.57) {Marital Status $=$ Married?};
        \node [feature, below left  = 0.8em and -3.5em of feat1] (feat2) {Education $=$ Dropout?};
        \node [feature, below right = 0.8em and -3.5em of feat1] (feat3) {Relationship $=$ Not-in-family?};

        \node [nscore, below left  = 0.8em and -2.3em of feat2] (pos1) {-0.2192};
        \node [pscore, below right  = 0.8em and -2.5em of feat2] (neg1) {0.1063};
        \node [nscore, below left = 0.8em and -2.4em of feat3] (pos2) {-0.1561};
        \node [nscore, below right = 0.8em and -2.4em of feat3] (neg2) {-0.3850};

        \draw [->,thick,black!80] (feat1) to[] node[above, pos=1.2, font=\scriptsize] {yes} (feat2.north);
        \draw [->,thick,black!80] (feat1) to[] node[above, pos=1.1, font=\scriptsize] { no} (feat3.north);
        \draw [->,thick,black!80] (feat2) to[] node[above, pos=1.2, font=\scriptsize] {yes} (pos1.north);
        \draw [->,thick,black!80] (feat2) to[] node[above, pos=1.1, font=\scriptsize] { no} (neg1.north);
        \draw [->,thick,black!80] (feat3) to[] node[above, pos=1.2, font=\scriptsize] {yes} (pos2.north);
        \draw [->,thick,black!80] (feat3) to[] node[above, pos=1.1, font=\scriptsize] { no} (neg2.north);
    \end{tikzpicture}
    &
    \begin{tikzpicture}[node distance = 4.0em, auto]
        \node [box] (box) {%
        \begin{minipage}[t!]{0.28\textwidth}
            \vspace{0.9cm}\hspace{1.5cm}
        \end{minipage}
        };
        \node[title] at (box.north) {$\text{T}_\text{2}$ ($\geq 50$k)};

        \node [feature] (feat1) at (0.23, 0.57) {Marital Status $=$ Married?};
        \node [feature, below left  = 0.8em and -3.1em of feat1] (feat2) {Occupation $=$ Service?};
        \node [feature, below right = 0.8em and -3.1em of feat1] (feat3) {Hours/w $>$ 45?};

        \node [nscore, below left  = 0.8em and -2.3em of feat2] (pos1) {-0.2231};
        \node [pscore, below right  = 0.8em and -2.5em of feat2] (neg1) {0.0707};
        \node [nscore, below left = 0.8em and -1.4em of feat3] (pos2) {-0.0080};
        \node [nscore, below right = 0.8em and -1.4em of feat3] (neg2) {-0.2549};

        \draw [->,thick,black!80] (feat1) to[] node[above, pos=1.2, font=\scriptsize] {yes} (feat2.north);
        \draw [->,thick,black!80] (feat1) to[] node[above, pos=1.1, font=\scriptsize] { no} (feat3.north);
        \draw [->,thick,black!80] (feat2) to[] node[above, pos=1.2, font=\scriptsize] {yes} (pos1.north);
        \draw [->,thick,black!80] (feat2) to[] node[above, pos=1.1, font=\scriptsize] { no} (neg1.north);
        \draw [->,thick,black!80] (feat3) to[] node[above, pos=1.2, font=\scriptsize] {yes} (pos2.north);
        \draw [->,thick,black!80] (feat3) to[] node[above, pos=1.1, font=\scriptsize] { no} (neg2.north);
    \end{tikzpicture}
    &
     \begin{tikzpicture}[node distance = 4.0em, auto]
        \node [box] (box) {%
        \begin{minipage}[t!]{0.276\textwidth}
            \vspace{0.9cm}\hspace{1.5cm}
        \end{minipage}
        };
        \node[title] at (box.north) {$\text{T}_\text{3}$ ($\geq 50$k)};

        \node [feature] (feat1) at (-0.1, 0.57) {Relationship $=$ Own-child?};
        \node [feature, below left  = 0.8em and -3.5em of feat1] (feat2) {Education $=$ Master?};
        \node [feature, below right = 0.8em and -3.5em of feat1] (feat3) {Education $=$ Dropout?};

        \node [pscore, below left  = 0.8em and -2.3em of feat2] (pos1) {0.1186};
        \node [nscore, below right  = 0.8em and -2.5em of feat2] (neg1) {-0.3483};
        \node [nscore, below left = 0.8em and -2.4em of feat3] (pos2) {-0.2844};
        \node [nscore, below right = 0.8em and -2.4em of feat3] (neg2) {-0.0128};

        \draw [->,thick,black!80] (feat1) to[] node[above, pos=1.2, font=\scriptsize] {yes} (feat2.north);
        \draw [->,thick,black!80] (feat1) to[] node[above, pos=1.1, font=\scriptsize] { no} (feat3.north);
        \draw [->,thick,black!80] (feat2) to[] node[above, pos=1.2, font=\scriptsize] {yes} (pos1.north);
        \draw [->,thick,black!80] (feat2) to[] node[above, pos=1.1, font=\scriptsize] { no} (neg1.north);
        \draw [->,thick,black!80] (feat3) to[] node[above, pos=1.2, font=\scriptsize] {yes} (pos2.north);
        \draw [->,thick,black!80] (feat3) to[] node[above, pos=1.1, font=\scriptsize] { no} (neg2.north);
    \end{tikzpicture}
    \\
   \end{tabular}
\end{adjustbox}

      \end{minipage} }
			\caption{Boosted tree~\cite{guestrin-kdd16a} consisting of 3 trees
				with the depth of each tree at most 2.}
    \label{fig:bt}
  \end{center}
\end{subfigure}
\caption{Example DL and BT models trained
      on the well-known \emph{adult} classification dataset.}
  \label{fig:models}
\end{figure*}
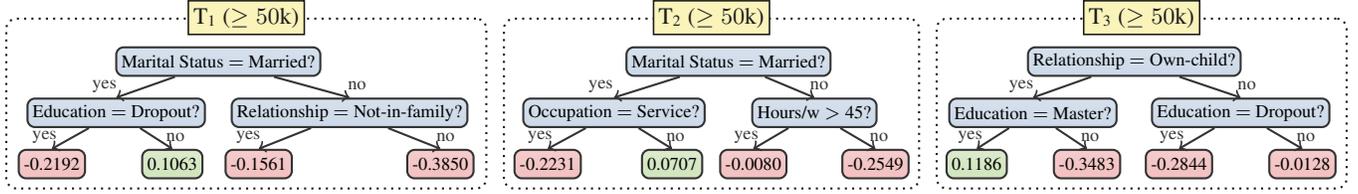

\subsubsection{Classification Problems.}

Classification problems consider a set of classes $\fml{K} = \{c_1,
c_2, \ldots, c_k\}$, and a set of features $\fml{F}=\{1, \ldots, m\}$.
The value of each feature $i \in \fml{F}$ is taken from a domain
$\fml{D}_i$, which can be integer, real-valued or Boolean.
%
%
Therefore, the complete feature space is defined as
$\mathbb{F}\triangleq\prod_{i=1}^{m}\fml{D}_i$.
A concrete point in feature space is represented by
$\mbf{v}=(v_1,\ldots,v_m)\in\mbb{F}$, where each $v_i \in \mbf{v}$ is
a constant taken by feature $i\in\fml{F}$.
An \emph{instance} or \emph{example} is denoted by a specific point
$\mbf{v} \in\mbb{F}$ in feature space and its corresponding class $c
\in \fml{K}$, i.e. a pair ($\mbf{v}, c$) represents an instance.
Moreover, the notation $\mbf{x} = (x_1,\ldots,x_m)$ denotes an
arbitrary point in feature space, where each $x_i\in\mbf{x}$ is a
variable taking values from its corresponding domain $\fml{D}_i$ and
representing feature $i\in\fml{F}$.

A classifier defines a classification function $\tau: \mathbb{F}
\limply \fml{K}$.
Whenever convenient, a classification function $\tau$ and an
associated class $c$ are represented by a \emph{decision predicate}
$\tau_c: \mathbb{F} \to \{0, 1\}$.
A decision predicate $\tau_c$ is given a specific class $c \in
\fml{K}$, such that
$\forall(\mbf{x}\in\mbb{F}).\tau_c(\mbf{x})\leftrightarrow(\tau(\mbf{x})=c)$.
There are many ways to learn classifiers for a given dataset.
%
In this paper, we consider: \emph{decision
lists} (DLs)~\cite{rivest-ml87,clark-ml89}, 
\emph{boosted trees} (BTs)~\cite{friedman-tas01,guestrin-kdd16a}, and
\emph{binarized neural networks} (BNNs)~\cite{hcseyb-neurips16}.
%

\begin{table}[t!]
	\caption{Several examples extracted from \emph{adult} dataset.}
	\begin{adjustbox}{center}
	\label{tab:inst}
	\scalebox{0.69}{
		\setlength{\tabcolsep}{3pt}
		\begin{tabular}{ccccccc}\toprule
			\textbf{Education} & \textbf{\Status} & \textbf{Occupation} &
			\textbf{Relationship} & \textbf{Sex} & \textbf{Hours/w} & \textbf{Target} \\ \midrule
			HighSchool & Married & Sales &
			Husband & Male & $ 40$ to $ 45$ & $\geq 50$k\\ \midrule
			Bachelors & Married & Sales &
			Wife & Female & $\leq 40$ & $\geq 50$k \\ \midrule
			\label{inst:3}Masters & Married & Professional &
			Wife & Female & $\geq 45$ & $\geq 50$k \\ \midrule
			Masters & Married & Professional &
			Wife & Female & $\leq 40$ &  $\geq 50$k \\ \midrule
			Dropout & Separated & Service &
			Not-in-family &	Male & $\leq 40$ &  $< 50$k\\ \midrule
			Dropout & Never-Married & Blue-Collar &
			Unmarried & Male & $\geq 45$ & $\geq 50$k \\ \bottomrule
		\end{tabular}
	}
	\end{adjustbox}
\end{table}

\begin{example}\label{ex:cls}
	Consider the data shown in Table~\ref{tab:inst}.
	It represents a snapshot of instances taken from a simplified
	version\footnote{For simplicity, the running example used throughout the text
		will correspond to a \emph{simplified} version of the
		\emph{adult} dataset~\cite{kohavi-kdd96}, where some
	of the features are dropped. Note that the experimental
	results shown below deal with the \emph{original} datasets.} of the
	\emph{adult} dataset~\cite{kohavi-kdd96}.
	Figure~\ref{fig:models} illustrates DL and BT models trained for this
	dataset.
	Observe that for instance
	$\mbf{v} = \{\text{Education}=\text{HighSchool},$
	$\text{\Status}=\text{Married}$, $\text{Occupation}=\text{Sales}$,
	Relation\-ship~$=\text{Husband}, \text{Sex}=\text{Male},$ $\text{Hours/w}=\text{40 to 45}\}$
	from Table~\ref{tab:inst}, rule
	$\text{R}_\text{2}$ in the DL in Figure~\ref{fig:dl} predicts
	$\geq 50k$.
	Similarly, the sum of the weights ($0.1063$, $0.0707$
	and $-0.0128$ in the $3$ trees, respectively) for prediction $\geq
	50k$ is positive ($0.1642$) in the BT in
	Figure~\ref{fig:bt}, and so the BT model also predicts $\geq 50k$
	for the aforementioned instance $\mbf{v}$.
	%
  %
\end{example}

\ignore{
\subsection{Booleanization of Classification Problems}

The approaches we outline for explanation rely on the
input classification problem being binarized.
Note that the ML model we explain need not be defined on the binarized
version of the problems, but the explanations will be in terms of the
binarized problem.

We can easily convert a non-binary categorical feature $i$ where $\fml{D}_i$
is a finite set to a set of
binary features by \emph{one hot encoding}. We replace feature $i$ by
$|\fml{D}_i|$ Boolean
features of the form $v_i = d, d \in \fml{D}_i$.

Similarly we can binarize a real or large domain integer feature $i$
by quantizing the domain $\fml{D}_i$ into $k$ Boolean features.
Suppose $l = \min \fml{D}_i$ and $u = \max \fml{D}_i$ then a simple uniform
quantization replaces feature $i$ by the $k$ features
$l + (u-l)/k*i \leq v_i \leq l + (u-l)/k*(i+1), 0 \leq i < k$.
}

\ignore{
\subsection{Data Mining, Background Knowledge.}
\pnote{I am not sure we need this section, we do need to say what we assume
  background knowledge is, we do need to clarify that it is on the
  Booleanized form of the problem!}
\emph{Data mining} is the
discovery process of extracting useful information from large data.
}

\subsubsection{Interpretability and Explanations.}
Interpretability is not formally defined since it is a
subjective concept~\cite{lipton-cacm18}.
In this paper, we define interpretability as the
conciseness of the computed explanations
for an ML model to justify a provided prediction.
The definition of explanation for an ML model is built on earlier
work~\cite{darwiche-ijcai18,inms-aaai19,darwiche-ecai20,marquis-kr20,msi-aaai22},
where explanations are equated with \emph{abductive explanations} (AXps),
which are subset-minimal sets of features sufficing to
explain the prediction given by an ML model.
Concretely, given an instance $\mbf{v} \in \mbb{F}$ and a computed prediction
$c \in \fml{K}$, i.e. $\tau(\mbf{v}) = c$, an AXp is a subset-minimal set of
features $\fml{X} \subseteq \fml{F}$, such that

\begin{equation} \label{eq:axp}
\forall(\mbf{x} \in \mbb{F}). \bigwedge\nolimits_{i \in \fml{X}}
(x_i = v_i) \limply (\tau(\mbf{x}) = c)
\end{equation}

Abductive explanations are also prime implicants of
the decision predicate $\tau_c$ and hence
a \emph{prime implicant} (PI) explanation is another name for an AXp.

\begin{example}\label{ex:axp}
	Consider the models in Figure~\ref{fig:models} and
	instance $\mbf{v}$ from Example~\ref{ex:cls}.
	By examining the DL model, specifying
	$\text{Education} = \text{HighSchool}$, $\text{\Status} =
	\text{Married}$, $\text{Occupation} = \text{Sales}$, and
	$\text{Relationship} = \text{Husband}$ guarantees that any compatible
	instance is classified by $R_2$ independent of the values of other
	features, i.e.\ Sex and Hours/w.
	Similarly, the prediction of an instance is guaranteed to be
	$\geq 50k$ in Figure~\ref{fig:bt} as long as the feature values above
	are used, since the sum of weights is promised to be $0.1063 +
	0.0707 + -0.0128 = 0.1642$ for class $\geq 50k$.
	Therefore, the (only) AXp $\fml{X}$ for the prediction of $\mbf{v}$
	is $\{$Education, \Status, Occupation, Relationship$\}$ in
	both models.
  %
  %
\end{example}

We also consider \emph{contrastive explanations} (CXps)
defined as subset-minimal sets of features that are
necessary to change the prediction if the features of a CXp
are allowed to take arbitrary values from their
domains.
Formally and following~\cite{inams-aiia20}, a CXp for prediction
$\tau(\mbf{v})=c$ is defined as a minimal subset $\fml{Y} \subseteq
\fml{F}$ such that
\begin{equation} \label{eq:cxp}
	\exists(\mbf{x}\in\mbb{F}).\bigwedge\nolimits_{i\not\in\fml{Y}}(x_i=v_i)\land(\tau(\mbf{x})\not=c)
\end{equation}

\begin{example}\label{ex:cxp}
	Consider the setup of Example~\ref{ex:axp}.
	Given either model,
	$\fml{Y} = \{$Occupation$\}$ is a CXp for instance $\mbf{v}$
	because the prediction for $\mbf{v}$ can be
	changed if feature `Occupation' is allowed to take another value,
	e.g.\ if the value is changed to `Service'.
	%
	%
	Similarly, changing the value of feature `Occupation'
	to `Service' triggers that the weights in the $3$ trees become
	$0.1063$, $-0.2231$ and $-0.0128$.
	Therefore, the total weight is $-0.0982$, i.e.\ the prediction is
	changed.
	By further examining the two models, other subsets of features can
	be identified as CXps for~$\mbf{v}$.
	The set of CXps is $\mbb{Y} = \{\{$Education$\}$,
		$\{$\Status$\}$, $\{$Occupation$\}$,
		$\{$Relationship$\}\}$, while the set of AXps demonstrated in
		Example~\ref{ex:axp} is
		$\mbb{X}=\{\{$Education, \Status, Occupation, Relationship$\}\}$.
  %
  %
\end{example}

Recent work, which builds on the seminal work of
Reiter~\cite{reiter-aij87}, establishes a minimal hitting set (MHS)
duality relationship between AXps and CXps~\cite{inams-aiia20}.
In other words, each CXp \emph{minimally hit} every AXp, and
vice-versa.
The explanations enumeration algorithms used in this paper employ
this fact.

\begin{example} \label{ex:xdual}
	Observe how the minimal hitting set duality holds for the set of
	abductive explanations $\mbb{X}$ and the set of contrastive
	explanations $\mbb{Y}$ shown in Example~\ref{ex:cxp}.
	The only abductive explanation minimally hits all the
	contrastive explanations and vice versa.
\end{example}

There is a growing body of recent work on formal
explanations~\cite{msgcin-nips20,msgcin-icml21,ims-ijcai21,ims-sat21,barcelo-nips21,kutyniok-jair21,darwiche-jair21,kwiatkowska-ijcai21,mazure-cikm21,tan-nips21,iims-jair22,rubin-aaai22,iisms-aaai22,hiicams-aaai22,msi-aaai22,an-ijcai22,leite-kr22,barcelo-corr22}.

\section{Extracting Background Knowledge} \label{sec:extract}

Recent work~\cite{rubin-aaai22} argues that background knowledge is
helpful in the context of formal explanations.
The idea is that, if identified, background knowledge may help forbid
some of the combinations of feature values that would otherwise have
to be taken into account by a formal reasoner, thus, slowing the
reasoner down and making the explanations unnecessarily long.
But the question of how such knowledge can be obtained in an
automated way remains open.
\begin{example}\label{ex:rextract1}
	%
	Assume that Table~\ref{tab:inst} represents \emph{trustable} information. The
	following two rules can be extracted:
	\begin{itemize}
		\item \tn{IF} Relationship $=$ Husband \tn{THEN} \Status{}
			$=$ Married
		\item  \tn{IF} Relationship $=$ Wife \tn{THEN} \Status{} $=$
			Married
	\end{itemize}
	These rules may be used to discard feature
	\emph{\Status} when computing explanations
	as long as \emph{Relationship} equals either \emph{Husband} or
	\emph{Wife} because of the implications identified.
\end{example}

Here we describe the MaxSAT-based approach to automatically extract
background knowledge, which represents hidden relations between
features of a dataset if the dataset is assumed to be trustable.
It builds on the recent two-stage approach~\cite{ilsms-aaai21} to
learning smallest size decision
sets~\cite{resende-mp92,leskovec-kdd16,ipnms-ijcar18,meel-cp18,meel-aies19,yislb-cp20,yisb-jair21}.
Concretely, we apply the first stage of~\cite{ilsms-aaai21} which
enumerates individual decision rules given a dataset, using
MaxSAT.
%


Without diving into the details, the idea of~\cite{ilsms-aaai21}
is as follows.
Given training data $\fml{E}$ and target class $c\in\fml{K}$, a
MaxSAT solver is invoked multiple times, each producing a unique
subset-minimal (irreducible) rule in the form of \emph{``IF
\emph{antecedent} THEN \emph{prediction} $c$''}, where the antecedent
is a set of feature values.
The MaxSAT solver is fed with various CNF constraints and an objective
function targeting rule size minimization.
%
The approach also detects and blocks \emph{symmetric rules}, i.e.\ those that do not
contribute new information to the rule-based representation of class
$c\in\fml{K}$.
%

%

We can modify the MaxSAT approach outlined above to learning
background knowledge in the form of decision rules, i.e.\ identifying
the dependency of a feature $i \in \fml{F}$ on other features $j \in
\fml{F} \setminus \{i\}$.
For this, we need to discard the prediction column from the dataset
$\fml{E}$ and instead focus on a feature $i\in\fml{F}$, consider some
of its values $v_{ij}\in\fml{D}_i$ and ``pretend'' to compute decision
rules for a ``fake class'' $x_i=v_{ij}$.
Thanks to the properties of the approach of~\cite{ilsms-aaai21}, all
the rules computed are guaranteed to be subset-minimal and to respect
training data $\fml{E}$.
Once all the rules for feature $i\in\fml{F}$ and value
$v_{ij}\in\fml{D}_i$ are computed, the same exercise can be repeated
for all the values in $\fml{D}_i\setminus\{v_{ij}\}$ but, more
importantly, all the other features.


%
\begin{example}\label{ex:rextract2}
	Consider again Table~\ref{tab:inst}.
	The two rules shown in Example~\ref{ex:rextract1} are computed by our
	rule learning approach if we focus on feature \emph{\Status}.
	The following two rules can be extracted when feature
	\emph{Relationship} is focused on instead:
	\begin{description}
		\item[$\bullet$] \tn{IF} \Status{} $=$ Married $\land$ Sex $=$ M. \tn{THEN} Rel. $=$ Husband
		\item[$\bullet$] \tn{IF} \Status{} $=$ Married $\land$ Sex $=$ F. \tn{THEN} Rel. $=$ Wife
		\end{description}
\end{example}
As can be observed below, both Example~\ref{ex:rextract1} and
Example~\ref{ex:rextract2} may be used to shorten explanations under
certain conditions (see Section~\ref{sec:apply}).

\ignore{
\paragraph{one-hot encoding.}
When dealing with the \emph{one-hot encoding} of
a categorical feature $i\in\fml{F}$, i.e.\ $D_i=\{1,\ldots,n\}$, we
implicitly enforce the constraint $\sum_{k\in[n]}{x_{ik}}=1$, which
the explainer \emph{has to respect}. Here, Boolean variables
$x_{ik}$ represent unique values of feature $i$, i.e.\ $x_{ik}=1$ iff
$f_i=j$. Similarly for quantized ordinal feautures we constrain exactly on
quantile to hold.
}

\myparagraph{Duplicate Rules.}
%
%
As mentioned above, all rules generated with the MaxSAT approach
of~\cite{ilsms-aaai21} are guaranteed to be subset-minimal.
Furthermore, none of the rules enumerated is symmetric with another
rule if considered in the \emph{if-then} form.
However, when the rules are treated as clauses, i.e.\ a disjunction of
Boolean literals, some rules may duplicate the other.
Indeed, recall that a rule of size $k\leq \left|\fml{F}\right|$ is of
the form $(f_1 \land \ldots \land f_{k-1}) \rightarrow f_k$ where each
$f_i$ represents a literal $(x_i=v_{i_{j_i}})$, $i\in\fml{F}$ and
$v_{i_{j_i}}\in\fml{D}_i$.
Clearly, this same proposition can be equivalently represented as a
clause $(\neg{f_1} \lor \ldots \lor \neg{f_{k-1}} \lor f_k)$.
Observe that the same clause can be used to represent another rule
$(f_1 \land \ldots \land f_{k-2} \land \neg{f_k}) \rightarrow
\neg{f_{k-1}}$, which can thus be seen as symmetric in the
\emph{clausal form}.
This way, a clause of size $k$ represents $k$ possible rules.
However, due to symmetry, it suffices to compute only one of them and
block all the ``duplicates'' by adding its clausal representation to
the MaxSAT solver.
This novel symmetry breaking mechanism significantly improves the
scalability of our approach.

\begin{example} \label{ex:duplt}
	Consider a rule \{\,IF
	\text{\Status} $=$ \text{Married} $\land$
	\text{Sex} $=$ \text{Male} THEN
	\text{Relationship} $=$ \text{Husband}\,\}
	computed when compiling feature-value
	\text{Relationship} $=$ \text{Husband}.
  This rule is represented as a clause
  $$\left(\text{\Status} \neq \text{Married} \lor
  \text{Sex} \neq \text{M.} \lor
  \text{Relationship} = \text{Husband}\right)$$
  There are two duplicates in other contexts:
  \begin{description}
	  \item[$\bullet$]  \tn{IF}
	\text{\Status} $=$ \text{Married} $\land$
	\text{Rel.} $\neq$ \text{Husband} \tn{THEN}
	\text{Sex} = \text{F.}
	 \item[$\bullet$]  \tn{IF} \text{Sex} $=$ \text{M.} $\land$
	\text{Rel.} $\neq$ \text{Husband} \tn{THEN}
	\text{\Status} $\neq$ \text{Married}
\end{description}
  %
  %
\end{example}

\ignore{
Note that our approach does not extract the rules targeting
 $\neg f_i$, e.g. the last rule in Example~\ref{ex:duplt}, where
the rule target is the negation of $marital\ status = married$.
\pnote{ARE we assuming a Booleanized problem already, I assume we are.
Do we extract rules for $\neg f_i$ or not its not clear.
I think what this para is saying is if the feature is a Boolean resulting
from one-hot encoding or quantization we dont extract rules for the negation!}
\jynote{We don't extract rules for $\neg f_i$ when the domain of $f_i$ has more than
	2 values. But when the domain of
a feature has only 2 values, such as sex, we extract rules for sex != male, because
extracting rules for sex != male is equivalent to extracting rules for sex = female.
For simplicity, can we just say we don't extract rules for $\neg f_i$?}
}

\myparagraph{Extraction limit.}
Even if we remove duplicate rules, there can still be many
rules to enumerate for an entire dataset.
Many of them will never, or only rarely, contribute to reducing the
size of explanations of the classifier.
Extracting these \emph{low value} rules is unnecessary in the rule
extracting process.
In practice, we noticed that some rules (e.g. long rules or rules
having a low support) never contribute to explanation reduction.
\ignore{
Hence, using some common sense constructive limit (e.g. on rule
size, or support, among others) on the rules when enumerating them and
focusing only on the rules that satisfy the desired criteria helps us
not only avoid an overhead of exhaustive rule enumeration but also
does not damage the quality of explanations.
}
Hence, we apply an \emph{extraction limit} to prevent
exhaustive rule enumeration, which enables us to focus only on
\emph{most useful} rules.
Here, extraction limit can be a restriction of a user's choice, e.g. a
total extraction runtime, a limit on the number of rules, rule support
or size, etc.

%
%
\ignore{
Observe that the objective function fed to the MaxSAT solver (in the
form of soft clauses) describes the preference of rules, which the
solver has to respect when computing the rules.
%
%
For example, when size limit is selected as the extraction limit, the
objective function can offer the preference of small rules and,
therefore, smallest rules are computed first.
}

\begin{algorithm}[t]
	\caption{Rule Extraction}\label{alg:rextract}
	\textbf{Input}: Dataset $\fml{E}$, extraction limit  $\lambda$ \\
	\textbf{Output}: Rules $\varphi$
	\begin{algorithmic}[1]
		\STATE $\fml{E}_f, \fml{F} \gets \texttt{DropClass}(\fml{E}), \texttt{ExtractFeatures}(\fml{E})$
		\STATE $\varphi, B \gets \emptyset, \emptyset $ \COMMENT{to extract and block rules, resp.}
		\FOR {$i \in \fml{F}$} \label{line:forfeat}
		\FOR {$\textsf{rule}\in \texttt{EnumerateRules}(\fml{E}_f, i, B$)}
		\IF {\texttt{limit}($\textsf{rule}, \lambda$) is $true$}
		\STATE \textbf{break}
		\ENDIF
		\STATE $\varphi \gets  \varphi \cup \textsf{rule}$
		\ENDFOR
		\STATE $B \gets \varphi$
		\ENDFOR
		\RETURN $\varphi$
	\end{algorithmic}
\end{algorithm}

\ignore{
\begin{algorithm}[t]
	\SetStartEndCondition{ }{}{}%
	\SetKwProg{Fn}{Function}{:}{}
	\SetKwFunction{Rextract}{Rextract}%
	\SetKw{KwTo}{in}\SetKwFor{For}{for}{\string:}{}%
	\SetKwIF{If}{ElseIf}{Else}{if}{:}{elif}{else:}{}%
	\SetKwFor{While}{while}{:}{fintq}%
	\AlgoDontDisplayBlockMarkers\SetAlgoNoEnd\SetAlgoNoLine%
	\DontPrintSemicolon

	\caption{Rule Extracting}\label{alg:rextract}
	\Fn{\Rextract{$\fml{E}, \lambda$}}{
	\SetKwInOut{Input}{Input}\SetKwInOut{Output}{Output}
	\Input{Dataset $\fml{E}$, extraction limit  $\lambda$}
	\Output{Rules $\varphi$}
	$\fml{E}_f, \fml{F} \gets \texttt{DropClass}(\fml{E}), \texttt{ExtractFeatures}(\fml{E})$ \\
	$\varphi, B \gets \emptyset, \emptyset $ \tcp*{Extract rules and block rules}
	\For{$i \in \fml{F}$}{\label{line:forfeat}
		\For{$\textsf{rule}\in \texttt{EnumerateRules}(\fml{E}_f, i, B$)}{
			\uIf{\texttt{limit}($\textsf{rule}, \lambda$) is $true$}{
				\textbf{break}
			}
			 $\varphi \gets  \varphi \cup \textsf{rule}$

		}
		$B \gets \varphi$
	}
	\Return{$\varphi$}
	}
\end{algorithm}
}

A high-level view on the overall rule extraction approach is provided in Algorithm~\ref{alg:rextract}.
Initially, the class column from the original dataset $\fml{E}$ is
dropped and the features $\fml{F}$ in $\fml{E}$ are acquired.
For each feature $i \in \fml{F}$, the algorithm enumerates the
decision rules targeting $i$ until the extraction limit is met or no
more rules can be found.
The rules previously learned are blocked in the clausal form to avoid
computing their duplicates.
Finally, the algorithm returns the rules extracted.
	Our approach computes only rules that
	are perfectly consistent with the \emph{known} data, which makes
	sense if the data is extensive and trustworthy.
	In practical settings, however, some of the data are unknown, i.e.\
	the rules computed may be inconsistent with unseen parts of the
	feature space $\mbb{F}$.
	If testing and validation data are available,
        then the rules can be tested against them.
	We can then exclude the rules that are not
	\emph{sufficiently} accurate wrt.\ test and/or validation data.
	%

\section{Knowledge-Assisted Explanations} \label{sec:apply}

In this section, we show how to apply background knowledge as
additional constraints when computing a single formal abductive or
contrastive explanation for an ML model prediction but also when
enumerating them.
We also show how to identify the rules that have been used when
extracting formal explanations, which comes in handy when trustable
explanations are of concern.


We assume the obtained background knowledge can be represented as a
formula $\varphi$.
Under that assumption,~\cite{rubin-aaai22} proposes to
compute AXps for positive predictions of a Boolean classifier $\tau:
\mbb{F}\rightarrow\{0,1\}$ taking into account constraints $\varphi$.
Observe that formula $\varphi$ can be seen as representing a predicate
$\varphi:\mbb{F}\rightarrow\{0,1\}$, the truth value of which, i.e.\
$\varphi(\mbf{x})$, can be tested for an instance $\mbf{v}\in\mbb{F}$.
The approach of~\cite{rubin-aaai22} relies on
\emph{compiling} a Boolean classifier $\tau(\mbf{x})$ into a
\emph{tractable} representation~\cite{darwiche-ijcai18} and proposes
to compute an AXp $\fml{X}\subseteq\fml{F}$ for prediction
$\tau(\mbf{v})=1$, $\mbf{v}\in\mbb{F}$, subject to constraints
$\varphi$ as a prime implicant of
$\left[\varphi(\mbf{x})\rightarrow\tau(\mbf{x})=1\right]$.

Observe that we can generalize this idea to the context of computing
formal abductive and contrastive explanations for \emph{any
classifier} that admits a logical representation suitable for making
reasoning oracle calls wrt.\ formulas~\eqref{eq:axp}
and~\eqref{eq:cxp}.
Namely, given a prediction $\tau(\mbf{x})=c$, $\mbf{v}\in\mbb{F}$,
$c\in\fml{K}$, an abductive explanation $\fml{X}\subseteq\fml{F}$
subject to background knowledge $\varphi$ is such that:
\begin{equation} \label{eq:axpc}
	\forall(\mbf{x}\in\mbb{F}).\bigwedge\nolimits_{j\in\fml{X}}(x_j=v_j)\limply\left[\varphi(\mbf{x})\limply(\tau(\mbf{x})=c)\right]
\end{equation}

More importantly, the same can be done with respect to contrastive
explanations.
Given a prediction $\tau(\mbf{x})=c$,
$\mbf{v}\in\mbb{F}$, $c\in\fml{K}$, a contrastive explanation
$\fml{Y}\subseteq\fml{F}$ subject to background knowledge $\varphi$ is
such that the following holds:
\begin{equation} \label{eq:cxpc}
	\exists(\mbf{x}\in\mbb{F}).\bigwedge\nolimits_{i\not\in\fml{Y}}(x_i=v_i)\land\left[\varphi(\mbf{x})\land(\tau(\mbf{x})\neq c)\right]
\end{equation}

\ignore{

  %
  %


}

Note that~\eqref{eq:axpc} and~\eqref{eq:cxpc} are the negation of each
other, i.e.\ a subset of features $\fml{Y}\subseteq\fml{F}$ \emph{is}
a CXp for prediction $\tau(\mbf{x})=c$ iff
$\fml{X}=\fml{F}\setminus\fml{Y}$ \emph{is not} an AXp.
This means when dealing with either AXps or CXps, one can reason
about (un)satisfiability of formula
$\bigwedge\nolimits_{i\in\fml{Z}}(x_i=v_i)\land\left[\varphi(\mbf{x})\land(\tau(\mbf{x})\neq
c)\right]$ with $\fml{Z}$ being either $\fml{X}$ or
$\fml{F}\setminus\fml{Y}$ depending on the kind of target explanation.
%
%
Therefore, if background knowledge $\varphi$ is a
\emph{conjunction} of constraints, e.g.\ rules, we can integrate them
in the existing formal explainability setup
of~\cite{inms-aaai19} with no additional overhead.


Following~\cite{inams-aiia20} and applying the same arguments, an
immediate observation to make is that in the presence of background
knowledge, the minimal hitting set duality between AXps and CXps
holds:

\begin{proposition} \label{prop:xdual}
	Let $\mbf{v}\in\mbb{F}$ be an instance such that $\tau(\mbf{v})=c$,
	$c\in\fml{K}$, and background knowledge $\varphi$ is compatible with
	$\mbf{v}$.
	Then any AXp $\fml{X}$ for prediction $\tau(\mbf{v})=c$ minimally
	hits any CXp for this prediction, and vice versa.
\end{proposition}

Proposition~\ref{prop:xdual} enables us to apply
algorithms originally studied in the context of over-constrained
systems~\cite{bs-dapl05,liffiton-jar08,blms-aicom12,mshjpb-ijcai13,mpms-ijcai15,iplms-cp15,lpmms-cj16,bcb-atva18}
to explore all AXps and CXps for ML predictions.
In particular, the existing explanation extraction and enumeration
algorithms~\cite{inms-aaai19,inams-aiia20} can be readily applied by
taking into account background knowledge, as shown
in~\eqref{eq:axpc} and~\eqref{eq:cxpc}.
%
%
%

Gorji et al.~\cite{rubin-aaai22} proved that
subset-minimal AXps
computed subject to additional
constraints for Boolean classifiers tend to be smaller than their
unconstrained ``counterparts''.
The rationale is that when additional constraints are imposed, some of
the features $i\in\fml{F}$ may be dropped from an AXp because the
equalities $x_i=v_i$ falsify the constraints, i.e.\ they represent data
instances that are \emph{not permitted} by the constraints.
Based on their result, the following generalization can be proved to 
hold:
\begin{proposition} \label{prop:caxp}
	%
	Consider $\mbf{v}\in\mbb{F}$ such that $\tau(\mbf{v})=c$,
	$c\in\fml{K}$, and background knowledge $\varphi$ is compatible with
	$\mbf{v}$.
	%
	%
	Then for any subset-minimal AXp $\fml{X}\subseteq\fml{F}$ for
	prediction $\tau(\mbf{v})=c$, there is a subset-minimal AXp
	$\fml{X}'\subseteq\fml{F}$ for $\tau(\mbf{v})=c$ subject to
	background knowledge $\varphi$ such that $\fml{X}'\subseteq\fml{X}$.
	%
\end{proposition}

\begin{proof}
	First, observe that if~\eqref{eq:axp} holds for a set $\fml{S}$
	then~\eqref{eq:axpc} holds for $\fml{S}$ too.
	%
	%
	Let $\fml{X}$ be a subset-minimal AXp for $\tau(\mbf{v})=c$ with no
	knowledge of $\varphi$, i.e.~\eqref{eq:axp} holds for $\fml{X}$.
	Thanks to the observation above,~\eqref{eq:axpc} also holds for
	$\fml{X}$.
	To make it subset-minimal subject to $\varphi$, we can apply linear
	search feature traversal (similar to the AXp extraction
	algorithm~\cite{inms-aaai19}) checking if any of the features of
	$\fml{X}$ can be dropped s.t.~\eqref{eq:axpc} still holds.
	The result subset-minimal set of features $\fml{X}'$ is the target
	AXp subject to knowledge $\varphi$.
\end{proof}

\begin{remark}
	Note that the opposite, i.e.\ that given AXp $\fml{X}'$ subject to
	background knowledge $\varphi$, there must exist a subset-minimal AXp
	$\fml{X}\supseteq\fml{X}'$ without background knowledge $\varphi$,
	in general does not hold.
	To illustrate a counterexample, consider a fully Boolean classifier
	$\tau:\{0,1\}^3\rightarrow\{0,1\}$ on features $\fml{F}=\{a, b,
	c\}$, which returns 1 iff $(a+b+c)\geq 2$.
	Consider instance $\mbf{v}=(1,1,0)$ classified as 1.
	Given knowledge $\varphi=(\neg{c}\rightarrow
	a)\land(\neg{c}\rightarrow b)$, a valid subset-minimal AXp is
	$\fml{X}'=\{c\}$.
	However, when discarding knowledge $\varphi$, the only
	subset-minimal AXp for $\mbf{v}$ is
	$\fml{X}=\{a,b\}\not\supseteq\fml{X}'$.
\end{remark}

\begin{figure}[t]
	\centering
        \frmeq{\scriptsize\arraycolsep=2pt
	\begin{array}{lllcl}
		\tn{R$_{0}$:} & \tn{IF}     & \text{\Status{} = Married} & \tn{THEN} & \text{Target $\geq 50$k} \\
		\tn{R$_{1}$:} & \tn{ELSE IF} & \text{Sex = Male}  \land \text{Relationship} \neq \text{Husband} & \tn{THEN} & \text{Target $< 50$k} \\
            \tn{R$_{\tn{\textsc{def}}}$:} & \tn{ELSE} &  & \tn{THEN} &  \text{Target $\geq 50$k} \\
        \end{array}
	}
	\caption{A DL for selected examples of \emph{adult} dataset.}
    \label{fig:dlexmp}
\end{figure}

\begin{example} \label{ex:smallaxp}
	Consider the DL in Figure~\ref{fig:dlexmp}.
	Given an instance~$\mbf{v} = \{$Education = Dropout, \Status{}
		= Separated, Occupation = Service, Relationship = Not-in-Family,
		Sex = Male, Hours/w = $\leq 40\}$, the prediction
		enforced by \tn{R$_{1}$} is $\leq 50$k and the AXp is $\fml{X} =
		\{$\Status, Relationship, Sex$\}$.
	Let a single constraint~$\varphi$ be
	$\{ \text{Sex} = \text{Male} \land \text{Relationship} =
	\text{Not-in-Family} \rightarrow \text{\Status} =
\text{Separated} \}$.
	%
	%
	Feature `\Status' can be dropped because the constraint $\varphi$
	ensures it to be set to the ``right value'' if the other two
	features are set as required, and hence \tn{R$_{0}$} is guaranteed
	not to fire.
	Thus, we can compute a smaller
	AXp $\fml{X}' = \{$ Relationship, Sex $\}$.
\end{example}

While using background knowledge~$\varphi$ pays off in terms of
interpretability of abductive explanations, this cannot be said
wrt.~contrastive explanations.
Surprisingly and as the following result proves, background knowledge
can only contribute to increase the size of contrastive
explanations.



\begin{proposition} \label{prop:ccxp}
	Consider $\mbf{v}\in\mbb{F}$ such that $\tau(\mbf{v})=c$,
	$c\in\fml{K}$, and background knowledge $\varphi$ is compatible with
	$\mbf{v}$.
	%
	%
	%
	Then for any subset-minimal CXp $\fml{Y}'\subseteq\fml{F}$ for
	prediction $\tau(\mbf{v})=c$ subject to knowledge
	$\varphi$, there is subset-minimal CXp $\fml{Y}\subseteq\fml{F}$ is
	a CXp for prediction $\tau(\mbf{v})=c$ such that
	$\fml{Y}'\supseteq\fml{Y}$.
\end{proposition}

\begin{proof}
	First, observe that if~\eqref{eq:cxpc} holds for a set $\fml{S}$
	then~\eqref{eq:cxp} holds for $\fml{S}$ too.
	%
	%
	%
	Let $\fml{Y}'$ be a subset-minimal CXp subject to background
	knowledge $\varphi$, i.e.~\eqref{eq:cxpc} holds for $\fml{Y}'$.
	By the observation made above,~\eqref{eq:cxp} also holds for
	$\fml{Y}'$.
	Now, by applying linear search dropping features of $\fml{Y}'$ and
	checking~\eqref{eq:cxp}~\cite{inms-aaai19}, one can get a
	subset-minimal $\fml{Y}\subseteq\fml{Y}'$ wrt.~\eqref{eq:cxp}, i.e.\
	$\fml{Y}$ is a subset-minimal CXp.
\end{proof}

\begin{remark}
	Note that the reverse direction: given a CXp $\fml{Y}$ generated
	without using background knowledge, there must exists a CXp
	$\fml{Y'} \supseteq \fml{Y}$ using background knowledge, does not
	hold.
	Consider a classifier on Boolean features $\fml{F} = \{a,b,c\}$
	which returns the parity ODD, EVEN of $a+b+c$.
	Consider background knowledge $a = b$.
	Now $\fml{Y} = \{a\}$ is a CXp for $\tau(1,1,1)=\text{ODD}$ without
	using background knowledge supported by instance
	$\tau(0,1,1)=\text{EVEN}$.
	But this does not agree with the background knowledge.
	The only CXp using the background knowledge is $\{c\}$, because $a$
	and $b$ must change together they never affect the parity.
	However and as our experimental results confirm, in practice these
	examples do not arise, as we always find a CXp using
	background knowledge that extends a CXp without background
	knowledge.
\end{remark}

One may wonder then why background knowledge is useful when
computing CXps.
The reason is that the CXps generated using background knowledge are
\emph{correct} under the assumption that the background knowledge
describes the actual relationships between features.
On the contrary, CXps generated without using background knowledge are
\emph{only correct} under the assumption that every combination of
feature values is possible, i.e.\ all features are independent and
their values are uniformly distributed across the feature space,
which hardly ever occurs in practice.

\begin{example}
	Consider the setup of Example~\ref{ex:smallaxp}.
	Observe that
	a CXp for the prediction is $\fml{Y} = \{$
	\Status{} $\}$.
	Its correctness relies on the fact that changing
        \Status{} to Married changes the prediction to $\geq 50k$.
	But given the background knowledge $\varphi$, this is clearly
	erroneous.
	Since the other fixed features in instance $\mbf{v}$ are $\{$ Sex
	$=$ Male, Education $=$ Dropout, Occupation $=$ Service,
	Relationship $=$ Not-in-family, Hours/w $=$ $\leq 40\;\}$, the
	modification is inconsistent with knowledge $\varphi$.
	This demonstrates the weakness of CXps as they rely on the
	assumption that any tuple of feature values in $\mbb{F}$ is
	possible.
	Applying constraint $\varphi$
	leads to
	a larger CXp $\fml{Y}' \triangleq \{$ \Status,
	Relationship
	$\}$.
	This clearly does allow the prediction to change and it is
	compatible with $\varphi$.
\end{example}

\subsection{Attributing Responsibilities to Knowledge}
Since the computed background knowledge is not always useful, e.g.\
the extracted rules may not necessarily contribute to smaller AXps, we
introduce an approach to discovering which of the rules are used to
reduce an explanation.
Using this approach, we can observe and measure the effect of the
value of extraction limit discussed in Section~\ref{sec:extract}, e.g.\
the size limit of 5 can be considered as a reasonable extraction
limit if the size of most of the \emph{useful rules} is no more than
5.
A further usage is that when providing a user with an explanation, we
can expose which background knowledge was used to generate the
explanation.
This enables the user to assess the quality of the rules used and
decide whether they trust or disagree with the background knowledge.

\begin{algorithm}[t!]
	\caption{Determine Background Knowledge Used}\label{alg:usedrules}
	\textbf{Input}: Classifier $\tau$, instance $\mbf{v}$, prediction $c=\tau(\mbf{v})$,
		constraints $\varphi$, AXp $\fml{X}'$ \\
	\textbf{Output} Used rules: $\varphi_u\subseteq\varphi$
	\begin{algorithmic}[1]
		\STATE $\varphi_u \gets \varphi$
		\IF {\label{ln:check1}$\texttt{Entails}(\fml{X}', \tau, \mbf{v}, c, \emptyset$)}
		\RETURN $\emptyset$
		\ENDIF
		\FOR {$r \in \varphi$}
			\IF {\label{ln:check2}$\texttt{Entails}(\fml{X}', \tau, \mbf{v}, c, \varphi_u \setminus \{r\}$)}
			\STATE $\varphi_u  \gets \varphi_u  \setminus \{r\}$
			\ENDIF
		\ENDFOR
		\RETURN $\varphi_u$
	\end{algorithmic}
\end{algorithm}

%

%
Given a knowledge-assisted AXp $\fml{X}'$ for prediction
$\tau(\mbf{v})=c$ and background knowledge $\varphi$, Algorithm~\ref{alg:usedrules}
reports a subset of rules $\varphi$ \emph{responsible} for the
explanation $\fml{X}'$.
The algorithm makes use of a number of calls to \texttt{Entails},
which is meant to be a call to reasoning oracle deciding the validity
of formula~\eqref{eq:axpc} subject to background knowledge specified
as the \emph{final parameter}.
First, we check if $\fml{X}'$ satisfies the AXp
condition~\eqref{eq:axp} with no knowledge given. 
If this is the case, then no rules are used when computing $\fml{X}'$
and so the algorithm returns $\emptyset$.
Otherwise, the algorithm proceeds by considering each rule
$r\in\varphi$ one by one and checking if condition~\eqref{eq:axpc}
holds for background knowledge $\varphi\setminus\{r\}$ (see
line~\ref{ln:check2}).
If it does then rule $r$ can be dropped;
otherwise, it is necessary for AXp $\fml{X}'$ and is thus kept.
This simple linear search procedure ends up identifying a
subset-minimal set of rules $\varphi_u$ that are responsible for
abductive explanation $\fml{X}'$.

\ignore{
Algorithm~\ref{alg:usedrules} outlines the proposed approach.
Given a knowledge-assisted AXp $\fml{X}'$ for prediction
$\tau(\mbf{v})=c$ and background knowledge $\varphi$, the algorithm
reports a subset of rules $\varphi$ \emph{responsible} for the
explanation $\fml{X}$.
The algorithm makes use of a number of calls to \texttt{Entails},
which is meant to be a call to reasoning oracle deciding the validity
of formula~\eqref{eq:axpc} subject to background knowledge specified
as the \emph{final parameter}.
First, we check if $\fml{X}'$ satisfies the AXp
condition~\eqref{eq:axp} with no knowledge given (see
line~\ref{ln:check1}).
If this is the case, then no rules are used when computing $\fml{X}'$
and so the algorithm returns $\emptyset$.
Otherwise, the algorithm proceeds by considering each rule
$r\in\varphi$ one by one and checking if condition~\eqref{eq:axpc}
holds for background knowledge $\varphi\setminus\{r\}$ (see
line~\ref{ln:check2}).
If it does then rule $r$ can be dropped;
otherwise, it is necessary for AXp $\fml{X}'$ and is thus kept.
This simple linear search procedure ends up identifying a
subset-minimal set of rules $\varphi_u$ that are responsible for AXp
$\fml{X}'$.
}

Note that a similar algorithm can be outlined for
identifying background knowledge useful when computing contrastive
explanations.
In that case, instead of making calls to \texttt{Entails}, one would
need to make calls to a reasoner deciding the validity of
formula~\eqref{eq:cxpc} subject to a varying set of background
constraints $\varphi_u$.



\section{Experimental Results} \label{sec:res}

This section assesses our approach to extracting background knowledge
wrt. popular association rule mining algorithms
and the quality of the enumerated AXps and CXps with background
knowledge for 3 different ML models: DLs, BTs, and BNNs.
Finally, this section applies the background knowledge identified for
evaluating correctness of the explanations produced by some heuristic
ML explainers.

\myparagraph{Setup and Prototype Implementation.}
All the experiments were run on an Intel Xeon 8260 CPU running Ubuntu
20.04.2 LTS, with a memory limit of 8 GByte.
%
%
A prototype of the proposed approach to extracting background
knowledge and computing AXps and CXps applying background knowledge
was developed as a set of Python scripts.\footnote{The implementation
as well as all datasets and logs of our experiments is available at
\url{https://github.com/jinqiang-yu/xcon22}.}
The implementation of knowledge extraction builds
on~\cite{ilsms-aaai21} and extensively uses state-of-the-art SAT
technology~\cite{imms-sat18,imms-jsat19}.
%
Also, the implementation of explanation enumeration for DLs and BNNs
makes use of the SAT technology~\cite{imms-sat18} while for BTs we
apply modern SMT solvers~\cite{gm-smt15}.

A few words should be said about the competition considered.
First, we compared our knowledge extraction approach to
the Apriori and Eclat algorithms~\cite{as-vlbd94,zaki-kdd97}.
In our experiments, these algorithms behave almost identically with
Eclat solving one more instance; as a result, we use Eclat as the best
competitor.
%
%
%
%
When running Eclat, we apply the same setup as used for our approach.
Finally, heuristic explainers are represented by
LIME~\cite{guestrin-kdd16}, SHAP~\cite{lundberg-nips17}, and
Anchor~\cite{guestrin-aaai18} in their default configurations.

\myparagraph{Datasets.}
The benchmarks considered include a selection of datasets publicly
available from UCI Machine Learning Repository~\cite{dua-2019} and
Penn Machine Learning Benchmarks~\cite{Olson2017PMLB}.
In total, 24 datasets are selected.
%
Whenever applicable, numeric
features in all benchmarks were quantized into 4, 5, or 6 intervals.
%
%
Therefore, the total number of quantized datasets considered is 62.

\myparagraph{Machine Learning Models.}
We used CN2~\cite{clark-ml89} to train the DL models studied.
%
BTs were computed by XGBoost~\cite{guestrin-kdd16} s.t. each class
is represented by 25 trees of depth 3.
%
%
BNNs were trained by PyTorch~\cite{pytorch-neurips19}.
%
Three configurations of hidden layers\footnote{The 3
  configurations are classified as \emph{small}, \emph{medium} and
  \emph{large}. The size of the hidden layers of these
  3 configurations is as follows: large: (64, 32, 24, 2); medium:
(32, 16, 8, 2); small: (10, 5, 5, 2).} were used when training
BNNs to achieve sufficient test accuracy.
As usual, each of the 62 datasets was randomly split into
80\% of training and 20\% of test data,
respectively.
The average test accuracy of the
DL, BT, and BNN models was 76.47\%, 76.17\%, and 80.31\%.

\subsection{Knowledge Extraction}\label{sec:resrextract}

\begin{figure}[!t]
  \centering
    \centering
    \includegraphics[width=0.46\textwidth]{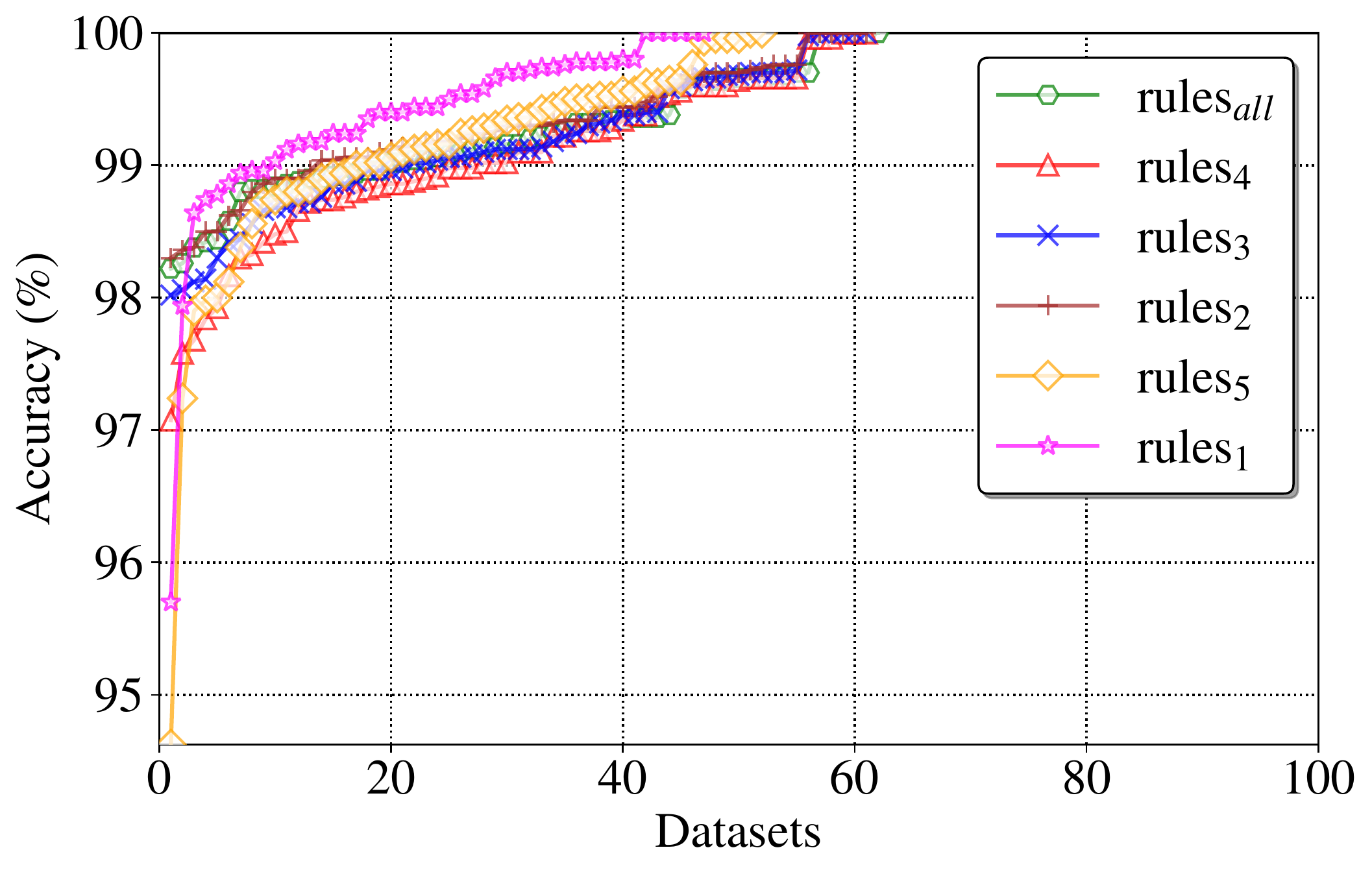}
    \caption{Accuracy of rules extracted by \emph{xcon}.}
    \label{fig:raccry}
\end{figure}

\begin{table}[t!]
\centering
\caption{Average accuracy of individual rules over test data.}
\label{tab:racry}
\scalebox{0.84}{
\begin{tabular}{ccccccc}\toprule
 & rule$_1$ & rule$_2$ & rule$_3$ & rule$_4$ & rule$_5$ & rule$_{all}$ \\ \midrule
Accuracy (\%) & $99.22$ & $99.35$ & $99.29$ & $99.16$ & $99.06$ & $99.08$ \\ \bottomrule
\end{tabular}
}
\end{table}

Knowledge extraction is tested using 5-fold cross validation.
The average accuracy of each rule is measured as the proportion of test
instances that violate that rule, averaged over the folds.
We consider rules of length 1 to 5, and extracting all possible rules
(\emph{all}).
%
%
%
%
%
%
%
\begin{figure}[!t]
	\centering
	%
  \begin{subfigure}[b]{0.238\textwidth}
    \centering
    \includegraphics[width=\textwidth]{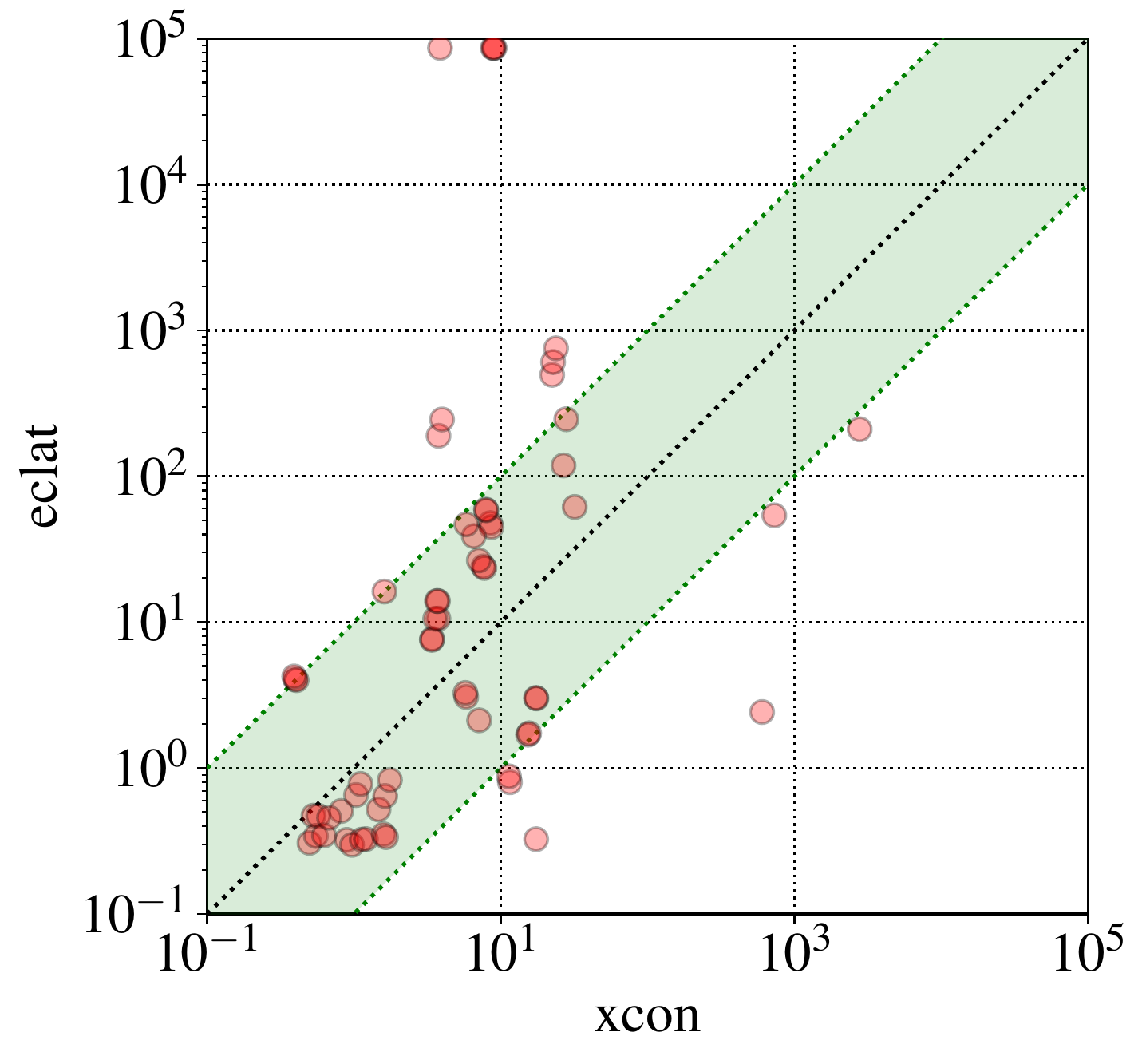}
    \caption{Runtime comparison.}
    \label{fig:eclatsizertime}
  \end{subfigure}
  \begin{subfigure}[b]{0.23\textwidth}
    \centering
    \includegraphics[width=\textwidth]{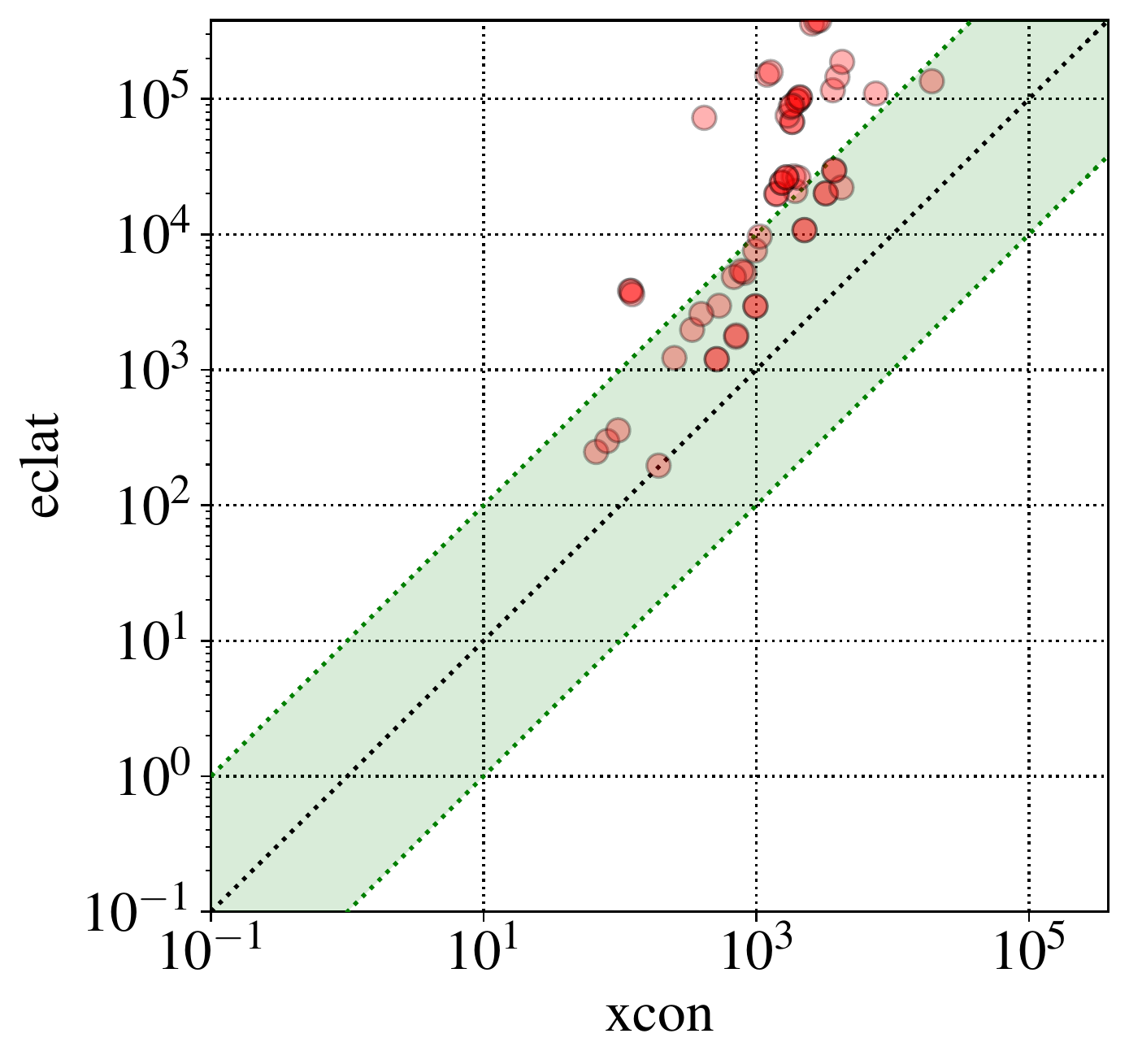}
    \caption{Number of rules comparison.}
    \label{fig:eclatsizenrules}
  \end{subfigure}
  \caption{\emph{Eclat} vs. \emph{xcon} -- performance comparison.}
  \label{fig:eclatrtime}
\end{figure}
Figure~\ref{fig:raccry} and Table~\ref{tab:racry} compare the average accuracy of the background
knowledge extracted.
%
As can be observed, the average accuracy of \emph{all rules}
and $rules_s$, $s\in \{1,\ldots,5\}$, exceeds 99\%.
\ignore{
As can be observed, average rule accuracy gets no lower than 94\%.
Furthermore, the average accuracy of \emph{all rules} in all datasets
is over 98\%.
Finally, for the majority of datasets, the accuracy of $rules_s$,
$s\in \{1,\ldots,5\}$, also exceeds 98\%.
}

Additionally, we compare the overall performance of exhaustive rule
extraction against rule extraction with the size limit 5.
On average, exhaustive (limited, resp.) rule enumeration ends up
computing 2116.29 (1964.24, resp.) rules per dataset.
According to our results, our approach is quite efficient and for the
lion's share of datasets (59 out of 62) both exhaustive and limited
enumeration finish within 30 seconds; for the 3 remaining datasets,
limited enumeration is a bit faster but both approaches finished rule
enumeration within 3000 seconds.

In the remainder of this section, we compare \emph{xcon} against Eclat
in terms of the overall performance. For a fair comparison, we set
Eclat to extract only rules of confidence~100\%, i.e.\ all the rules
extracted are perfectly consistent with the \emph{known} data.
%
%

%

\paragraph{Scalability.}
Figure~\ref{fig:eclatsizertime} demonstrates that \emph{xcon} can
extract rules faster or on par with Eclat in the vast majority of the
considered datasets.
Moreover, Eclat can only extract rules for the train-test 5-fold pairs
of 58 (out of 62) datasets, while \emph{xcon} is able to extract rules
for all the considered datasets. 

\paragraph{Rule Number.}
Figure~\ref{fig:eclatsizenrules} depicts the comparison of the number
of extracted rules in the 58 datasets solved by both approaches.
Observe that Eclat extracts more rules than \emph{xcon} because it
uses a less expressive language for the feature literals, i.e. it
cannot extract rules containing the \emph{negation} of a feature-value
pair.
For example, assume \emph{xcon} can extract a rule [IF $x_1 \not = 0 $
THEN $x_2 = 1$] given features~$1$ and $2$ and their
domains~$\fml{D}_1 = \fml{D}_2 = \{0, 1, 2\}$.
In this case, Eclat is unable to extract the above rule -- instead, it
has to extract two distinct rules to represent the same information,
i.e.\ [IF $x_1 = 1 $ THEN $x_2 = 1$] but also [IF $x_1 = 2 $ THEN $x_2
= 1$].


\subsection{Knowledge-Assisted Explanations} \label{sec:resexp}
%
This section evaluates the proposed approach to computing formal
explanations for DLs, BTs, and BNNs, where the computed background
knowledge was applied. In particular, we evaluate the runtime of
explanation enumeration, explanation size, as well as the portion of
background knowledge used when computing explanations.
Note that here we consider only the rules of size at most 5, which is
shown to be a reasonable value in Section~\ref{sec:rsize}.

For each of the 62 datasets, we selected all 
\emph{test} instances and enumerated 20 \emph{smallest size} AXps or
CXps for each such instance.
Hereinafter, \emph{xcon}$_{\ast}$ s.t.\ $\ast\in\{dl, bt, bnn\}$
denotes the proposed approach applied for explaining DL, BT, and BNN
models, respectively.
Furthermore, a superscripted version \emph{xcon}$_\ast^r$ is used to
denote the configurations that apply background knowledge.



\begin{figure*}[!ht]
	\begin{subfigure}[b]{0.244\textwidth}
    \centering
    \includegraphics[width=\textwidth]{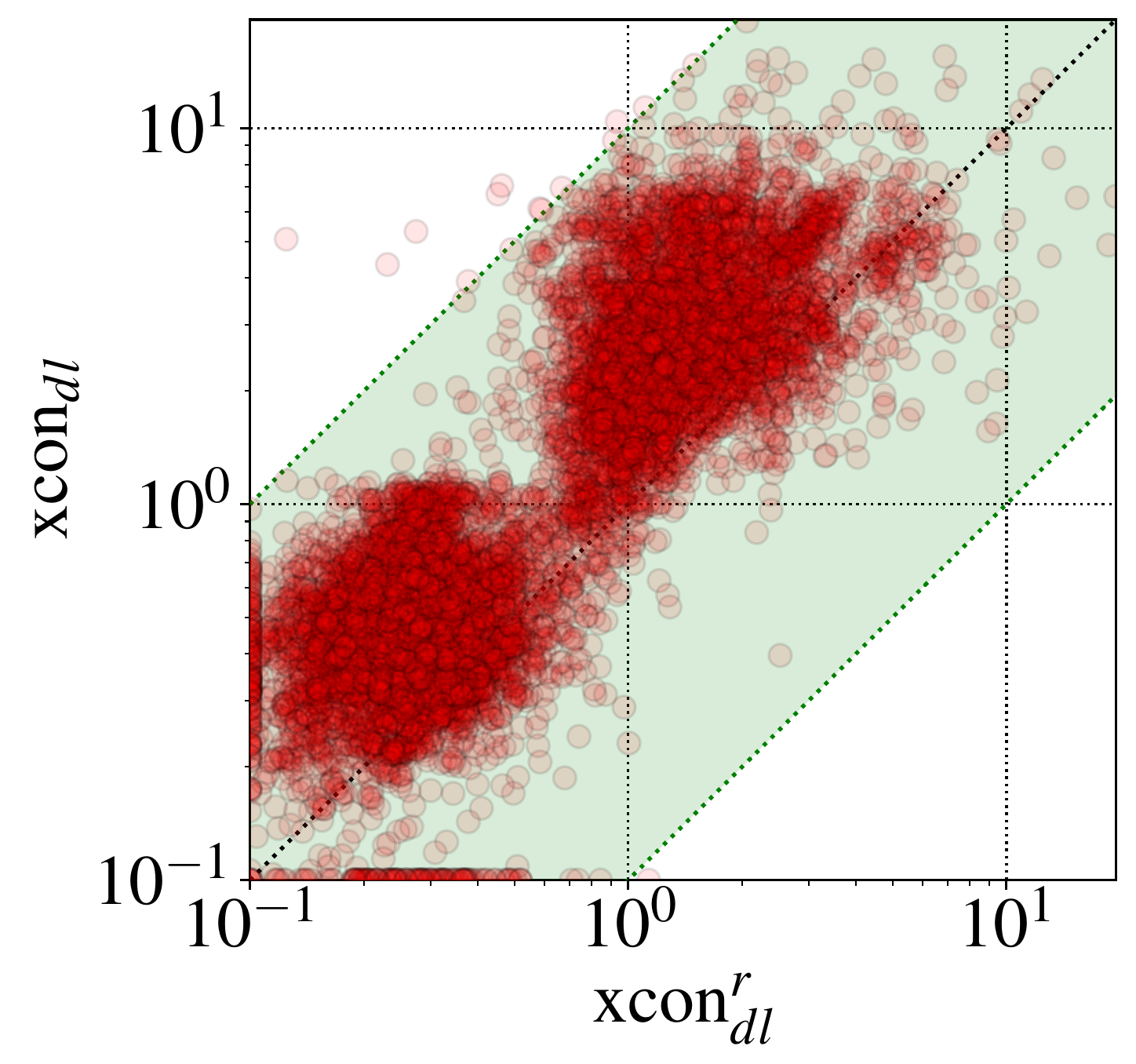}
    \caption{AXp runtime for DLs.}
    \label{fig:dlaxprtime}
  \end{subfigure}
  \hfill
  \begin{subfigure}[b]{0.244\textwidth}
    \centering
    \includegraphics[width=\textwidth]{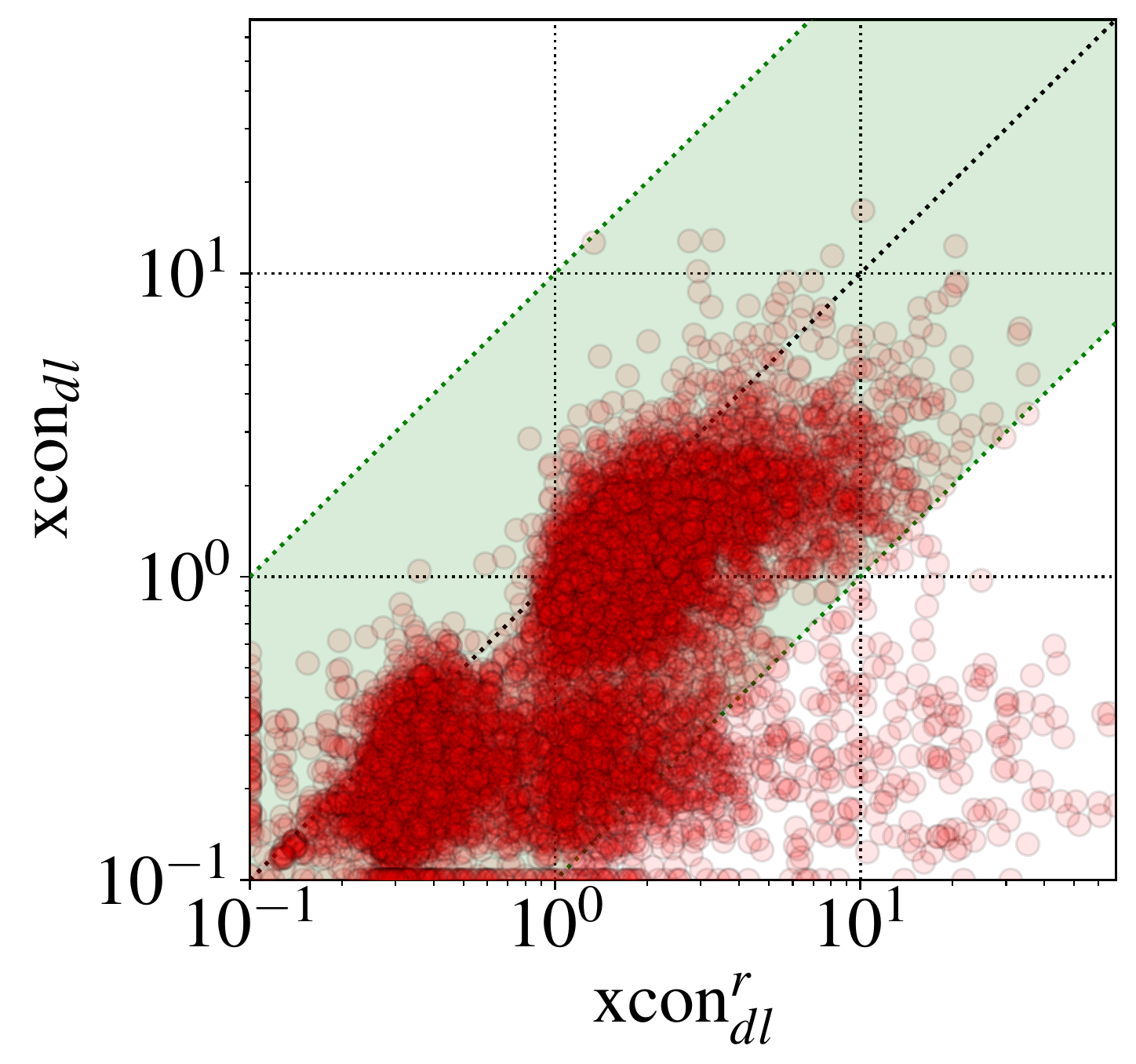}
    \caption{CXp runtime for DLs.}
    \label{fig:dlcxprtime}
  \end{subfigure}
  \hfill
  \begin{subfigure}[b]{0.23\textwidth}
    \centering
    \includegraphics[width=\textwidth]{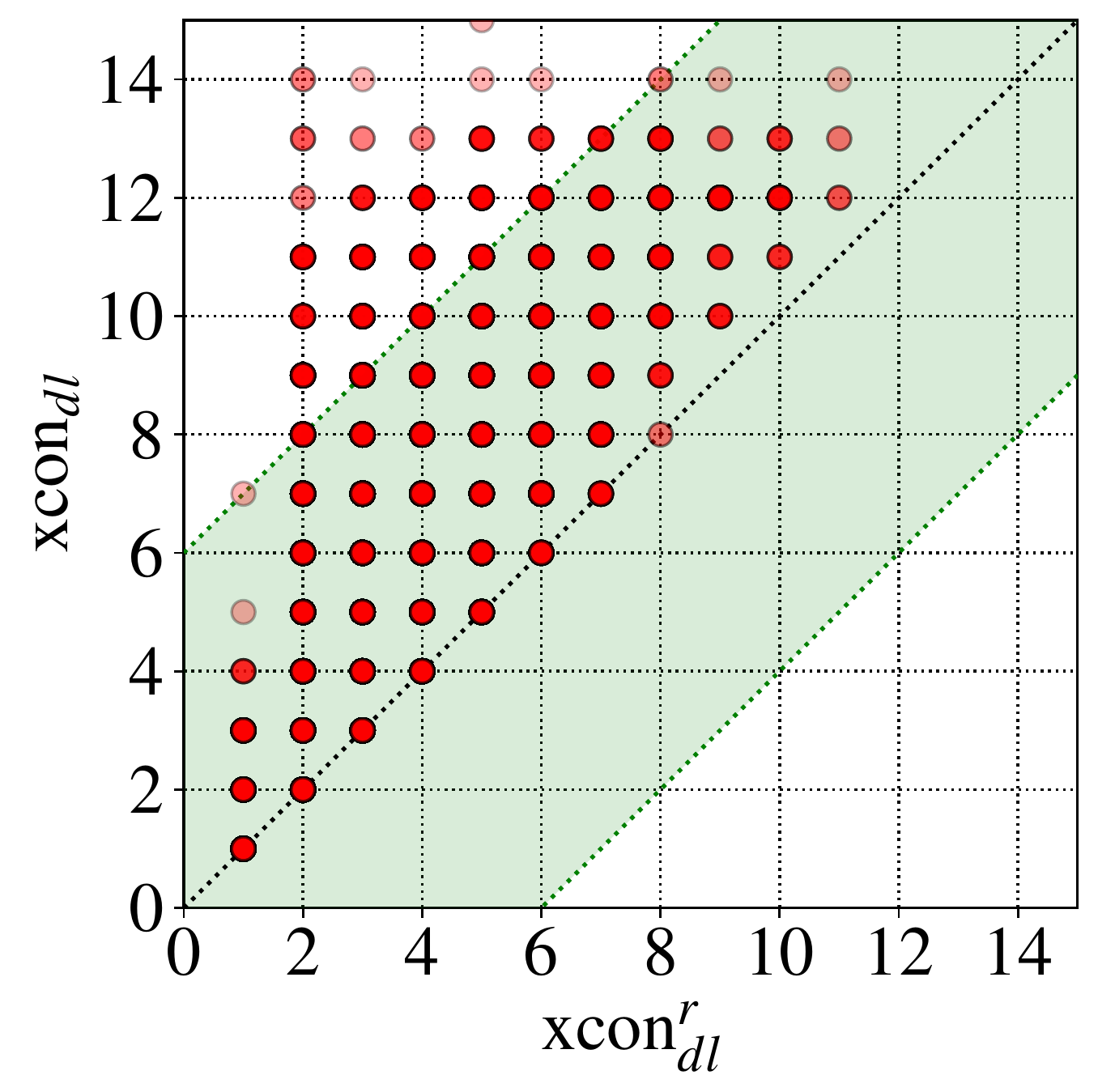}
    \caption{AXp size for DLs.}
    \label{fig:dlaxp}
  \end{subfigure}%
  \hfill
  \begin{subfigure}[b]{0.23\textwidth}
    \centering
    \includegraphics[width=\textwidth]{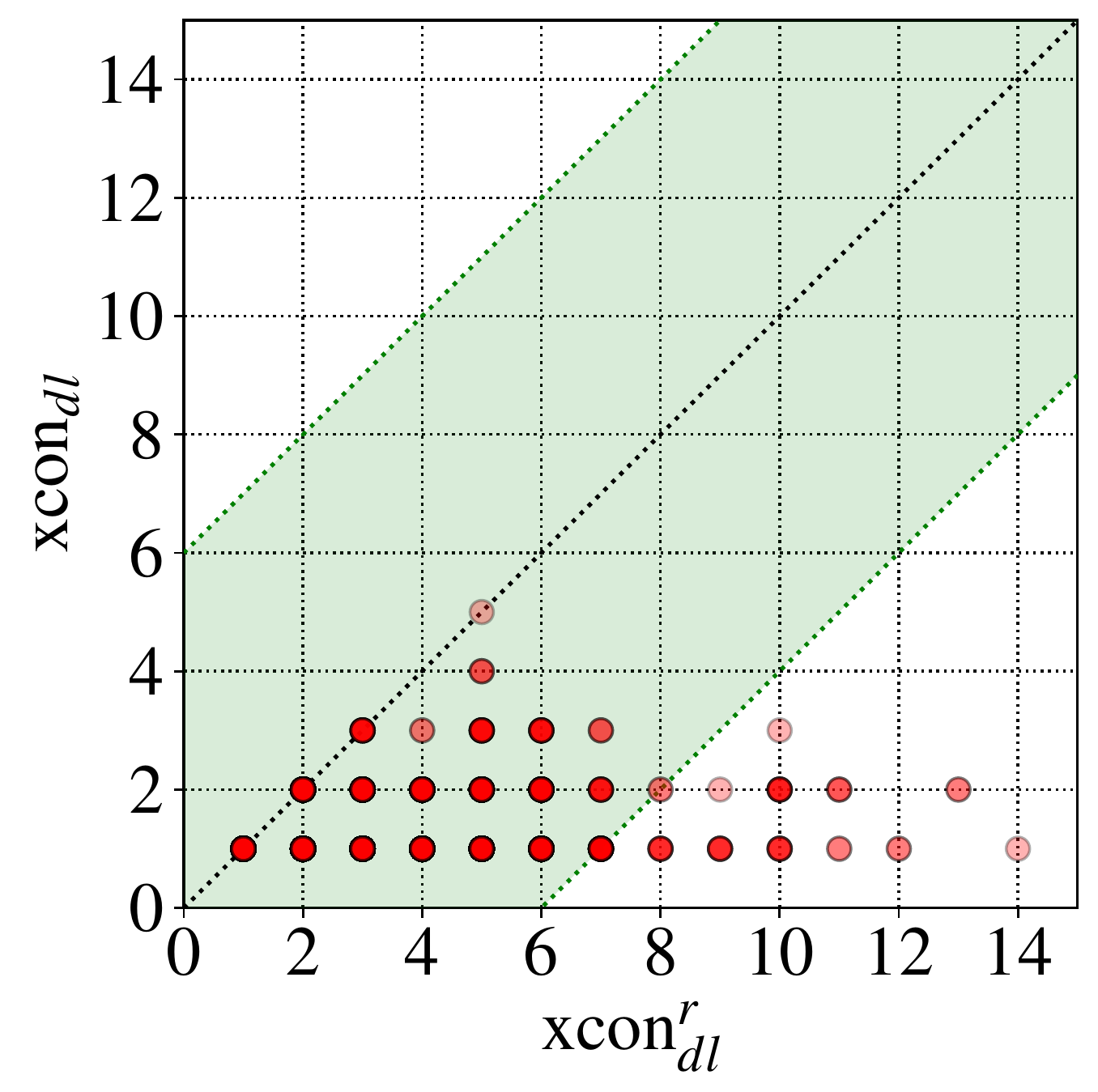}
    \caption{CXp size for DLs.}
    \label{fig:dlcxp}
  \end{subfigure}
  \caption{Impact of \emph{xcon} rules on runtime~(ms) and explanation size for DLs.}
  \label{fig:dlxprtime}
\end{figure*}

\begin{figure*}[!ht]
	\begin{subfigure}[b]{0.244\textwidth}
    \centering
    \includegraphics[width=\textwidth]{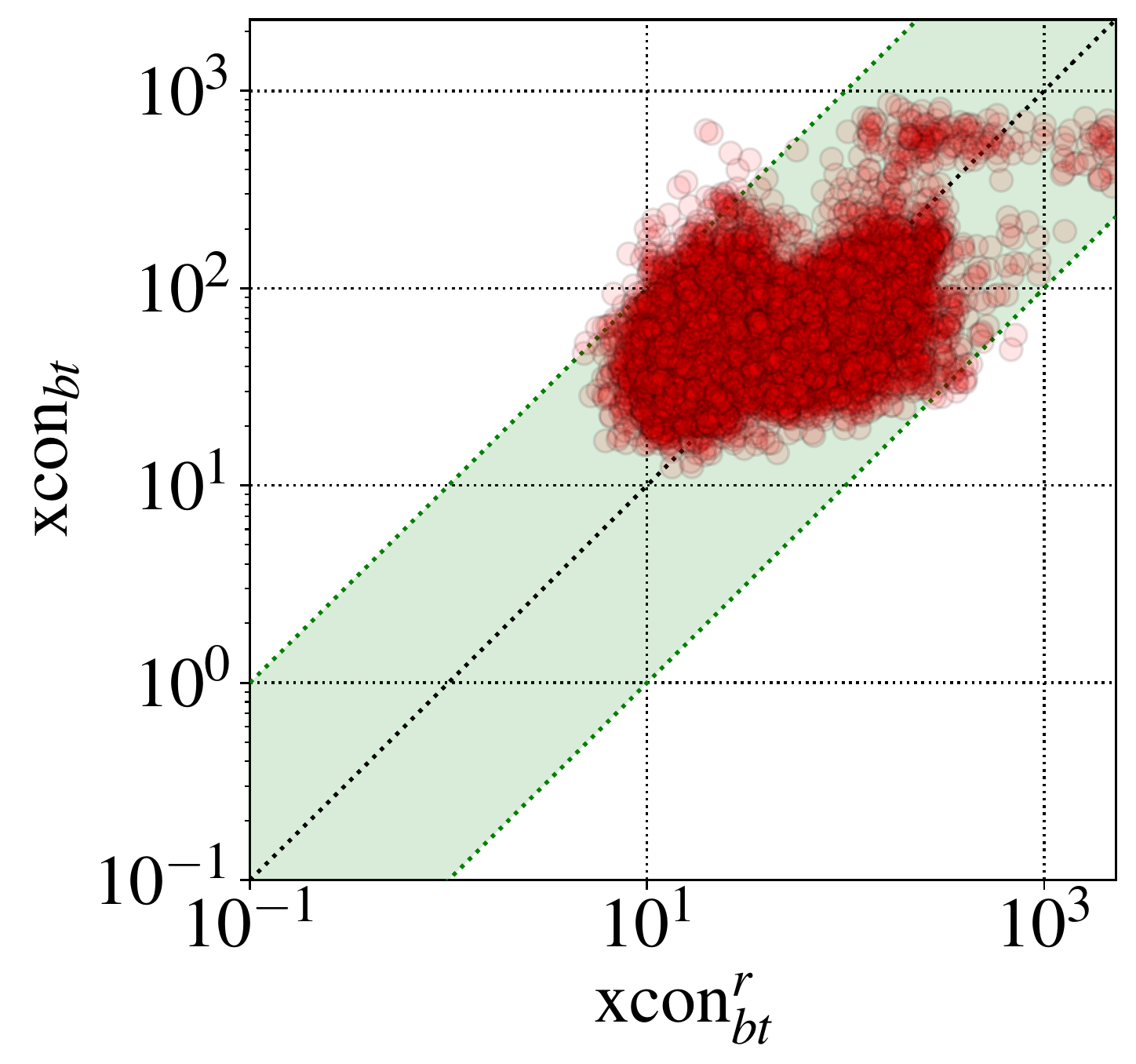}
    \caption{AXp runtime for BTs.}
    \label{fig:btaxprtime}
  \end{subfigure}
  \hfill
  \begin{subfigure}[b]{0.244\textwidth}
    \centering
    \includegraphics[width=\textwidth]{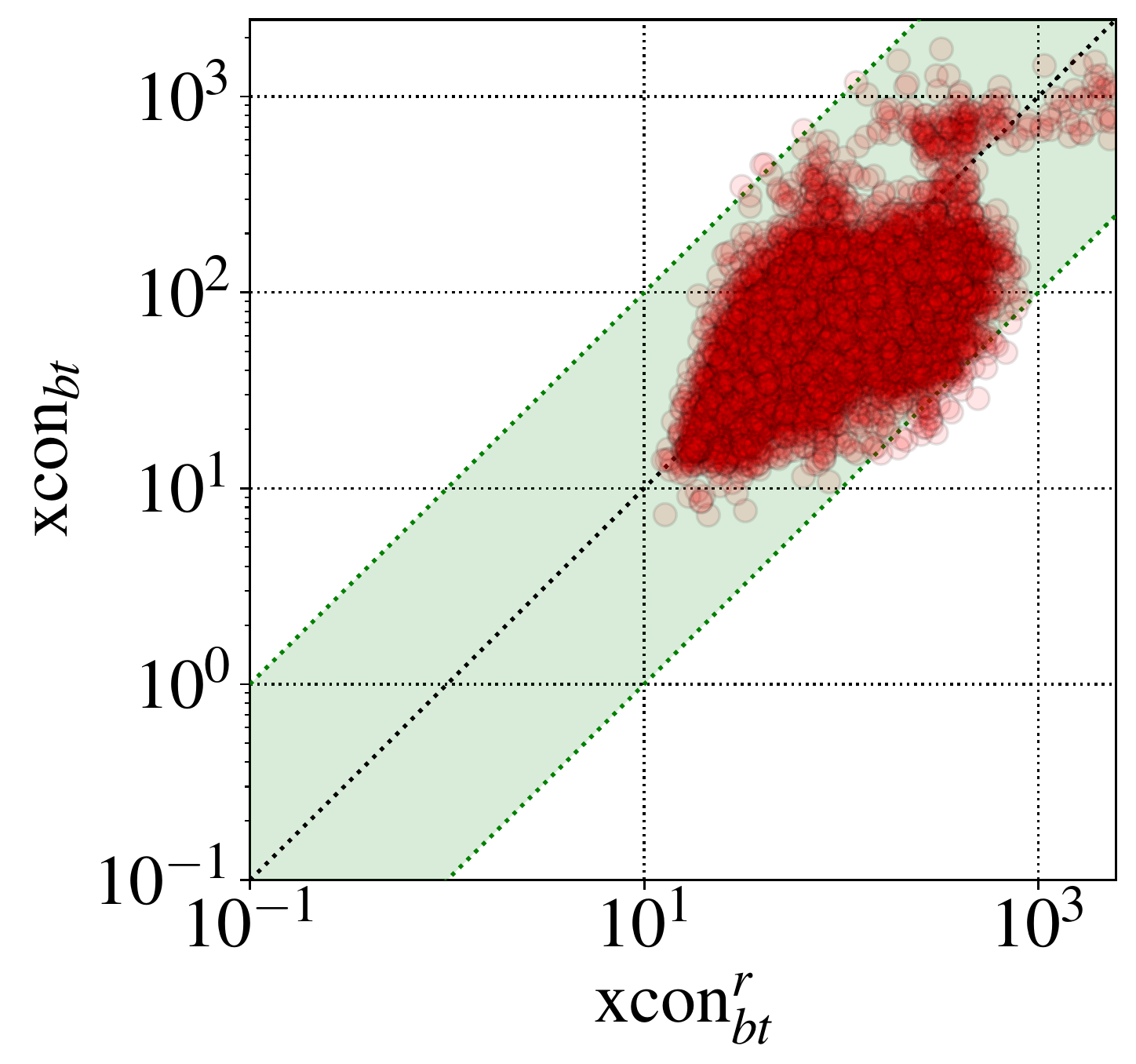}
    \caption{CXp runtime for BTs.}
    \label{fig:btcxprtime}
  \end{subfigure}
  \hfill
  \begin{subfigure}[b]{0.23\textwidth}
    \centering
    \includegraphics[width=\textwidth]{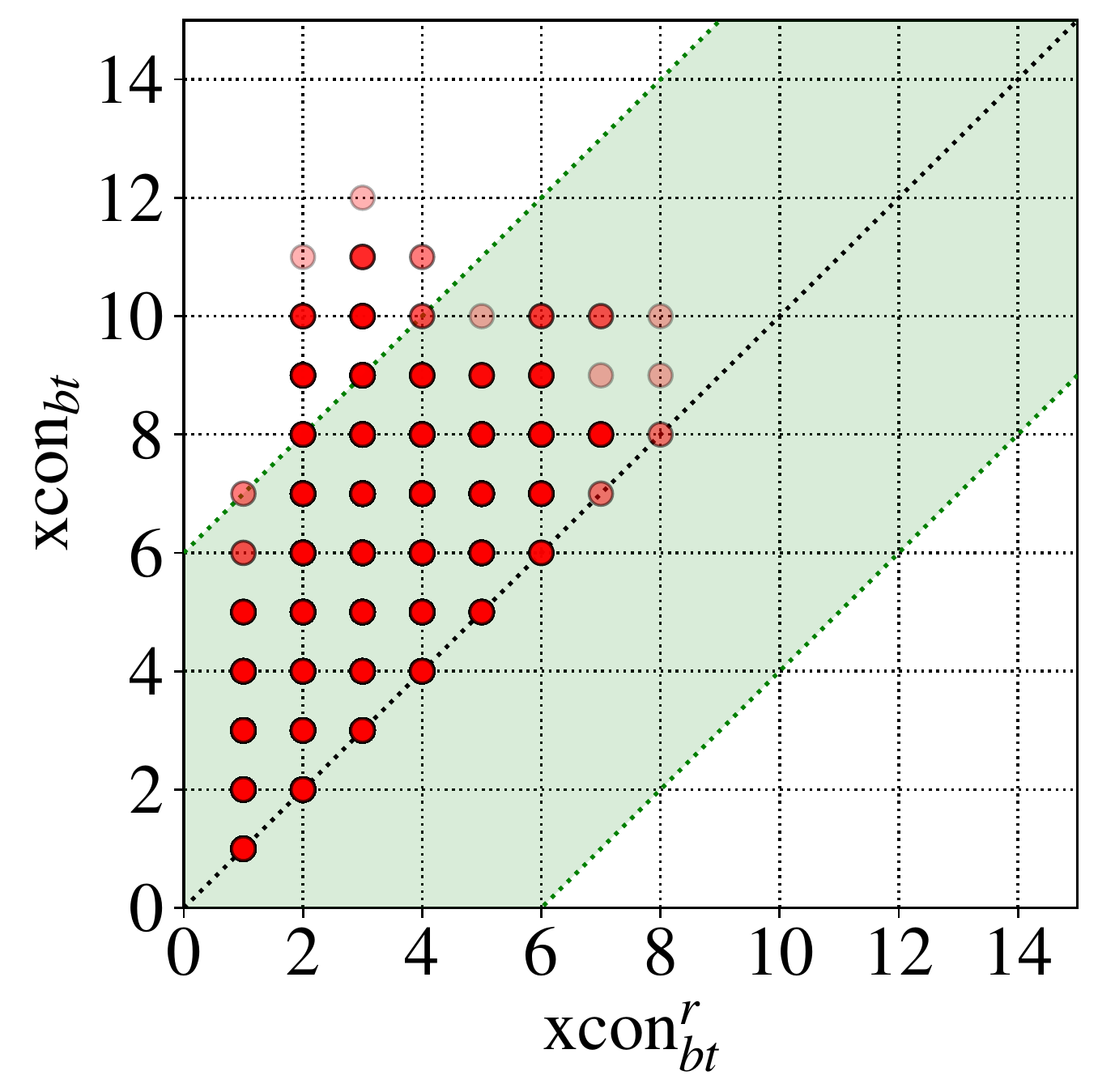}
    \caption{AXp size for BTs.}
    \label{fig:btaxp}
  \end{subfigure}%
  \hfill
  \begin{subfigure}[b]{0.23\textwidth}
    \centering
    \includegraphics[width=\textwidth]{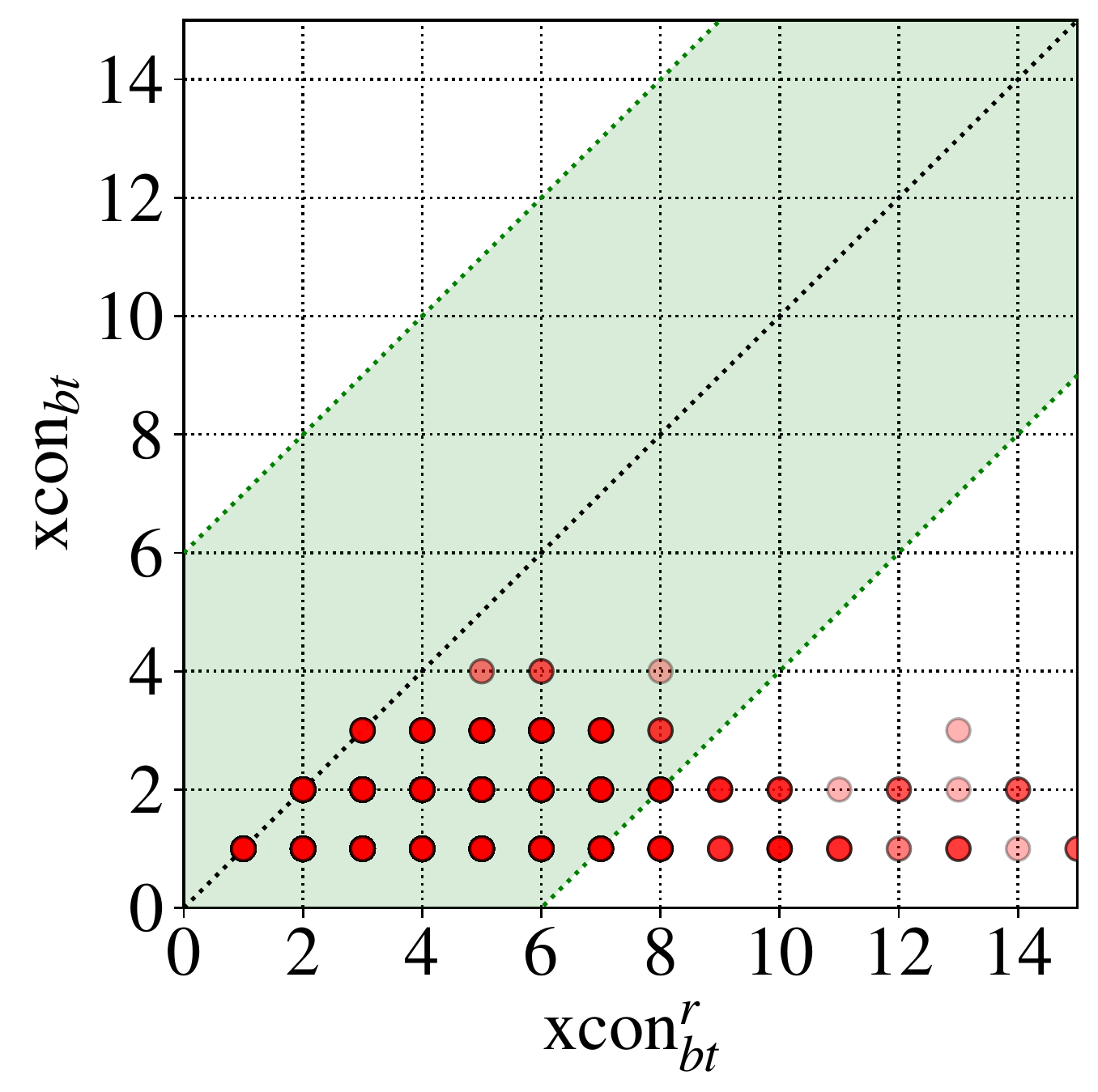}
    \caption{CXp size for BTs.}
    \label{fig:btcxp}
  \end{subfigure}
  \caption{Impact of \emph{xcon} rules on runtime~(ms) and explanation size for BTs.}
  \label{fig:btxprtime}
\end{figure*}

\begin{figure*}[!ht]
	\begin{subfigure}[b]{0.244\textwidth}
    \centering

    \includegraphics[width=\textwidth]{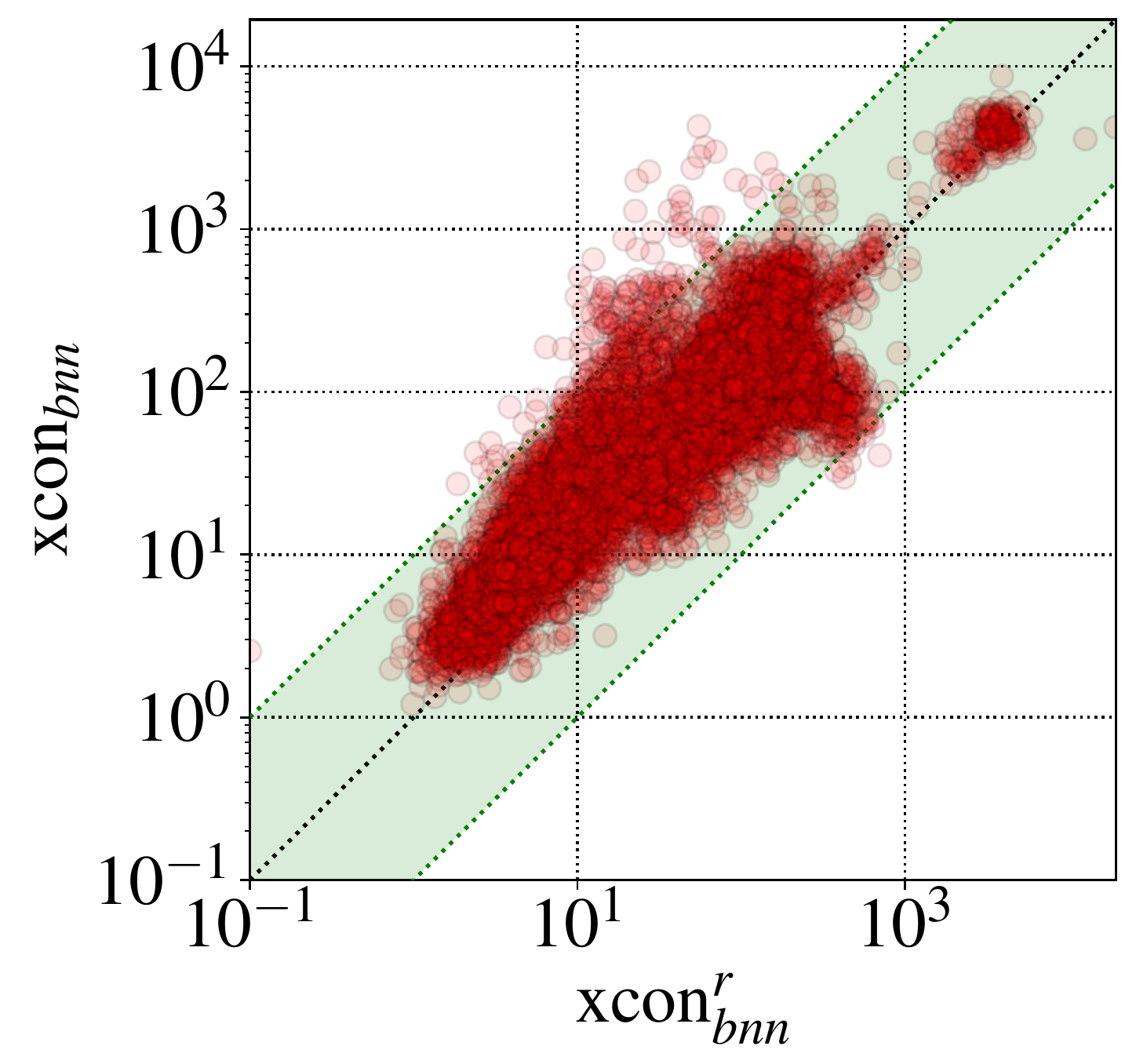}
    \caption{AXp runtime for BNNs.}
    \label{fig:bnnaxprtime}
  \end{subfigure}
  \hfill
  \begin{subfigure}[b]{0.244\textwidth}
    \centering
    \includegraphics[width=\textwidth]{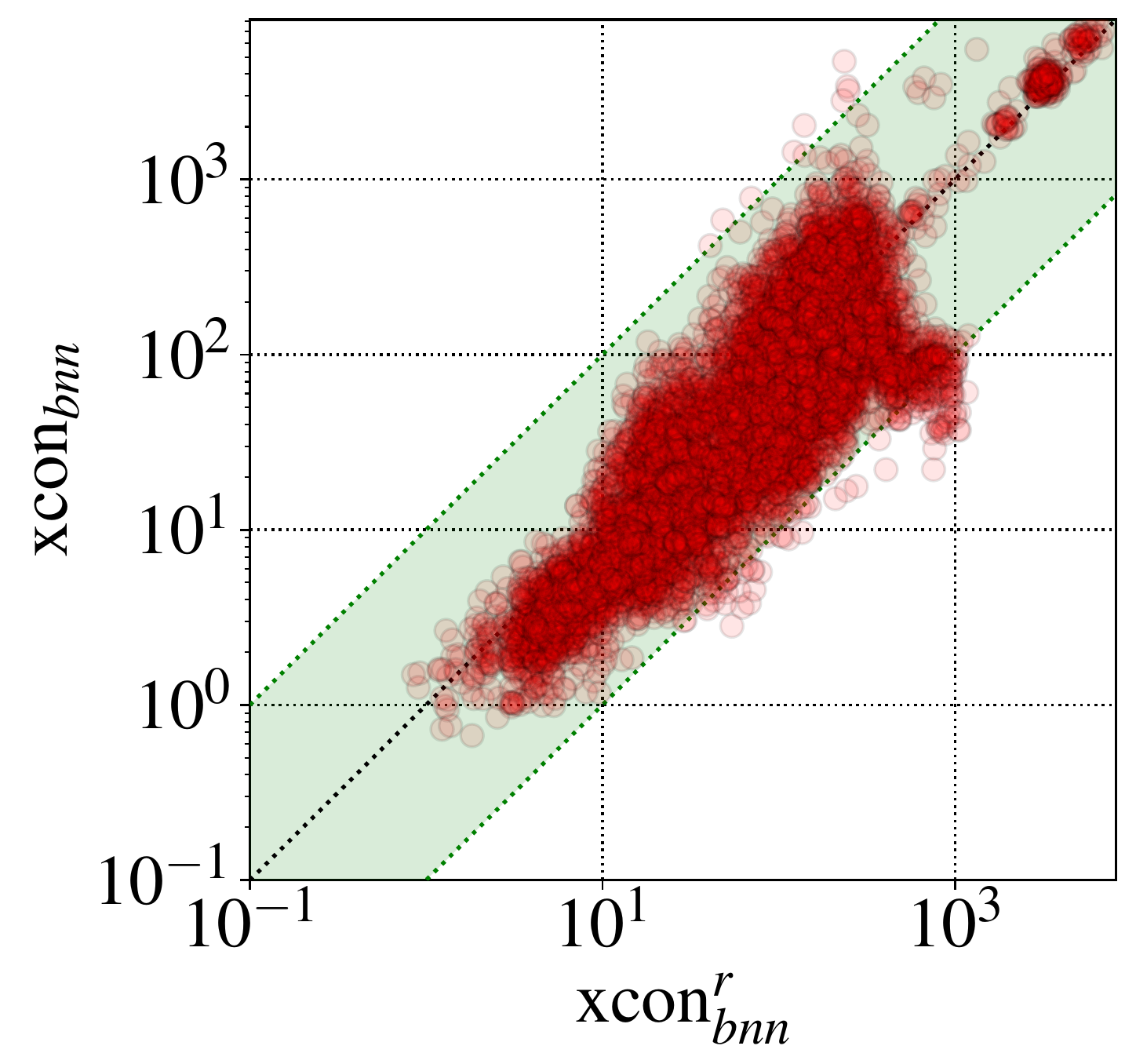}
    \caption{CXp runtime for BNNs.}
    \label{fig:bnncxprtime}
  \end{subfigure}
  \hfill
  \begin{subfigure}[b]{0.23\textwidth}
    \centering
    \includegraphics[width=\textwidth]{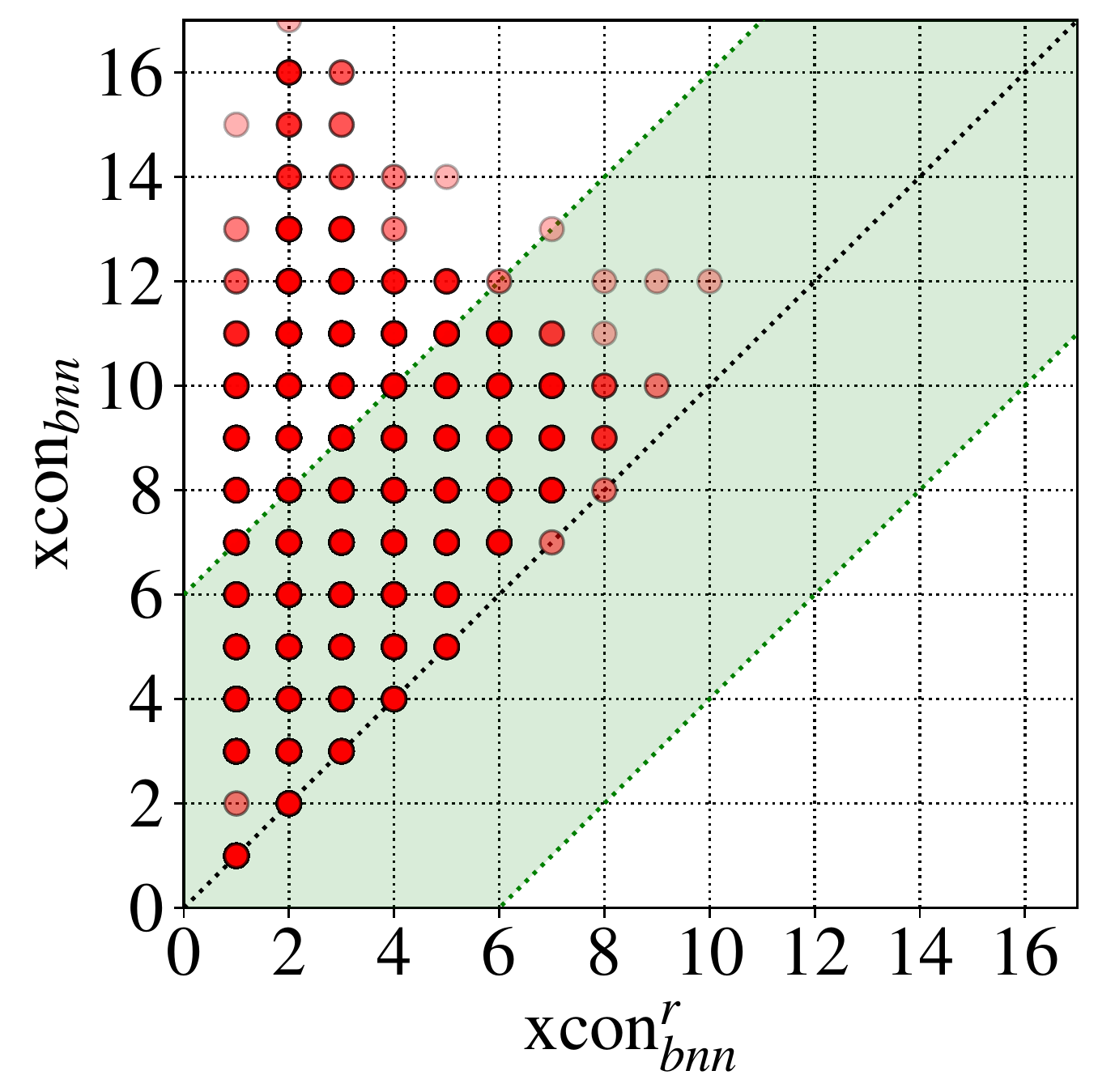}
    \caption{AXp size for BNNs.}
    \label{fig:bnnaxp}
  \end{subfigure}%
  \hfill
  \begin{subfigure}[b]{0.23\textwidth}
    \centering
    \includegraphics[width=\textwidth]{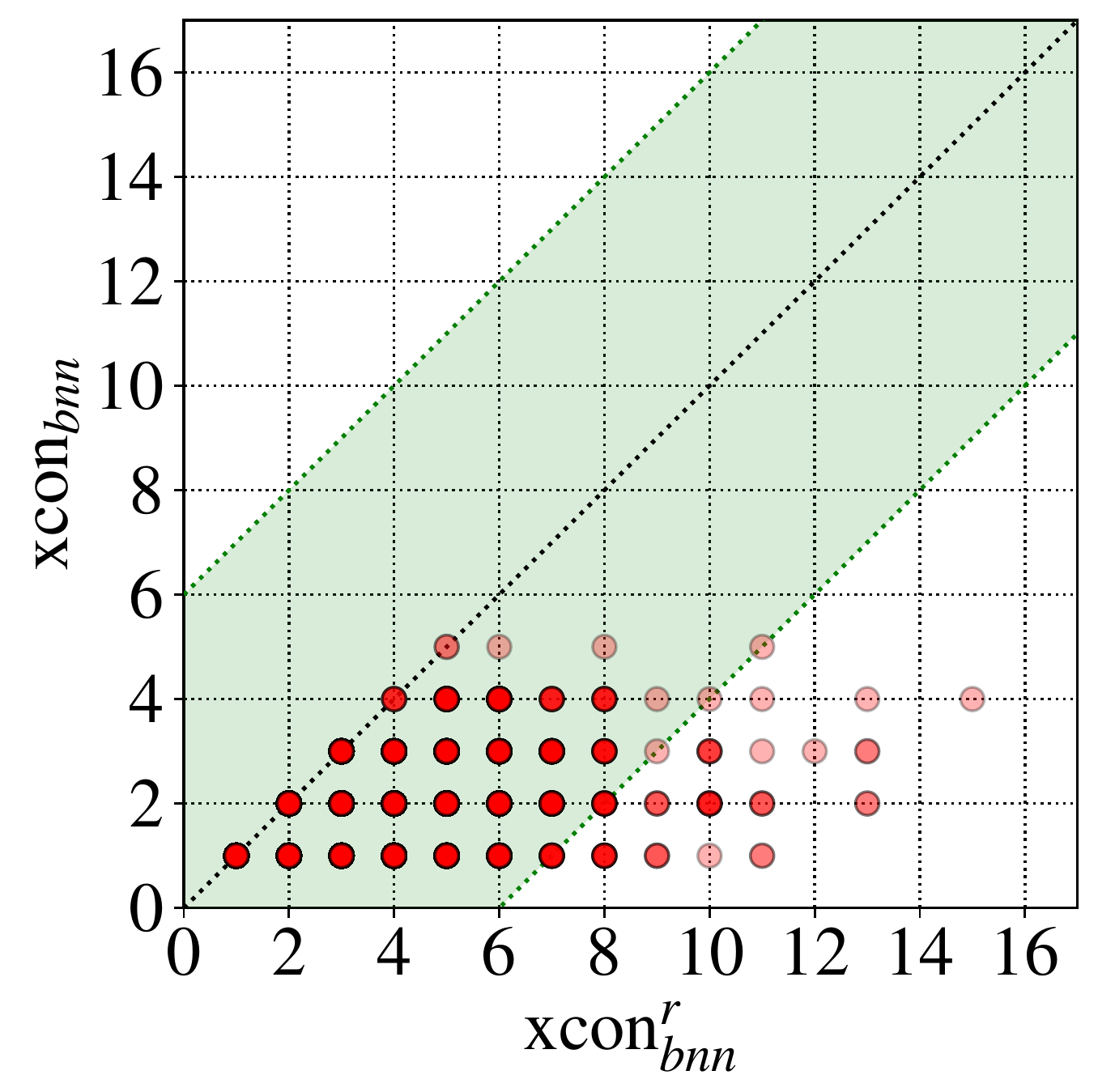}
    \caption{CXp size for BNNs.}
    \label{fig:bnncxp}
  \end{subfigure}
  \caption{Impact of \emph{xcon} rules on runtime~(ms) and explanation size for BNNs.}
  \label{fig:bnnxprtime}
\end{figure*}

\begin{table}[t!]
\centering
\caption{Change of average minimum explanation size.}
\label{tab:exp}
\scalebox{0.84}{
\begin{tabular}{ccccccc}\toprule
	\multirow{2}{*}{Dataset} & \multirow{2}{*}{Feats}  & \multirow{2}{*}{Model} & \multicolumn{2}{c}{AXp Size} &
	\multicolumn{2}{c}{CXp Size}  \\ \cmidrule{4-7}
				 & & & Before & After & Before & After  \\ \midrule
 &  & DL & 7.46 & 3.65 & 1.00 & 1.60 \\
adult & 65 & BT & 5.02 & 2.84 & 1.10 & 2.13 \\
 &  & BNN & 7.51 & 3.00 & 1.40 & 2.15 \\\midrule
 &  & DL & 5.65 & 3.74 & 1.01 & 1.15 \\
compas & 16 & BT & 3.91 & 3.09 & 1.06 & 1.15 \\
 &  & BNN & 4.40 & 2.79 & 1.19 & 1.30 \\\midrule
 &  & DL & 5.30 & 4.30 & 1.00 & 1.41 \\
lending & 35 & BT & 1.99 & 1.80 & 1.00 & 2.04 \\
 &  & BNN & 4.36 & 2.49 & 1.35 & 1.90 \\\midrule
 &  & DL & 9.51 & 5.58 & 1.00 & 1.23 \\
recidivism & 29 & BT & 6.04 & 4.04 & 1.17 & 1.67 \\
 &  & BNN & 7.01 & 4.01 & 1.42 & 1.82 \\\bottomrule
\end{tabular}
}
\end{table}

\myparagraph{Scalability.}
The scatter plots in
Figures~\ref{fig:dlaxprtime},~\ref{fig:dlcxprtime},
~\ref{fig:btaxprtime},~\ref{fig:btcxprtime},~\ref{fig:bnnaxprtime} and
\ref{fig:bnncxprtime} depict the comparison of the average runtime of
computing a single AXp or CXp for an instance (taken across all the 20
explanations computed) between \emph{xcon}$_\ast$ and
\emph{xcon}$_\ast^r$.
Clearly, for all the 3 models, the use of background knowledge
significantly improves the performance of AXp extraction (see
Figures~\ref{fig:dlaxprtime}, \ref{fig:btaxprtime}, and
\ref{fig:bnnaxprtime}).
At first glance, the performance of CXp extraction deteriorates
significantly in the case of DLs (Figure~\ref{fig:dlcxprtime}) if
compared to the other two models.
We should say that this impression is caused by a different scaling
used in Figure~\ref{fig:dlcxprtime} --- observe that CXp extraction is
1--3 orders of magnitude faster for DLs than for the other 2 models,
both when applying and not applying background knowledge.
Also, this can be explained by the fact that the average CXp size in
the case of DLs increases tremendously, which leads to a much larger
number of reasoning oracle calls when computing an explanation.
For BTs and BNNs, the use of background knowledge neither improves nor
degrades the computation of CXps (Figures~\ref{fig:btcxprtime} and
\ref{fig:bnncxprtime}), even though an increase of CXp size can be
also observed.

\myparagraph{Explanation Quality.}
The change of smallest size of AXps and CXps in an instance is shown
in Figure~\ref{fig:dlaxp}, \ref{fig:dlcxp}, \ref{fig:btaxp},
\ref{fig:btcxp}, \ref{fig:bnnaxp}, and \ref{fig:bnncxp}.
As can be seen in Figure~\ref{fig:dlaxp},\ref{fig:btaxp} and
\ref{fig:bnnaxp}, our results demonstrate how background knowledge, if
present, contributes to AXp size reduction across all models.
In particular, in many cases the size of a smallest AXp drops from 14
to 2, from 11 to 2, and from 17 to 2, for DLs, BTs, and BNNs,
respectively.
%
%
%
%
In contrast to AXps, Figures~\ref{fig:dlcxp}, \ref{fig:btcxp} and
\ref{fig:bnncxp} illustrate that the size of smallest CXps is
increased when background knowledge is applied.
Namely, in a number of cases the size of a smallest CXp jumps from 1
to 14, from 1 to 15, and from 2 to 13 for DLs, BTs, and BNNs,
respectively.
%
These results exemplify how easy it is to flip the prediction when no
background knowledge is present illustrating the potential
correctness issues for the corresponding CXps.

Table~\ref{tab:exp} details the change of the average size of smallest
AXps and CXps computed without or with background knowledge for DLs,
BTs and BNNs and for a selection of 4 publicly available datasets:
\emph{adult}, \emph{compas}, \emph{lending} and \emph{recidivism},
which were previously studied in the context of heuristic and formal
explanaitions.
%
(Here, all numeric features, if any, are quantized into 6 intervals.)
%
Note that Table~\ref{tab:exp} confirms the general observations made
that background knowledge triggers smaller AXps but larger CXps for
all the models studied.
The average size of smallest AXps in \emph{adult} and
\emph{recidivism} drops by around 4 for DLs and BNNs, while the
average smallest CXp size slightly increases for the two models.
In \emph{compas}, the size of smallest AXps in the three models
decreases by 1--2 and the size of smallest CXps subtly increases.
The size of smallest AXps in \emph{lending} drops by 1.00 in DLs,
0.19 in BTs, and 1.87 in BNNs.

\myparagraph{On Eclat-Assisted Explanations.}
The experiment above was repeated for the background knowledge
extracted with the use Eclat and its results are detailed below.
Namely, Figure~\ref{fig:adlxp}, Figure~\ref{fig:abtxp} and
Figure~\ref{fig:abnnxp} evaluate the proposed approach to computing
formal explanations for DLs, BTs, and BNNs, taking into account the
rules mined by Eclat.
Note that here only 58 datasets tackled by Eclat are considered.
Similar to \emph{xcon}$_{\ast}$ above, \emph{eclat}$_{\ast} \in
\{dl, bt, bnn\}$ represents the formal explanation approach applied to
DL, BT, and BNN models, respectively.
Moreover, \emph{eclat}$^r_{\ast}$ is used to denote the
configurations that apply background knowledge extracted by Eclat.
Scalability-wise and in contrast to the case of \emph{xcon}, where the
performance of AXps computation improves and the performance of CXp
computation degrades in the presence of background knowledge extracted
by the MaxSAT approach, the use of Eclat-provided background
knowledge degrades the performance of CXp generation for DLs as well
as both AXp and CXp computation for BTs.
This can be explained by the larger number of rules extracted by
Eclat compared to the MaxSAT approach.
In terms of the quality of explanations, observations similar to the
case of \emph{xcon} can be made, i.e.\ background knowledge extracted
by Eclat can trigger AXp size reduction across all the 3 models,
while the size of CXps increases due to the background knowledge.

\begin{figure*}[!ht]
	\centering
	\begin{subfigure}[b]{0.248\textwidth}
    \centering
    \includegraphics[width=\textwidth]{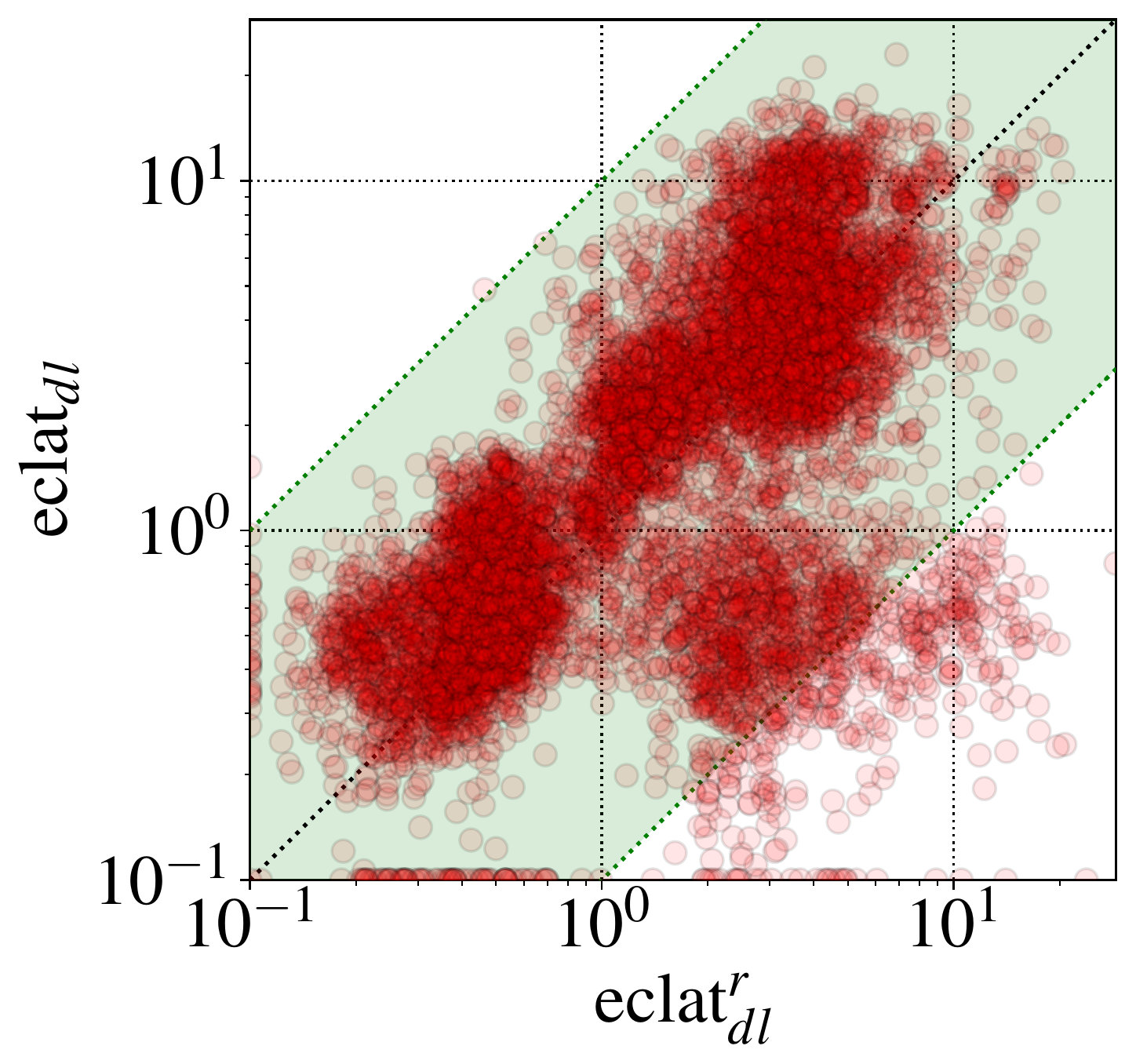}
    \caption{AXp runtime for DLs.}
    \label{fig:adlaxprtime}
  \end{subfigure}
  \hfill
  \begin{subfigure}[b]{0.248\textwidth}
    \centering
    \includegraphics[width=\textwidth]{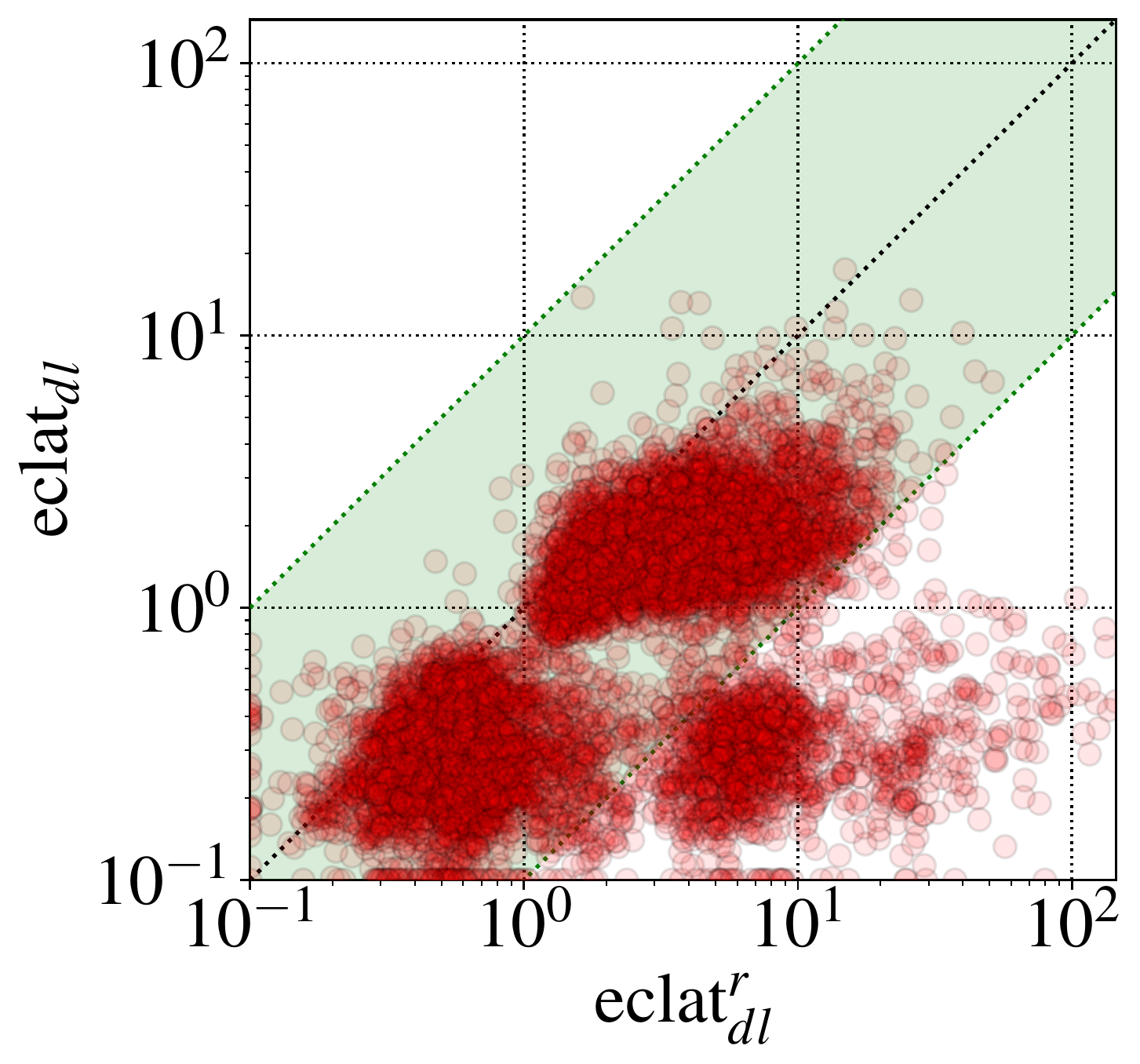}
    \caption{CXp runtime for DLs.}
    \label{fig:adlcxprtime}
  \end{subfigure}
  \hfill
  \begin{subfigure}[b]{0.234\textwidth}
    \centering
    \includegraphics[width=\textwidth]{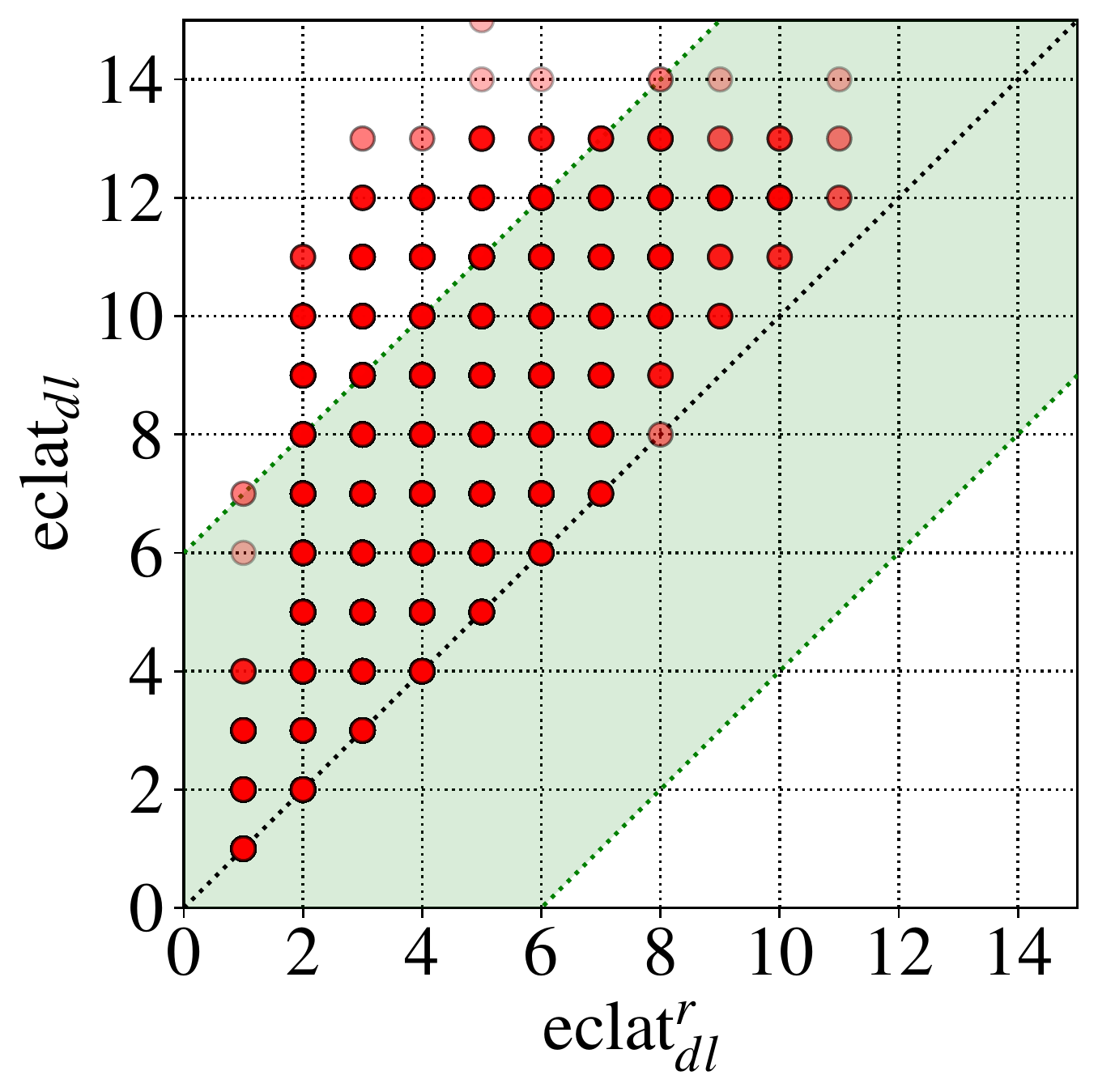}
    \caption{AXp size for DLs.}
    \label{fig:adlaxp}
  \end{subfigure}%
  \hfill
  \begin{subfigure}[b]{0.234\textwidth}
    \centering
    \includegraphics[width=\textwidth]{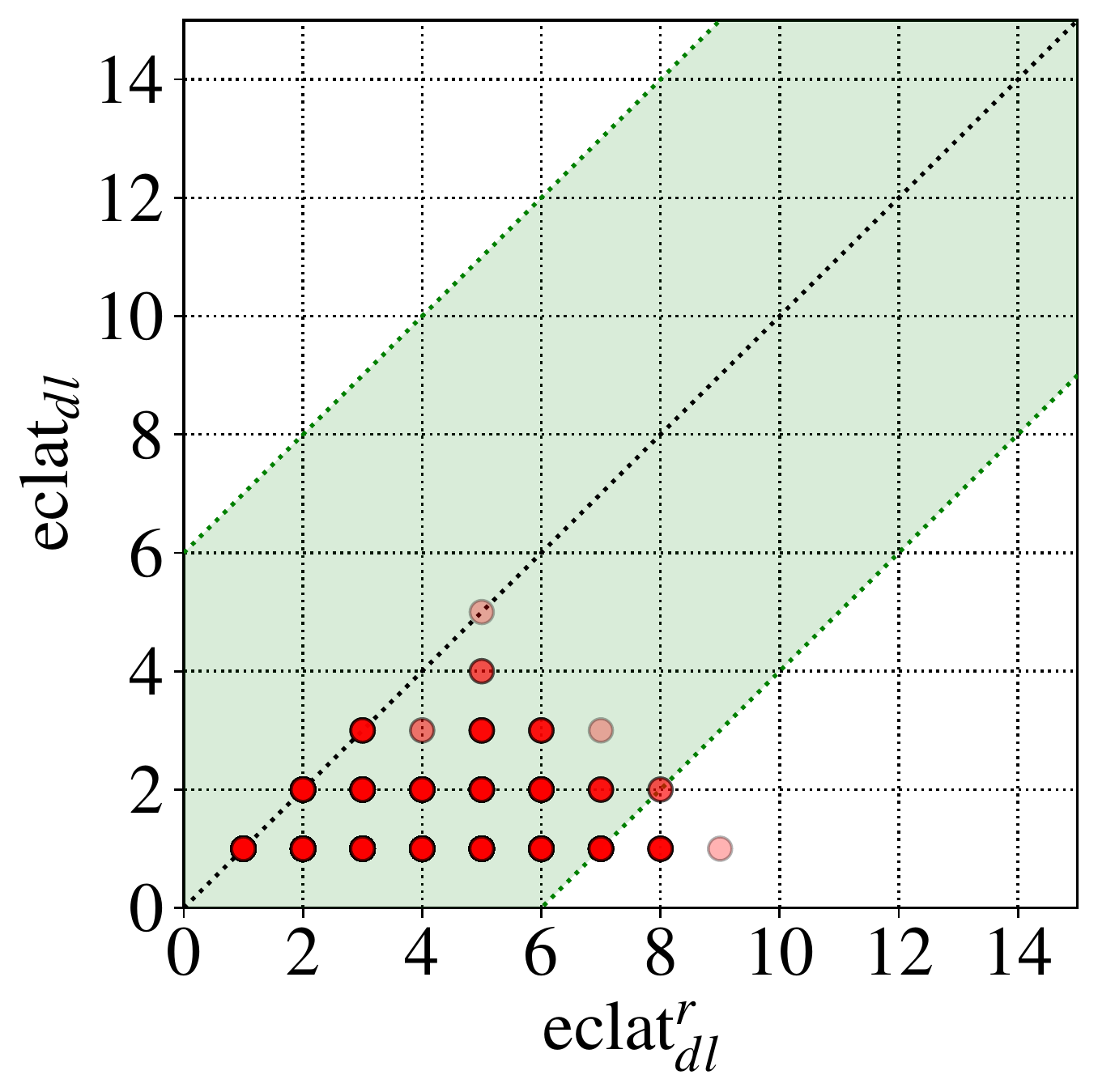}
    \caption{CXp size for DLs.}
    \label{fig:adlcxp}
  \end{subfigure}
  \caption{Impact of Eclat rules on runtime~(ms) and explanation size for DLs.}
  \label{fig:adlxp}
\end{figure*}

\begin{figure*}[!ht]
	\centering
	\begin{subfigure}[b]{0.244\textwidth}
    \centering
    \includegraphics[width=\textwidth]{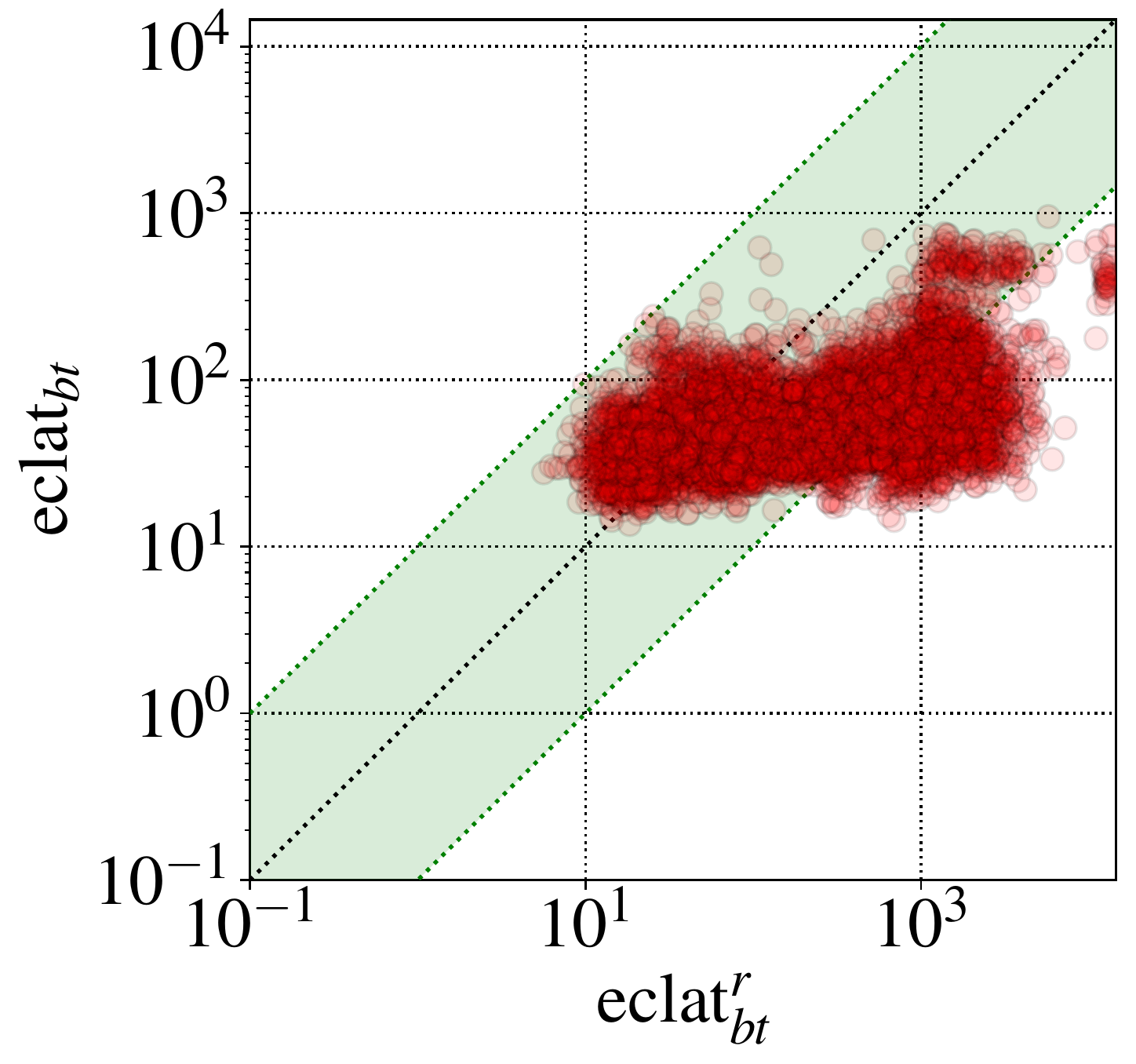}
    \caption{AXp runtime for BTs.}
    \label{fig:abtaxprtime}
  \end{subfigure}
  \hfill
  \begin{subfigure}[b]{0.244\textwidth}
    \centering
    \includegraphics[width=\textwidth]{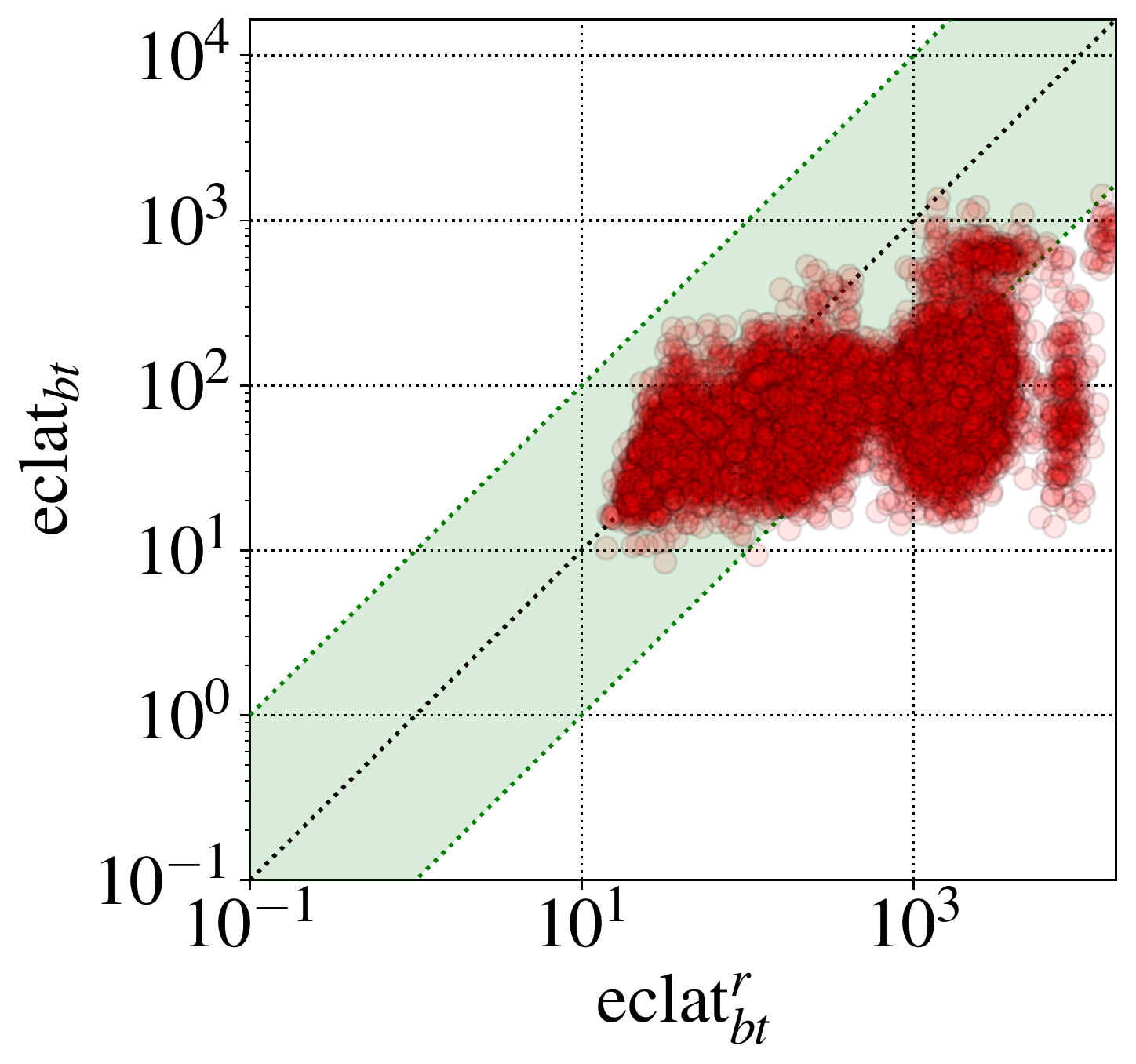}
    \caption{CXp runtime for BTs.}
    \label{fig:abtcxprtime}
  \end{subfigure}
  \hfill
  \begin{subfigure}[b]{0.237\textwidth}
    \centering
    \includegraphics[width=\textwidth]{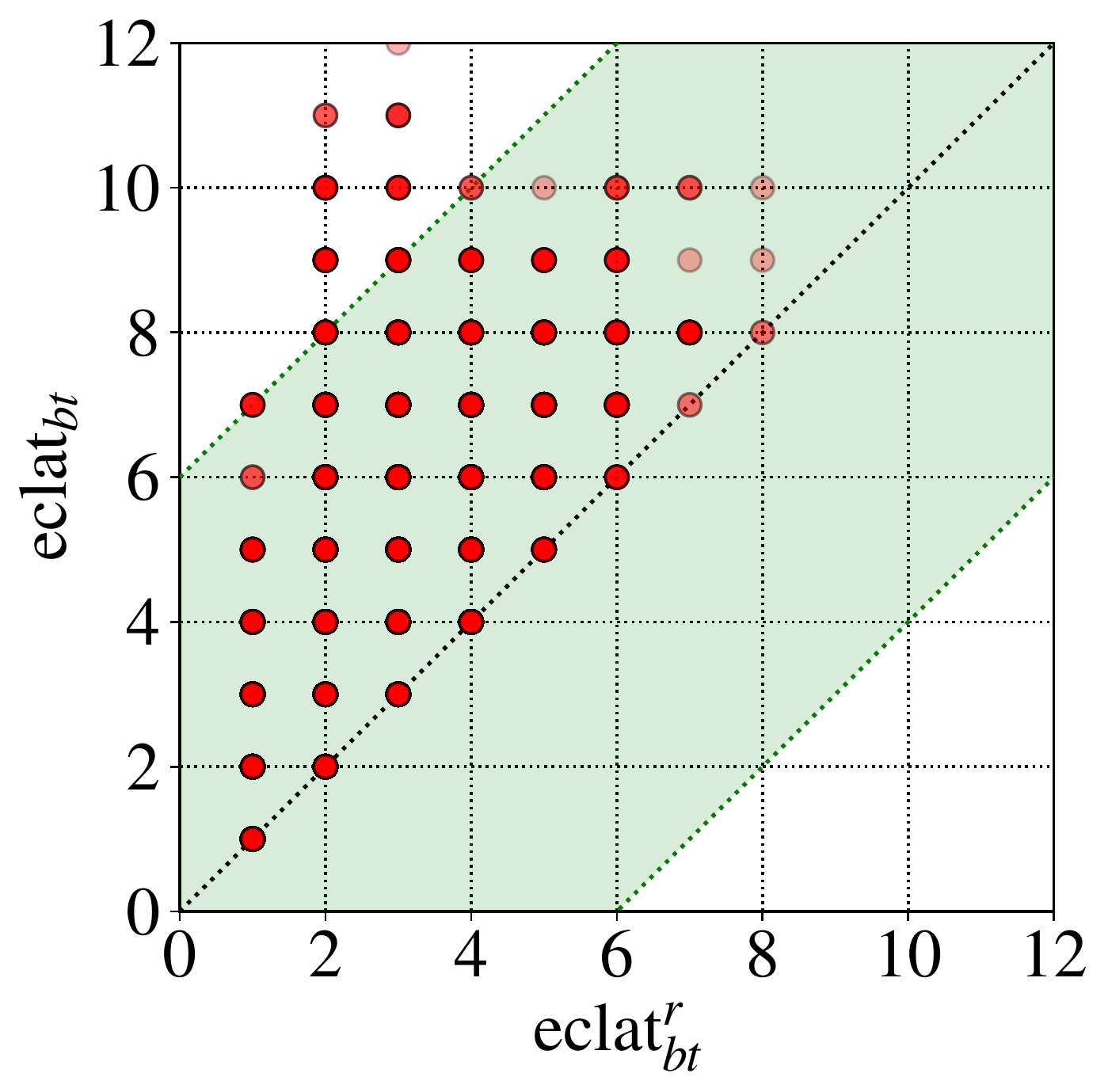}
    \caption{AXp size for BTs.}
    \label{fig:abtaxp}
  \end{subfigure}%
  \hfill
  \begin{subfigure}[b]{0.237\textwidth}
    \centering
    \includegraphics[width=\textwidth]{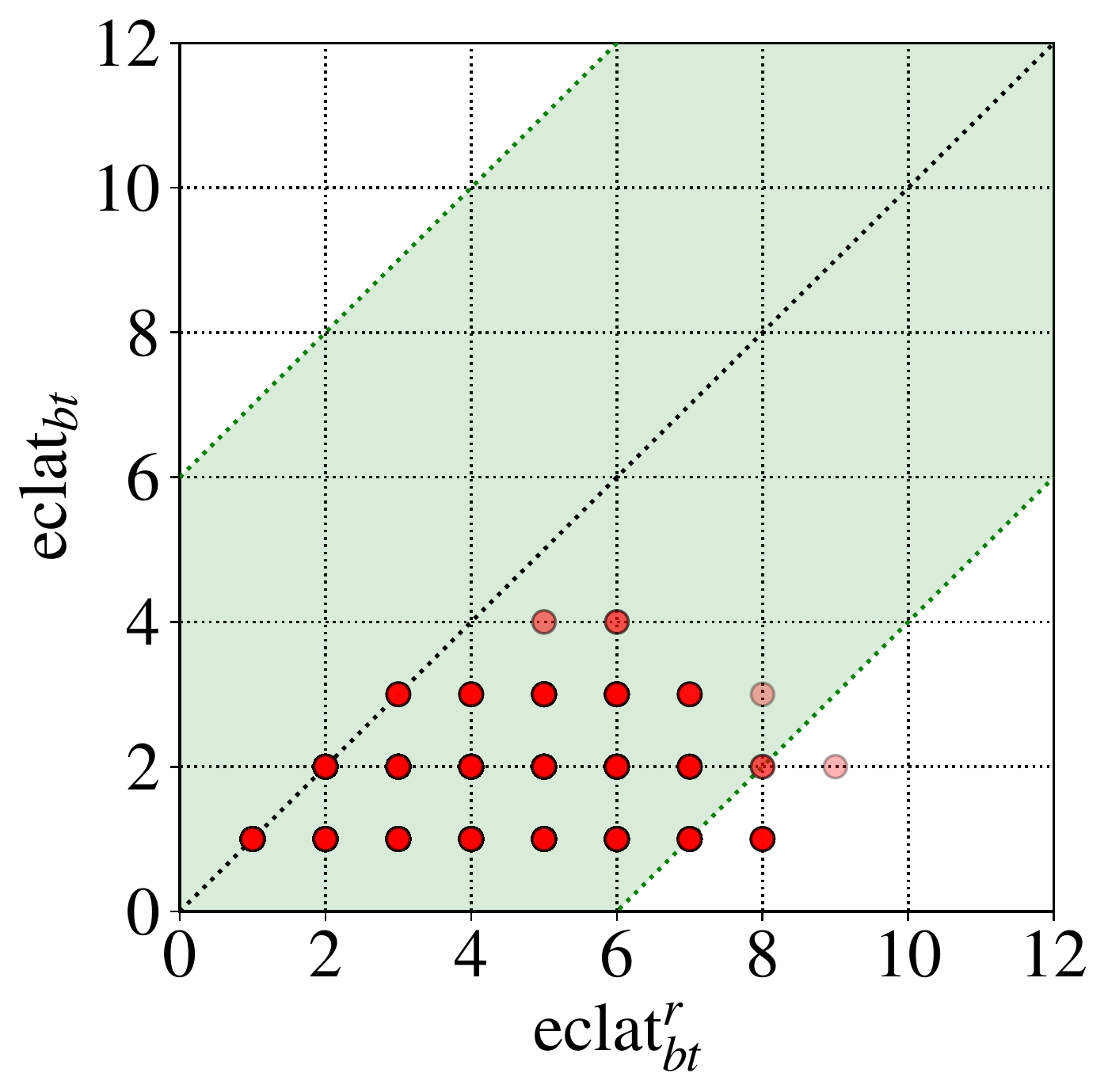}
    \caption{CXp size for BTs.}
    \label{fig:abtcxp}
  \end{subfigure}
  \caption{Impact of Eclat rules on runtime~(ms) and explanation size for BTs.}
  \label{fig:abtxp}
\end{figure*}

\begin{figure*}[!ht]
	\centering
	\begin{subfigure}[b]{0.244\textwidth}
    \centering
    \includegraphics[width=\textwidth]{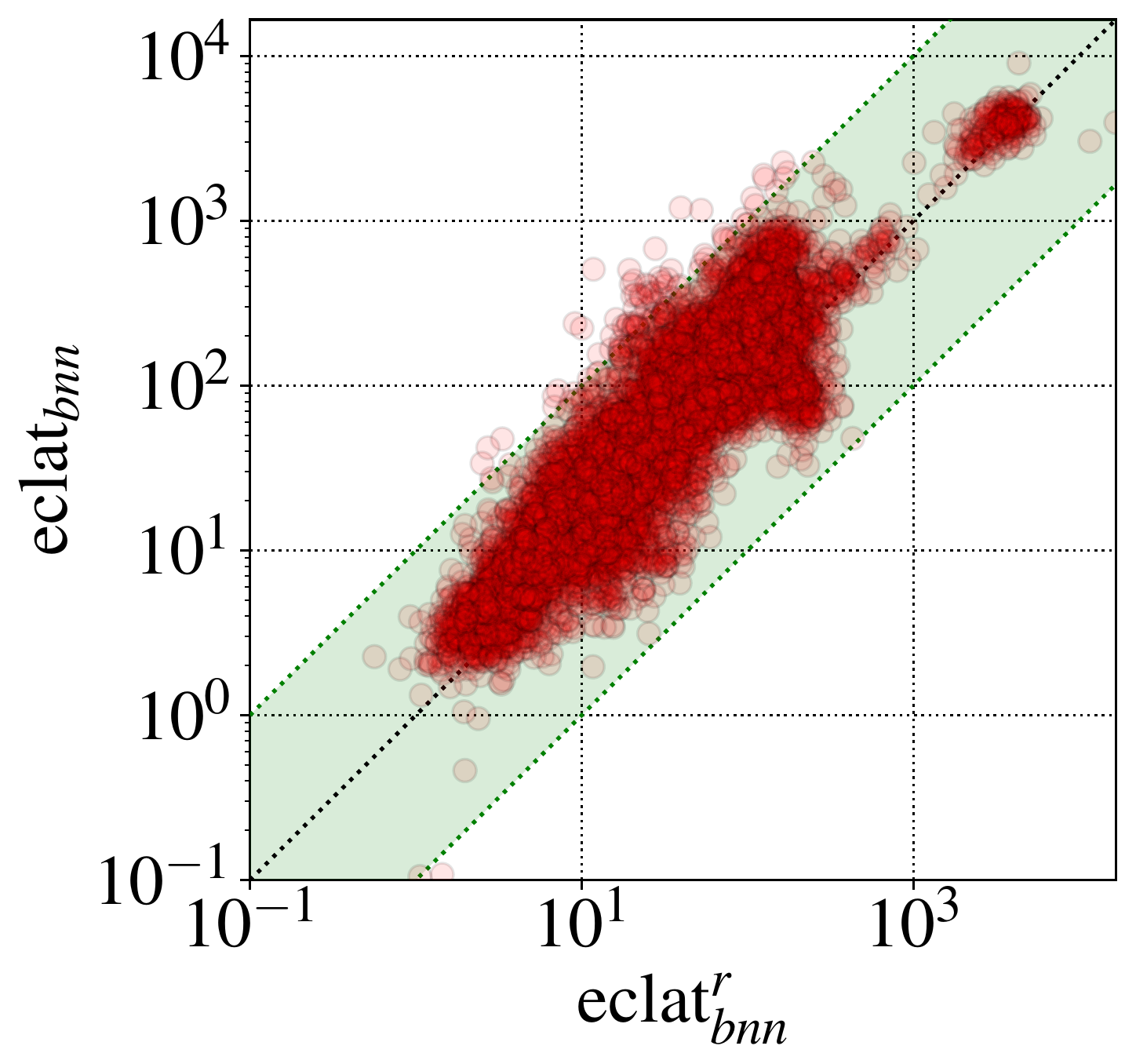}
    \caption{AXp runtime for BNNs.}
    \label{fig:abnnaxprtime}
  \end{subfigure}
  \hfill
  \begin{subfigure}[b]{0.244\textwidth}
    \centering
    \includegraphics[width=\textwidth]{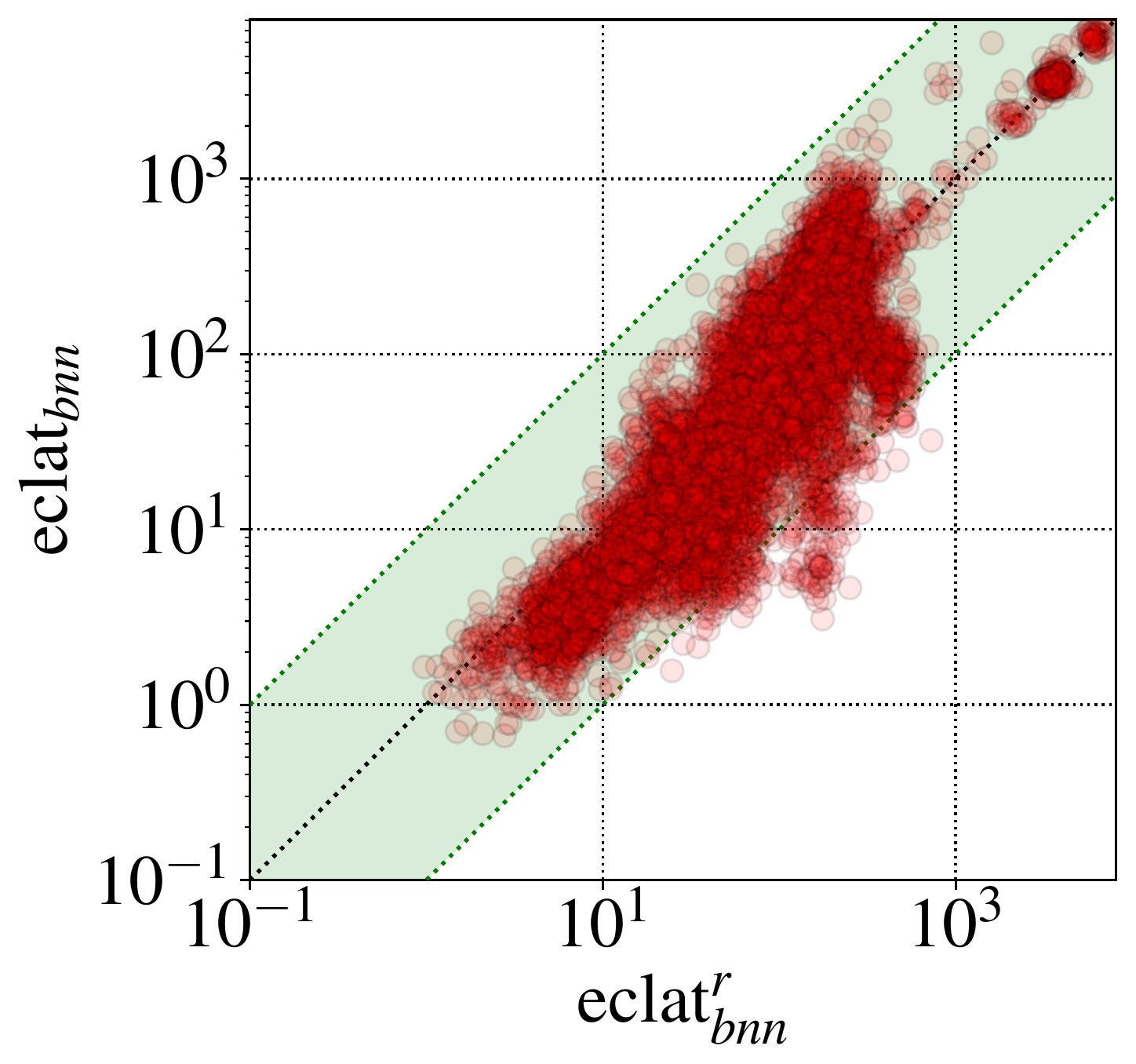}
    \caption{CXp runtime for BNNs.}
    \label{fig:abnncxprtime}
  \end{subfigure}
  \hfill
  \begin{subfigure}[b]{0.237\textwidth}
    \centering
    \includegraphics[width=\textwidth]{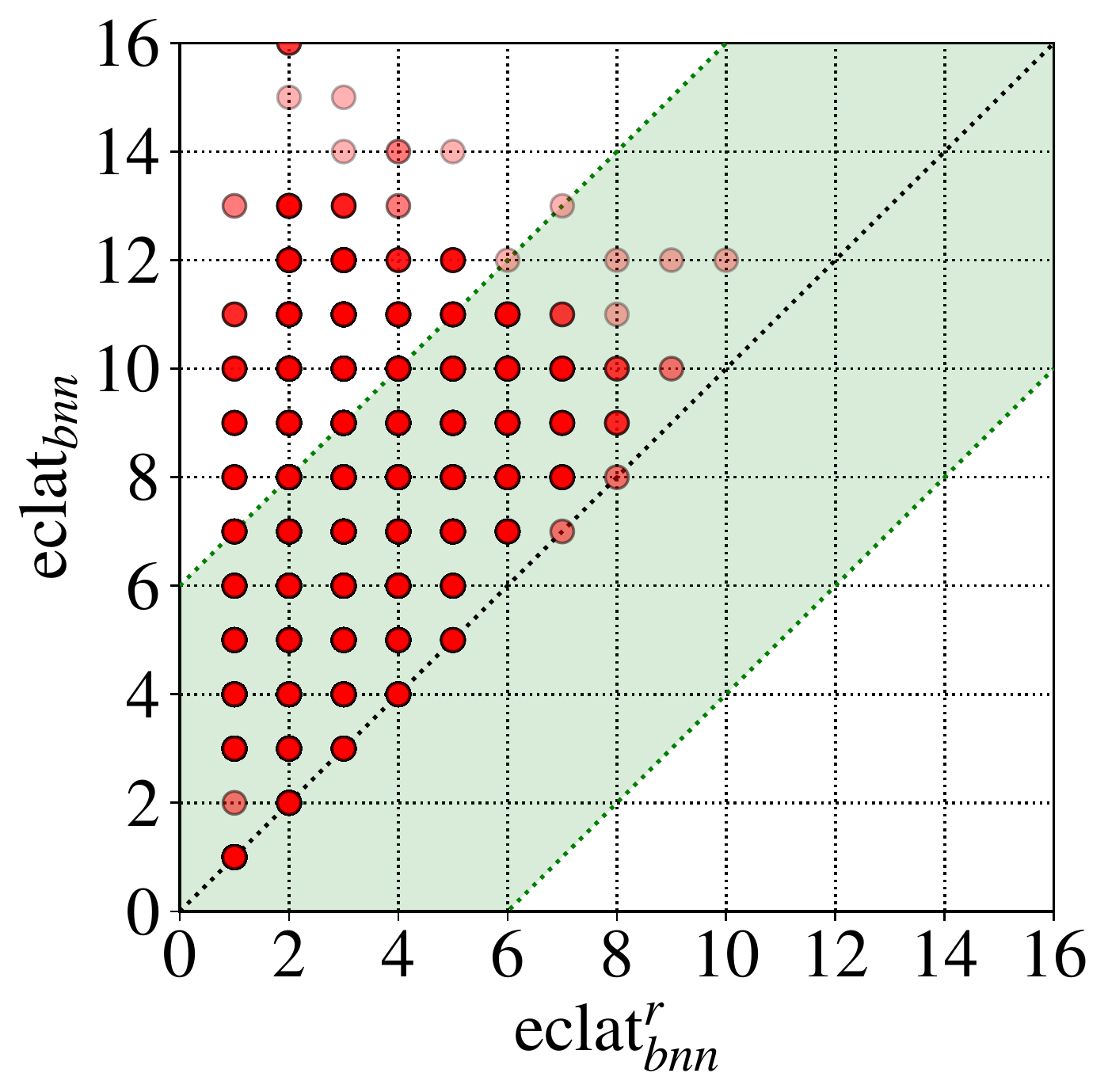}
    \caption{AXp size for BNNs.}
    \label{fig:abnnaxp}
  \end{subfigure}%
  \hfill
  \begin{subfigure}[b]{0.237\textwidth}
    \centering
    \includegraphics[width=\textwidth]{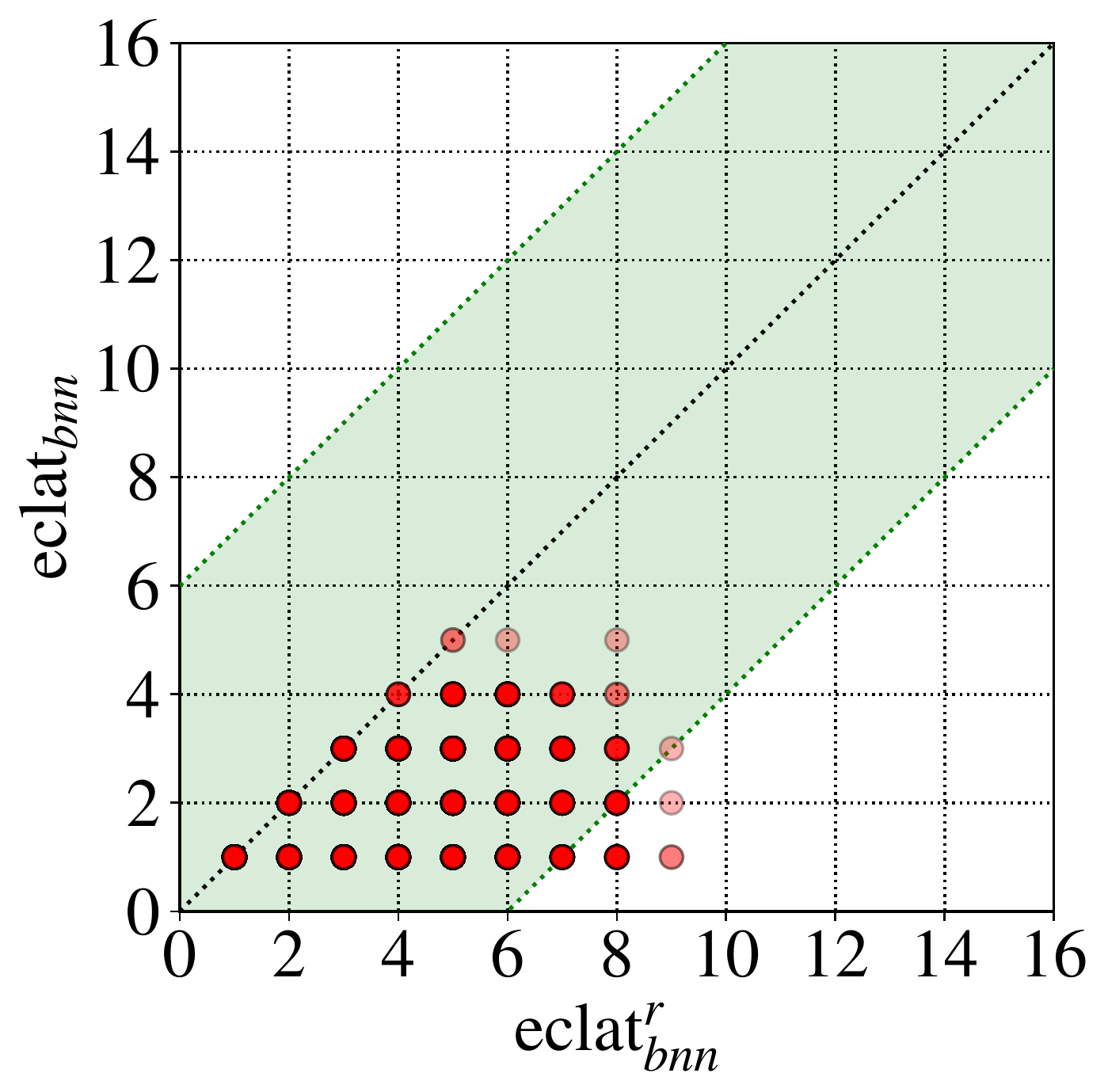}
    \caption{CXp size for BNNs.}
    \label{fig:abnncxp}
  \end{subfigure}
  \caption{Impact of Eclat rules on runtime~(ms) and explanation size for BNNs.}
  \label{fig:abnnxp}
\end{figure*}

\subsection{Usefulness of Background Knowledge} \label{sec:rsize}

\begin{table}[t!]
\centering
\caption{Size distribution of used rules.}
\label{tab:usedr}
\scalebox{0.72}{
\begin{tabular}{cccccccccc}\toprule
	\multirow{2}{*}{Dataset} & \multirow{2}{*}{Feats} & \multirow{2}{*}{Model} & \multicolumn{7}{c}{Distribution (\%)} \\ 
				 & & & 1 & 2 & 3 & 4 & 5 & 6 & 7 \\ \midrule
 &  & DL & 10.9 & 17.2 & 37.7 & 21.9 & 8.4 & 3.1 & 0.8 \\
adult & 65 & BT & 7.3 & 10.0 & 39.5 & 30.2 & 10.1 & 2.4 & 0.4 \\
 &  & BNN & 9.8 & 11.5 & 39.6 & 26.8 & 9.0 & 2.7 & 0.5 \\\midrule
 &  & DL & 55.4 & 17.4 & 22.3 & 3.3 & 0.2 & 1.4 & $-$ \\
compas & 16 & BT & 53.2 & 29.0 & 16.1 & 1.0 & 0.1 & 0.6 & $-$ \\
 &  & BNN & 41.4 & 27.3 & 27.2 & 2.9 & 0.9 & 0.3 & $-$ \\\midrule
 &  & DL & 43.4 & 4.1 & 3.3 & 18.7 & 20.2 & 9.3 & 1.1 \\
lending & 35 & BT & 41.7 & 7.6 & 4.5 & 13.3 & 23.2 & 9.1 & 0.7 \\
 &  & BNN & 36.2 & 3.6 & 3.5 & 21.3 & 24.6 & 9.5 & 1.2 \\\midrule
 &  & DL & 2.9 & 1.5 & 9.1 & 25.8 & 28.6 & 20.3 & 8.7 \\
recidivism & 29 & BT & 2.1 & 1.5 & 8.1 & 25.8 & 30.9 & 20.7 & 8.4 \\
 &  & BNN & 1.6 & 1.4 & 7.5 & 24.1 & 36.6 & 18.7 & 7.4 \\\bottomrule
\end{tabular}
}
\end{table}

To assess rules' contribution into explanation extraction, we applied
the setup of Section~\ref{sec:resexp}, i.e.\ we enumerated at most
20 smallest size AXps for each test instance.\footnote{This
experiment is conducted only for the proposed MaxSAT-based approach
for knowledge extraction.}
%
Table~\ref{tab:usedr} presents the evaluation of which rules contribute
to AXp size reduction for DLs, BTs, and BNNs for the same selection of
datasets studied in Table~\ref{tab:exp}, i.e.\ \emph{adult},
\emph{compas}, \emph{lending} and \emph{recidivism}.
However and in contrast to the previous experiment, the rules here are
exhaustively extracted for each of the datasets, i.e. \emph{no
extraction limit} is applied.
This resulted in extracting rules up to size 7.

Our experimental results indicate that rules of size greater than 5
are not frequently used when computing AXps.
Table~\ref{tab:usedr} shows that the size of more than 98\% of the
useful rules in the three models for \emph{compas} ranges from 1 to
4.
For \emph{adult}, more than 95\% of the useful rules comprise 1 to
5 literals.
The rules of size 5 are significant in the case of \emph{lending}
and \emph{recidivism}, where more than 20\% of the useful rules
contain 5 literals.
However, there are 29.0\%, 29.1\%, and 26.1\% of the useful
rules larger than size 5 for \emph{recidivism} in DLs, BT, and BNNs,
respectively, while less than 11\% of the useful rules contain more
than 5 literals for the other 3 datasets.
These results support our choice of value 5 as the extraction limit
since the size of the vast majority of useful rules is no more than 5.
%

\ignore{
\paragraph{Rules to Reduce an AXp.}
Figure~\ref{fig:ruse} depicts the average number of rules used to
reduce the size of an AXp for
an instances in $72$ datasets in the $3$ ML models.
The total number of instances is $4897$.
The number is calculated as $\frac{|R_u|}{|AXps_{r}|}$,
where $|R_u|$ is the number of rules used in an instance
and $|AXps_{r}|$ is the number of AXps that can be reduced by the rules.
Note that $0$ in indicates that there are no AXps that can be reduced in an instance.
The rules extracted in the rule extracting process are restricted to at most size $5$.
The maximum number of rules used for an AXp in an instance
in DLs and BTs is between $105$ as well as $112$ respectively,
while $158$ is the highest number of rules used in BNNs.
As can be observed, the number of rules used is no more than $10$ for
$2696$ instances in DLs, $3006$ instances in BTs and $1753$ instances in BNNs.

\begin{figure}[!h]
    \centering
    \includegraphics[width=0.49\textwidth]{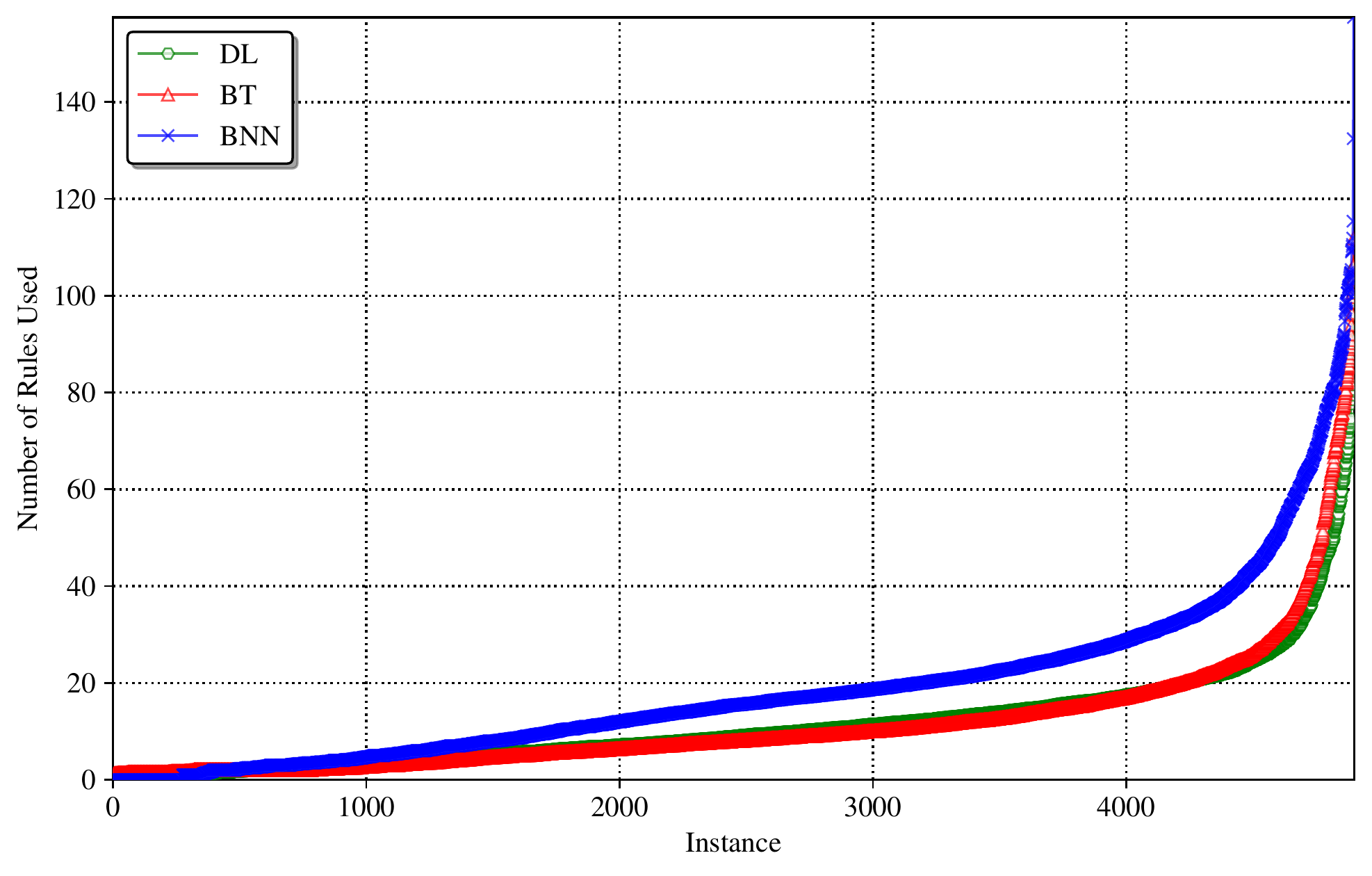}
    \caption{Rules Used per AXp.}
    \label{fig:ruse}
 \end{figure}%

\pnote{We dont talk about reporting the background knowledge used to the user!}
}

\subsection{Formal vs. Heuristic Explanations}\label{res:hres}

Following~\cite{inms-aaai19,nsmims-sat19,ignatiev-ijcai20},
we apply formal explanations to assess the runtime and
explanation quality for the heuristic approaches LIME, SHAP, and Anchor.
The idea is to show the importance of trustable background knowledge
when targeting a more accurate quality assessment.

%

\paragraph{Scalability.}
Figure~\ref{fig:hrtime} and Table~\ref{tab:hrtime} illustrate the runtime of a single explanation
extraction for a data instance across the 62 datasets performed by
LIME, SHAP, Anchor, \emph{xcon}$_\ast$, and \emph{xcon}$^r_\ast$.
Here, \emph{xcon}$^r_{\ast}$ and \emph{xcon}$_{\ast}$ represent the
proposed approach to computing AXps or CXps with/without background
knowledge, s.t.\ $\ast\in\{axp, cxp\}$.
Observe that both \emph{xcon}$_\ast$ and \emph{xcon}$^r_\ast$
outperform LIME and Anchor for all the 3 models, explaining a data
instance in a fraction of a second.
LIME and Anchor are 1-2 orders of magnitude slower for DL and BNN
models, while LIME outweighs Anchor when generating explanations for
BTs.
The worst performance for DL and BNN models is demonstrated
by SHAP while, surprisingly, SHAP outperforms the other competitors
for BTs models.

\begin{table}[t!]
\centering
\caption{Average runtime per explanation.}
\label{tab:hrtime}
\scalebox{0.70}{
\setlength{\tabcolsep}{6pt}
\begin{tabular}{cS[table-format=3.0]S[table-format=3.0]S[table-format=3.0]S[table-format=3.0]S[table-format=5.0]S[table-format=6.0]S[table-format=5.0]}\toprule
\multirow{2}{*}{Model} & \multicolumn{7}{c}{Runtime per explanation (ms)} \\ \cmidrule{2-8}
		       & xcon$_{axp}$ & xcon$^r_{axp}$ & xcon$_{cxp}$ & xcon$^r_{cxp}$ & $\text{LIME}$ & $\text{SHAP}$ & $\text{Anchor}$ \\ \midrule
DL & 2 & 1 & 1 & 2 & 3755 & 42555 & 3800 \\
BT & 80 & 82 & 97 & 151 & 98 & 6 & 351 \\
BNN & 196 & 152 & 179 & 199 & 15607 & 183058 & 11384 \\ \bottomrule
\end{tabular}
}
\end{table}

\begin{figure*}[h!]
	\centering
	\begin{subfigure}[b]{0.33\textwidth}
    \centering
    \includegraphics[width=\textwidth]{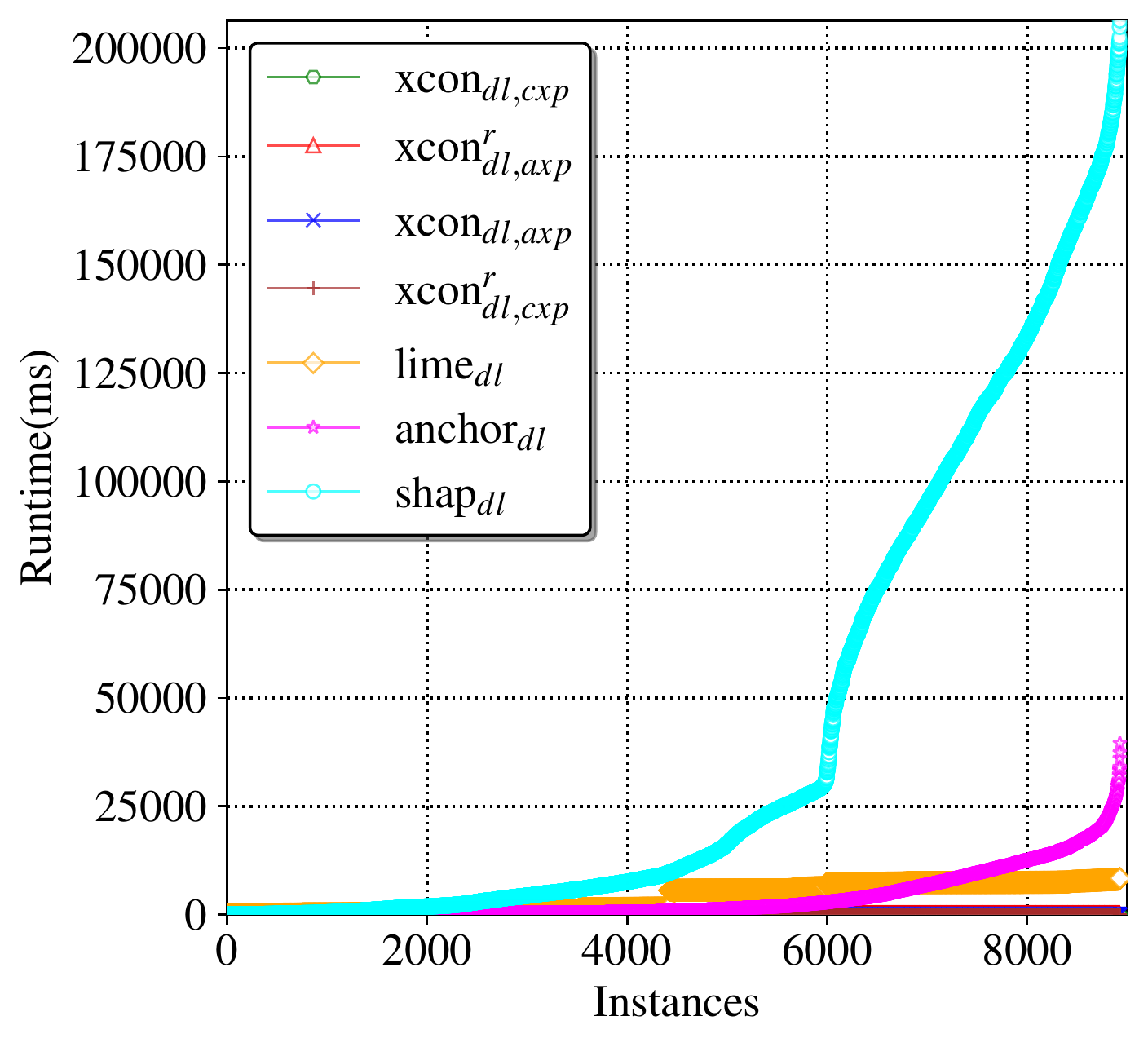}
    \caption{Runtime for DLs}
    \label{fig:dlhrtime}
  \end{subfigure}
  \hfill
  \begin{subfigure}[b]{0.315\textwidth}
    \centering
    \includegraphics[width=\textwidth]{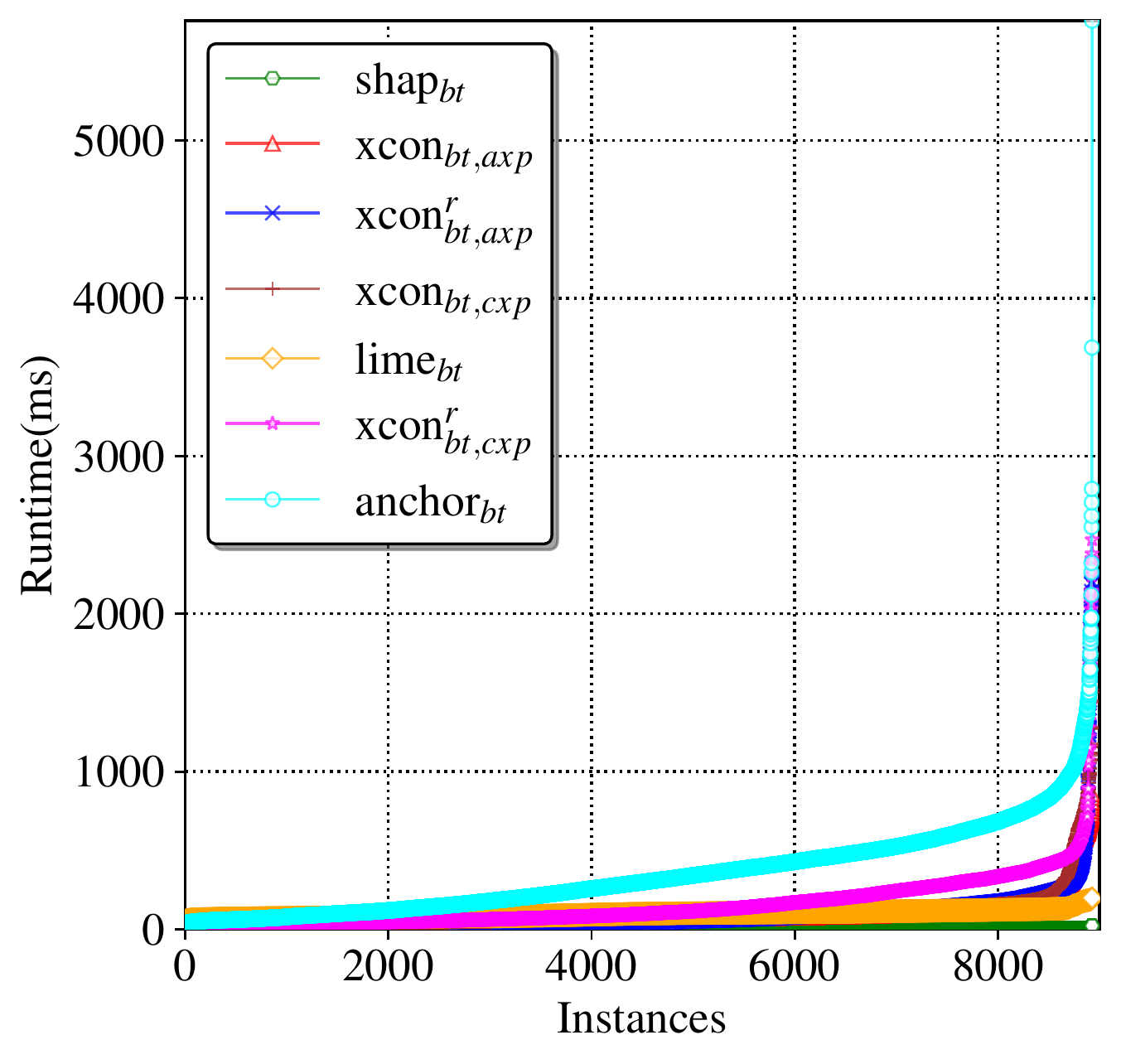}
    \caption{Runtime for BTs}
    \label{fig:bthrtime}
  \end{subfigure}
  \hfill
  \begin{subfigure}[b]{0.33\textwidth}
    \centering
    \includegraphics[width=\textwidth]{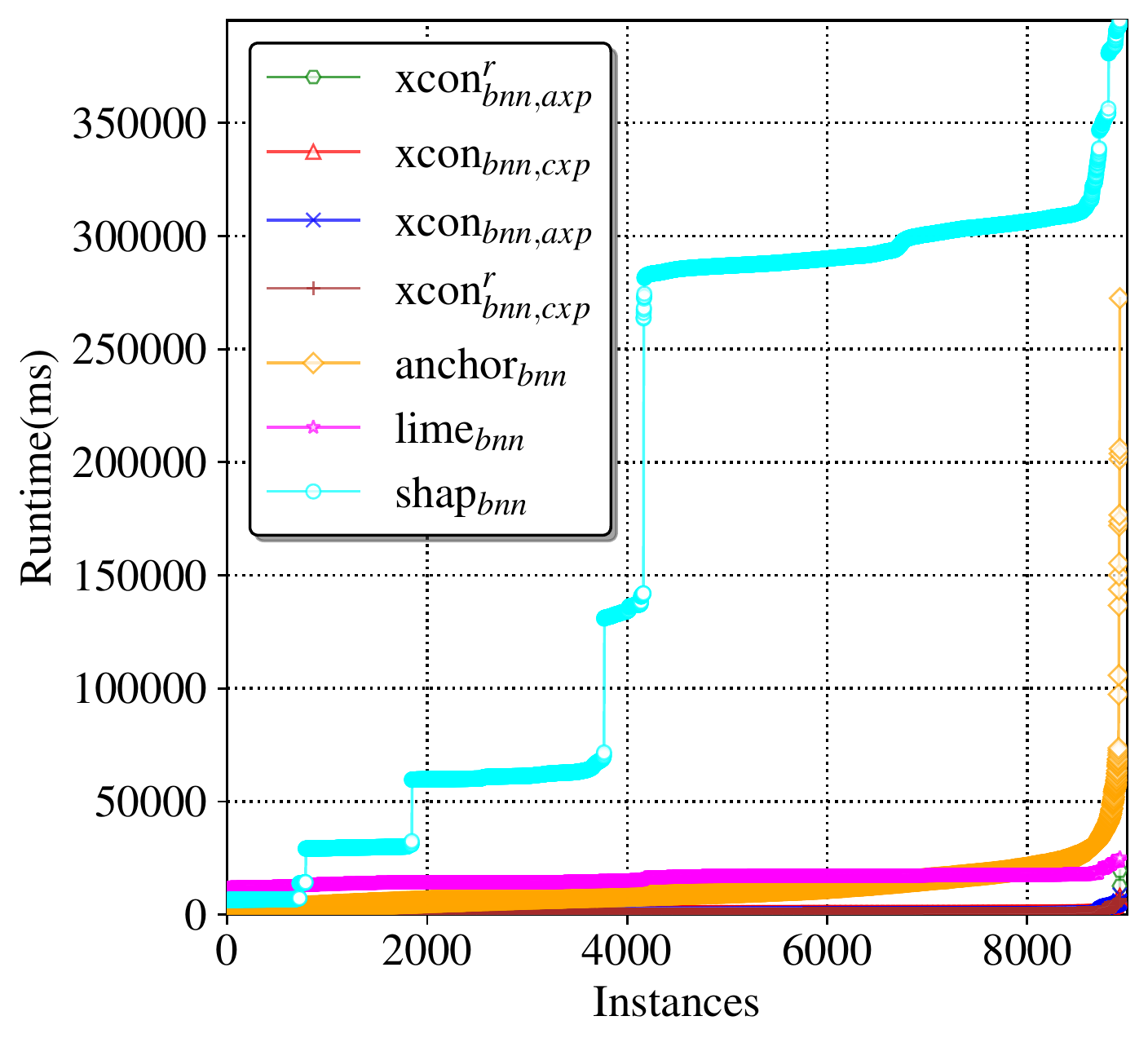}
    \caption{Runtime for BNNs}
    \label{fig:bnnhrtime}
  \end{subfigure}
  \caption{Runtime~(ms) of the considered explainers per explanation for DLs, BTs and BNNs.}
  \label{fig:hrtime}
\end{figure*}

\paragraph{Correctness.}
The average correctness of computed explanations\footnote{LIME/SHAP
assign weights to \emph{all the features}. We use only those whose
weight contributes to the decision made based on sign.} is shown in
Figure~\ref{fig:correct} and Table~\ref{tab:cor}.
Here, an explanation is said to be correct if it answers a
\emph{``why''} question and satisfies~\eqref{eq:axp} (or
\eqref{eq:axpc} in the presence of background knowledge) or it answers
a \emph{``why not''} question and satisfies~\eqref{eq:cxp} (or
\eqref{eq:cxpc} in the presence of background
knowledge).
The superscripted notation \emph{lime}$^r_{\ast}$,
\emph{shap}$^r_{\ast}$, and \emph{anchor}$^r_{\ast}$ is used to denote
the fact that background knowledge is applied when evaluating
correctness of the explanations produced by LIME, SHAP, and Anchor,
respectively.
%
%
Figure~\ref{fig:correct} and Table~\ref{tab:cor} show that the average correctness is higher when
background knowledge is applied as the number of features required in
a minimal correct explanation answering a \emph{``why''} question can
drop, which is demonstrated in Section~\ref{sec:resexp}.
However, heuristic approaches are not able to achieve 100\% correctness
in the majority of the datasets.
The best results are demonstrated by SHAP in both
Figure~\ref{fig:correctnobg} and Figure~\ref{fig:correctbg}.
SHAP's explanations for most of the datasets achieve 40\% correctness
when no background knowledge applied, while its correctness jumps to
80\% for the vast majority of datasets when background knowledge is
taken into account.
As of LIME and Anchor, without background knowledge, the correctness
of most of the explanations is less than 20\% for
\emph{anchor}$_{bt}$, \emph{lime}$_{bnn}$, \emph{anchor}$^r_{bnn}$,
and \emph{lime}$_{dl}$, but the correctness dramatically increases
when background knowledge is applied.
Figure~\ref{fig:correctbg} demonstrates that with background knowledge
the best correctness is achieved by SHAP, followed by Anchor, where
the major correctness for SHAP, Anchor, and LIME is more than 80\%,
60\% and 40\%, respectively.

Heuristic explainers consistently
demonstrate low correctness when no background knowledge is applied,
which confirms the earlier results of~\cite{inms-corr19,ignatiev-ijcai20}.
However, the situation changes dramatically when we apply the
background knowledge.
This is because some of the counterexamples invalidating heuristic
explanations are forbidden by the knowledge extracted.
Assuming that this knowledge is valid, these correctness results
better reflect the reality and so are more trustable.

\begin{figure*}[h!]
	\centering
	\begin{subfigure}[b]{0.49\textwidth}
    \centering
     \includegraphics[width=\textwidth]{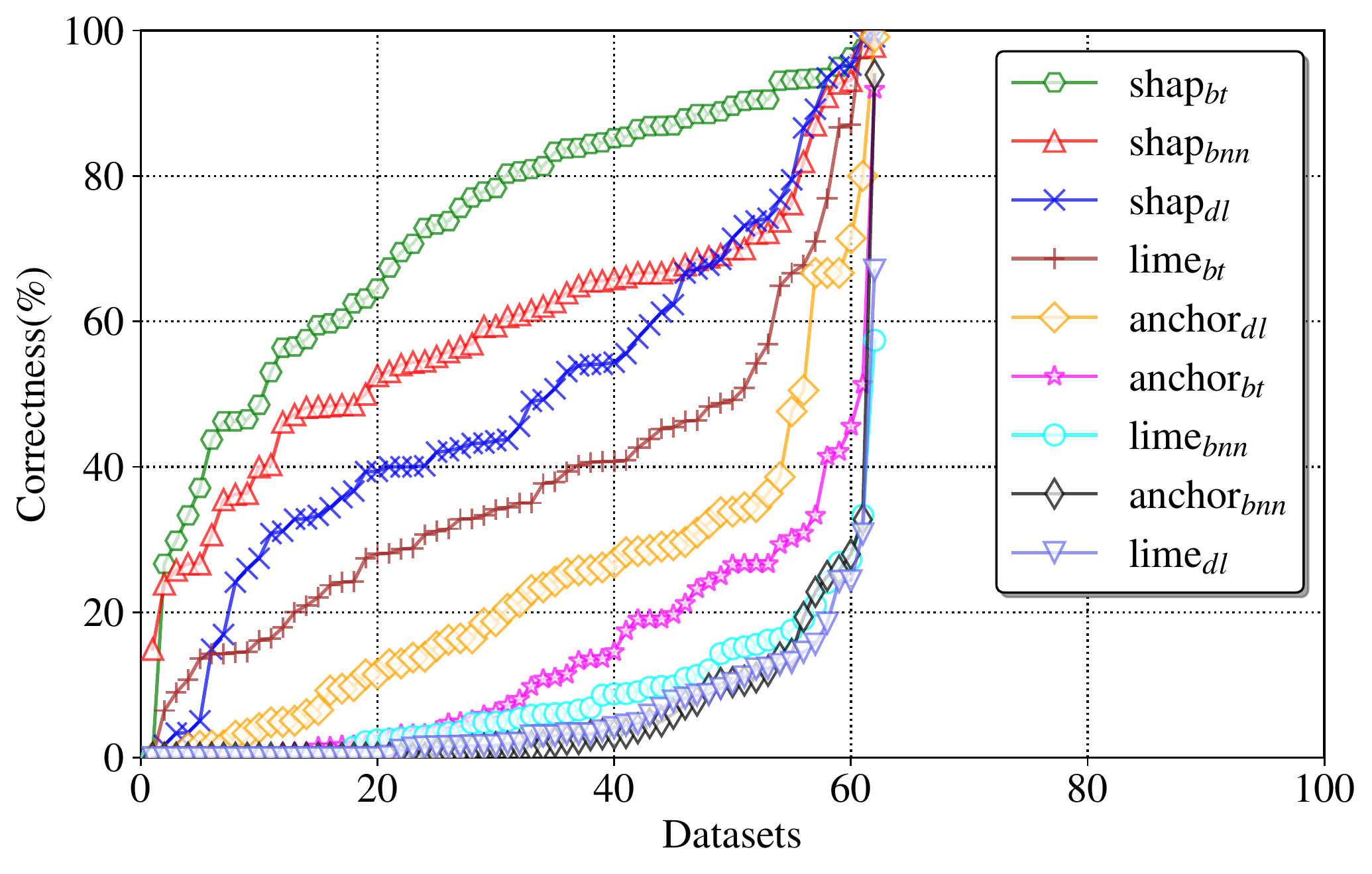}
    \caption{Correctness without background knowledge.}
    \label{fig:correctnobg}
 \end{subfigure}
  \hfill
  \begin{subfigure}[b]{0.49\textwidth}
    \centering
     \includegraphics[width=\textwidth]{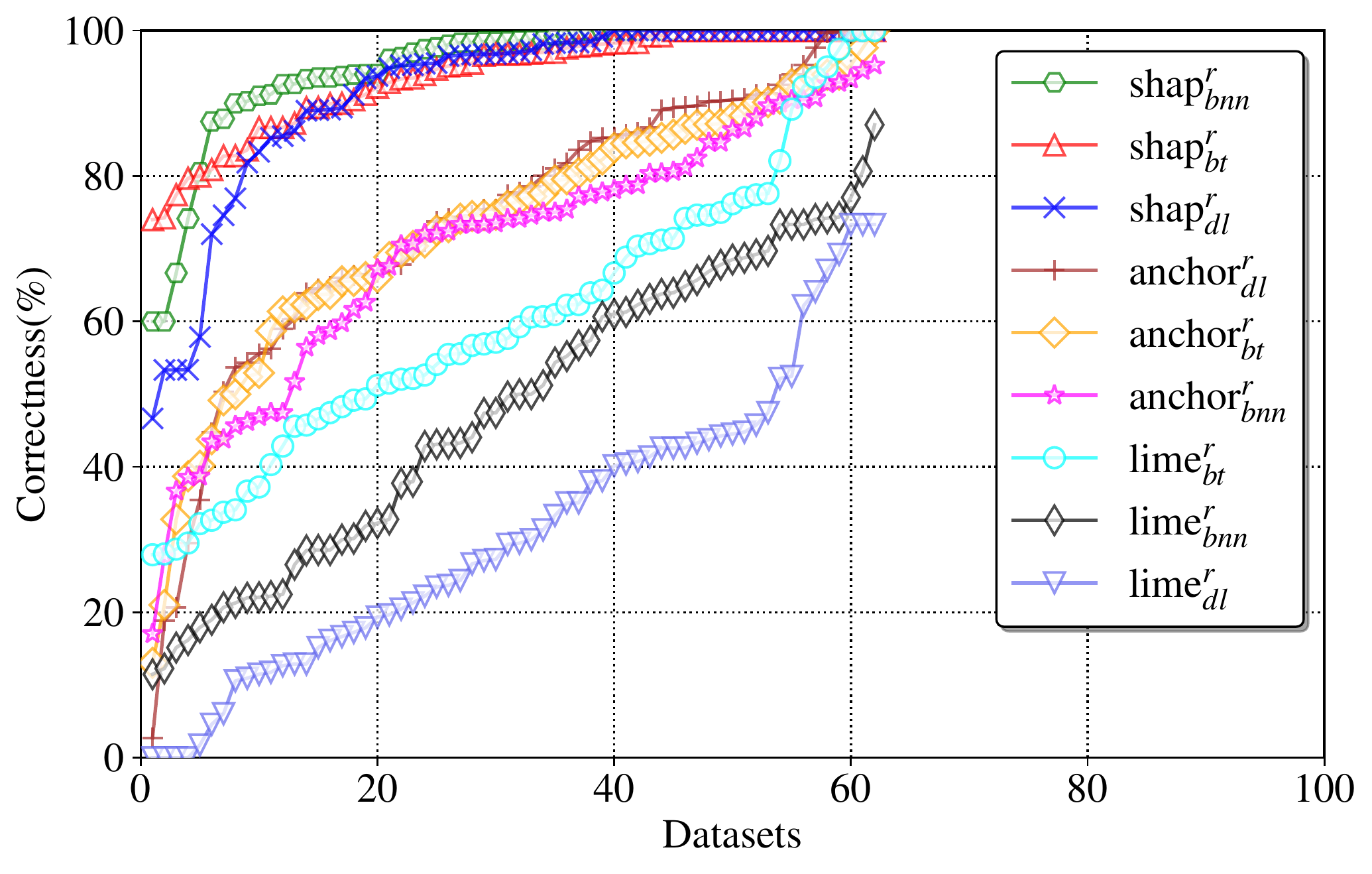}
    \caption{Correctness with background knowledge.}
    \label{fig:correctbg}
 \end{subfigure}
  \caption{Correctness of heuristic explanations.}
  \label{fig:correct}
\end{figure*}

\begin{table}[t!]
\centering
\caption{Average correctness of LIME, SHAP and Anchor.}
\label{tab:cor}
\scalebox{0.84}{
\begin{tabular}{cS[table-format=2.2]S[table-format=2.2]S[table-format=2.2]S[table-format=2.2]S[table-format=2.2]S[table-format=2.2]}\toprule
\multirow{3}{*}{Explainer} & \multicolumn{6}{c}{Correctness (\%)} \\ \cmidrule{2-7}
 & \multicolumn{3}{c}{Without knowledge} & \multicolumn{3}{c}{With knowledge}\\ \cmidrule{2-7}
 & $\text{DL}$ & BT & BNN & $\text{DL}$ & BT & BNN \\ \midrule
LIME & 6.06 & 38.26 & 8.2 & 31.06 & 60.63 & 47.88\\
SHAP & 49.47 & 72.89 & 58.89 & 91.72 & 93.75 & 95.0\\
Anchor & 24.03 & 13.85 & 6.57 & 73.85 & 73.0 & 70.1\\ \bottomrule
\end{tabular}
}
\end{table}

\paragraph{Explanation quality.}
Although explanation correctness dramatically increases when
background knowledge is used, large size of correct explanations can
render them uninterpretable.
In this experiment, we evaluate the size of correct explanations
computed by LIME, SHAP, and Anchor, and check how far those correct
explanations are from their subset-minimal counterparts.
Concretely, given a correct heuristic explanation computed either by
LIME, or SHAP, or Anchor, we apply the formal approach to reduce it
further, with or without background knowledge.
Then we contrast the size of correct explanations and their
corresponding size-minimal correct explanations for DL, BT, and BNN
models.
The comparison is detailed in the scatter plots of
Figure~\ref{fig:limesize}, \ref{fig:shapsize}, and
\ref{fig:anchorsize}.
As can be observed, a vast majority of correct explanations computed
by LIME, SHAP, and Anchor are not minimal.
Their size significantly exceeds the size of subset-minimally reduced
explanations.
Furthermore, the size difference increases when background knowledge
is available, which is in line with our earlier observations regarding
the AXp computation.

\begin{figure*}[h!]
	\centering
	\begin{subfigure}[b]{0.32\textwidth}
    \centering
    \includegraphics[width=\textwidth]{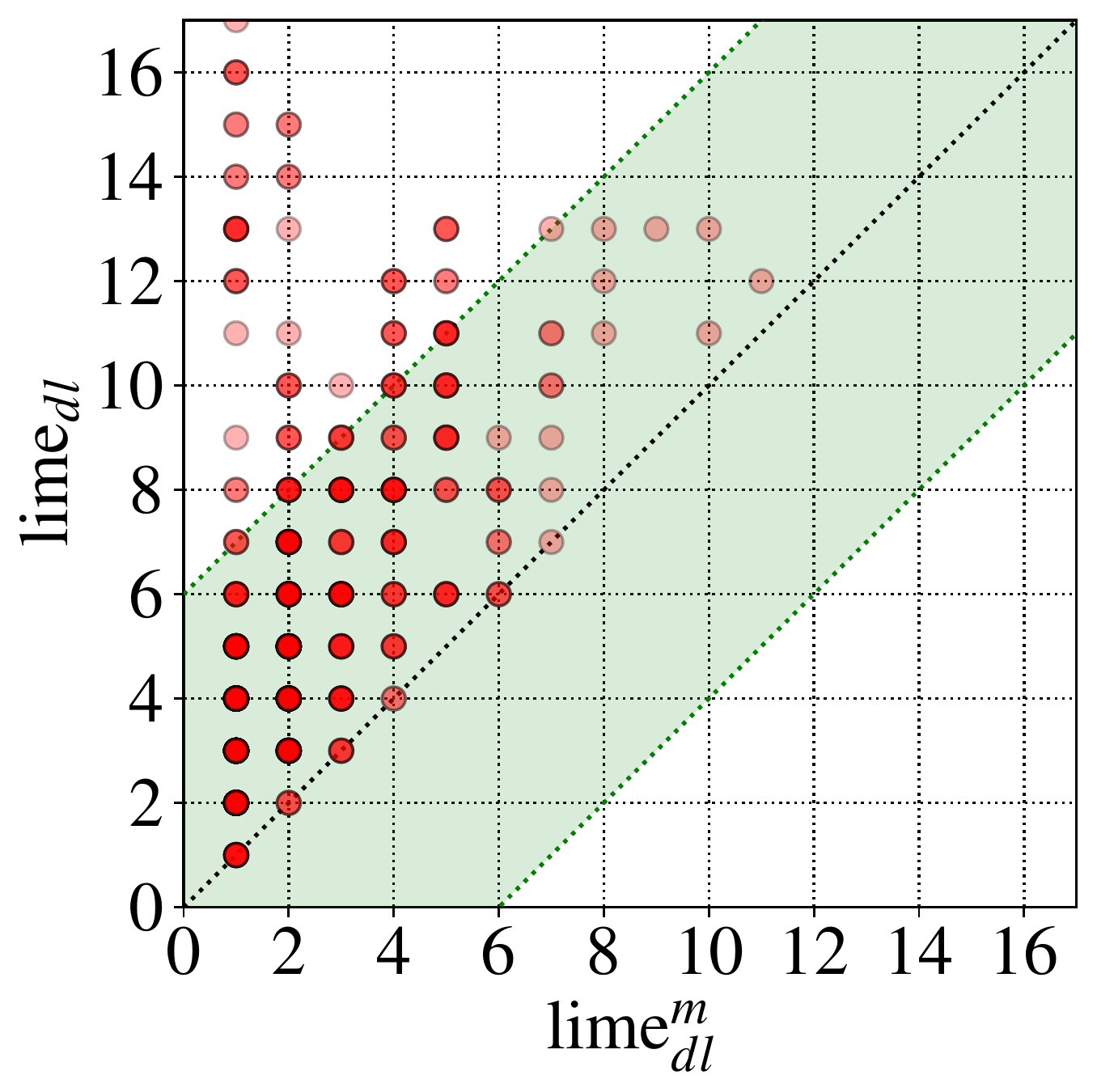}
    \caption{For DLs, without knowledge.}
    \label{fig:dllimenobg}
  \end{subfigure}
  \hfill
  \begin{subfigure}[b]{0.32\textwidth}
    \centering
    \includegraphics[width=\textwidth]{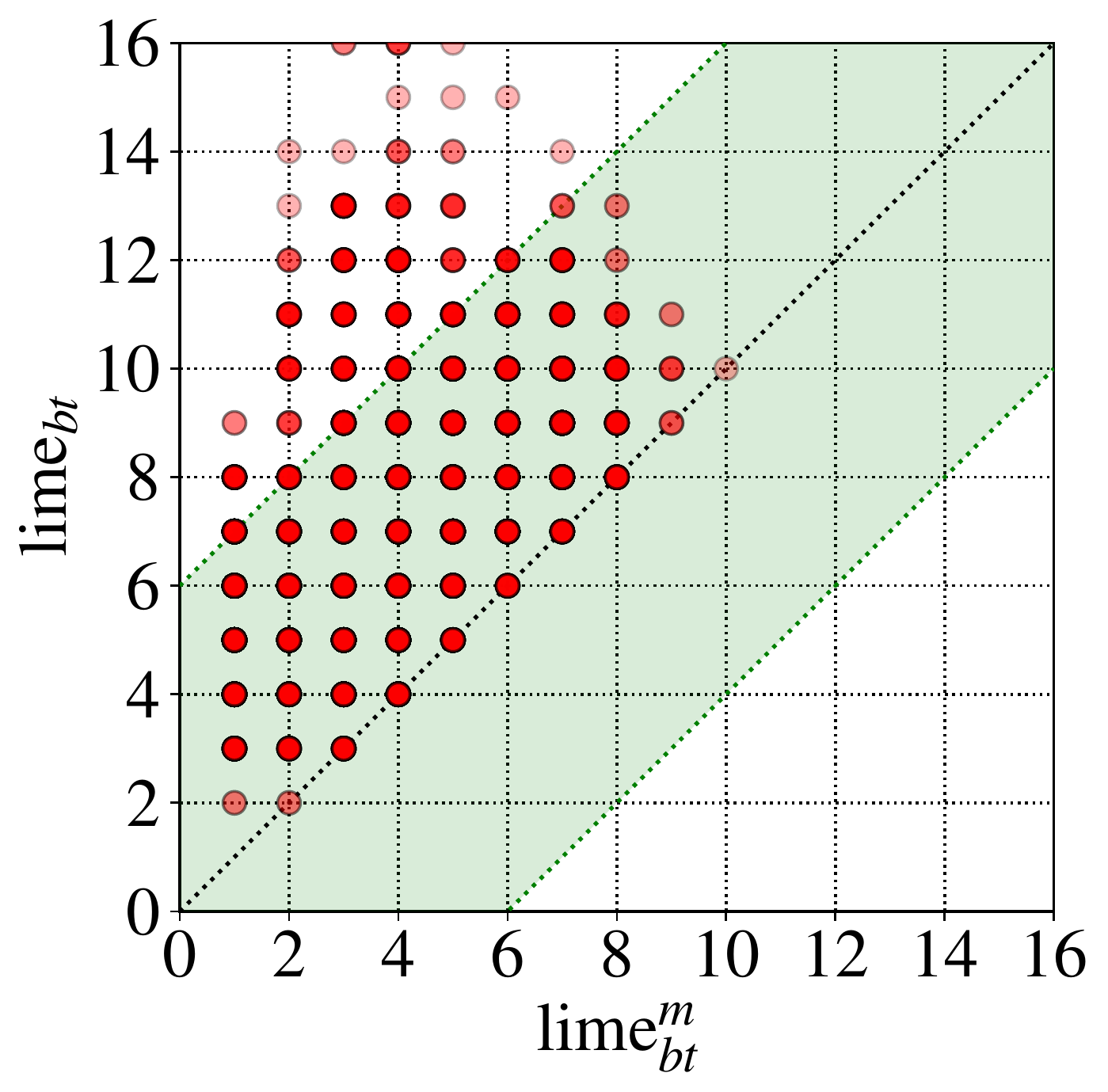}
    \caption{For BTs, without knowledge.}
    \label{fig:btlimenobg}
  \end{subfigure}
  \hfill
  \begin{subfigure}[b]{0.32\textwidth}
    \centering
    \includegraphics[width=\textwidth]{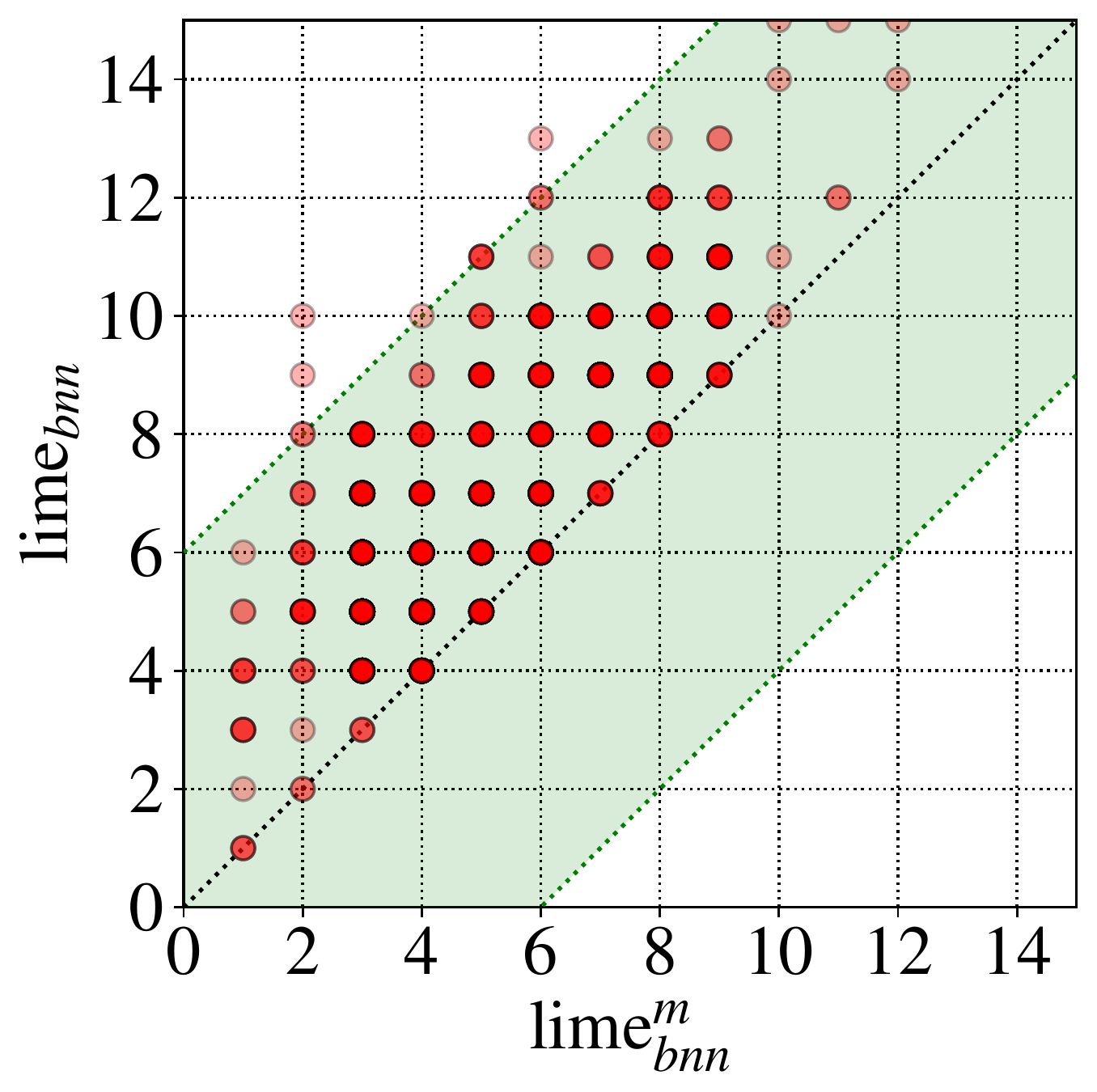}
    \caption{For BNNs, without knowledge.}
    \label{fig:bnnlimenobg}
  \end{subfigure}
  \hfill
  \begin{subfigure}[b]{0.32\textwidth}
    \centering
    \includegraphics[width=\textwidth]{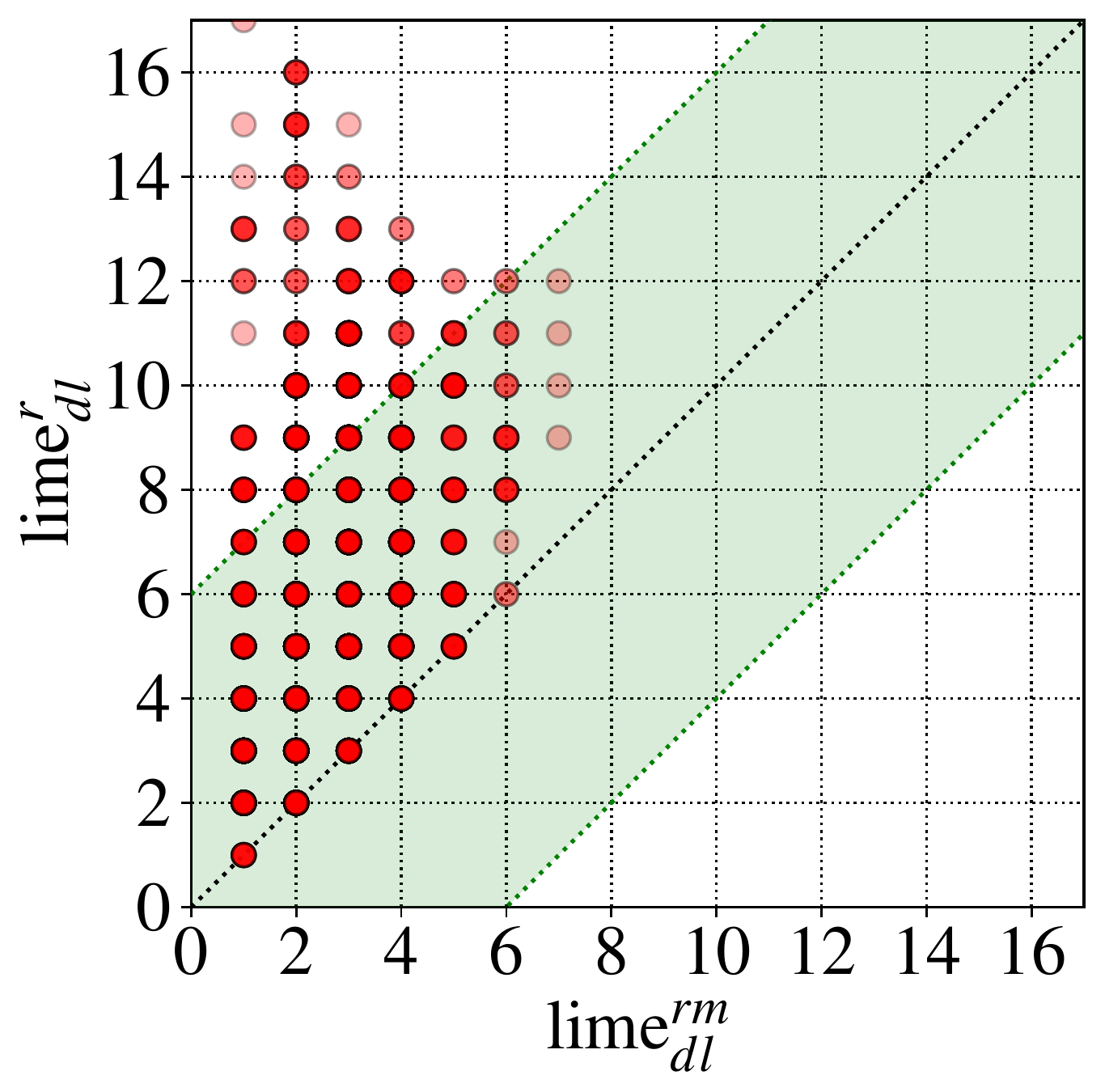}
    \caption{For DLs, with background knowledge.}
    \label{fig:dllimebg}
  \end{subfigure}
  \hfill
  \begin{subfigure}[b]{0.32\textwidth}
    \centering
    \includegraphics[width=\textwidth]{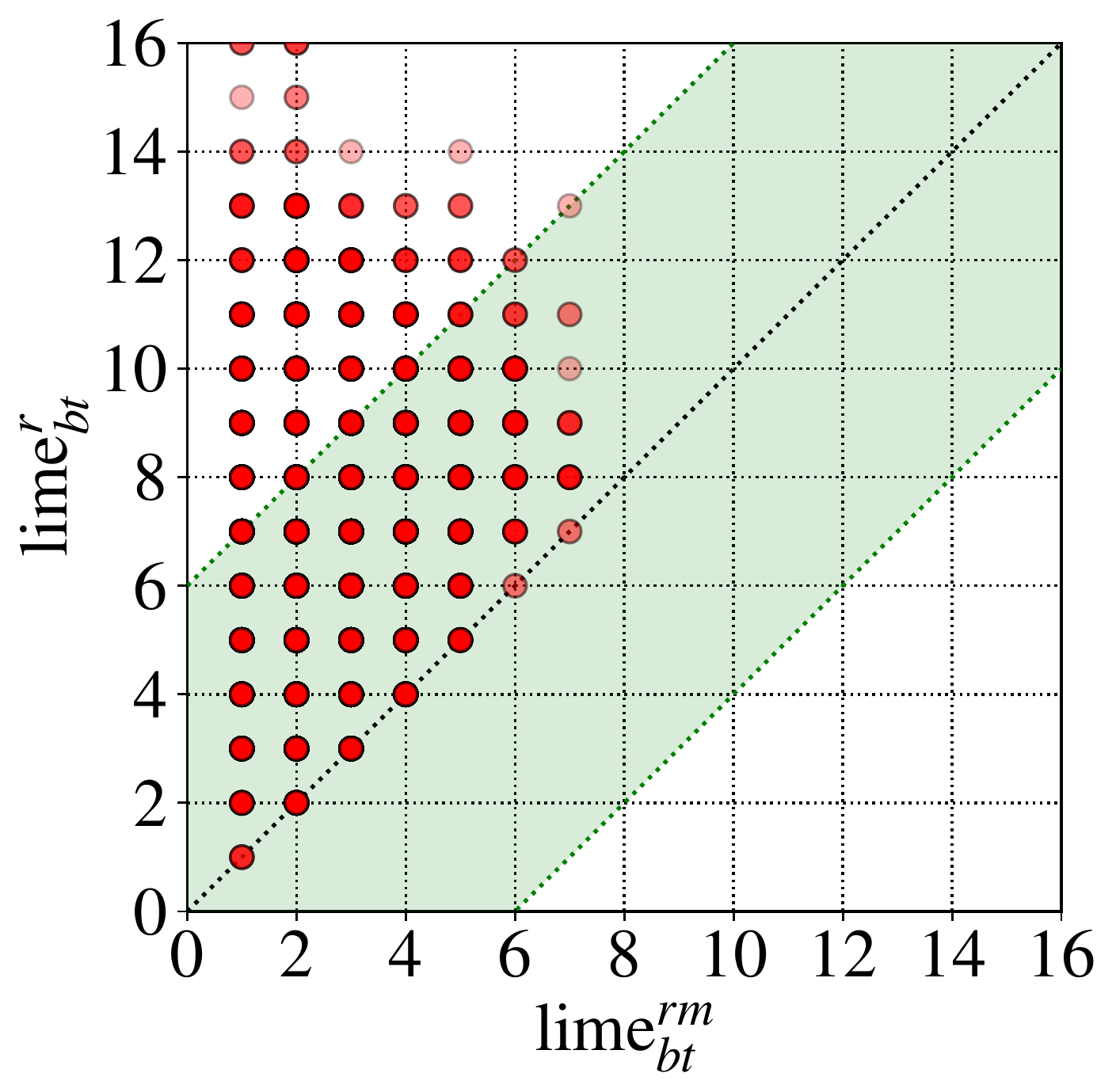}
    \caption{For BTs, with background knowledge.}
    \label{fig:btlimebg}
  \end{subfigure}
  \hfill
  \begin{subfigure}[b]{0.32\textwidth}
    \centering
    \includegraphics[width=\textwidth]{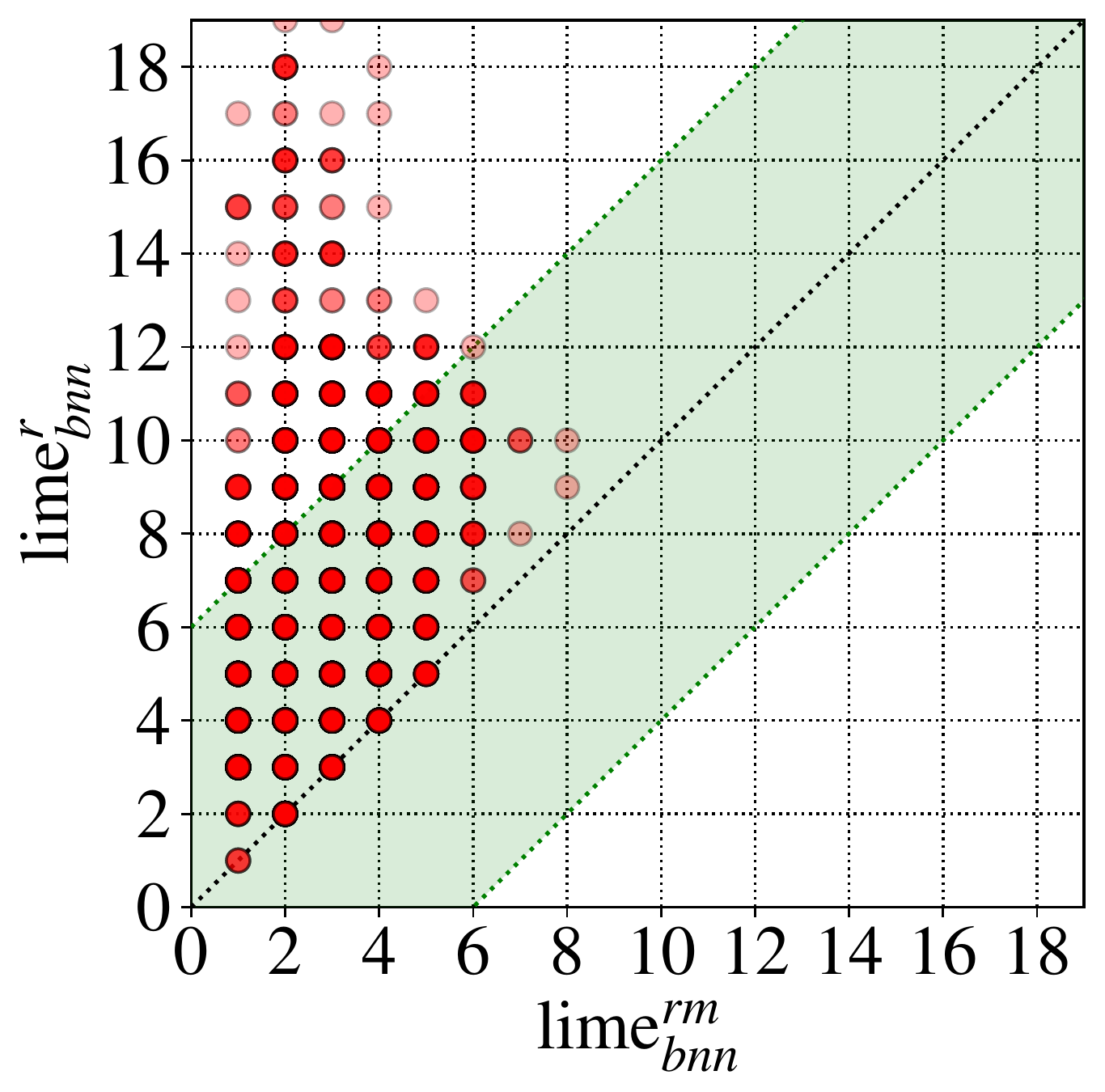}
    \caption{For BNNs, with background knowledge.}
    \label{fig:bnnlimebg}
  \end{subfigure}
  \caption{Size of LIME explanations for DLs, BTs and BNNs.}
  \label{fig:limesize}
\end{figure*}

\begin{figure*}[h!]
	\centering
	\begin{subfigure}[b]{0.32\textwidth}
    \centering
    \includegraphics[width=\textwidth]{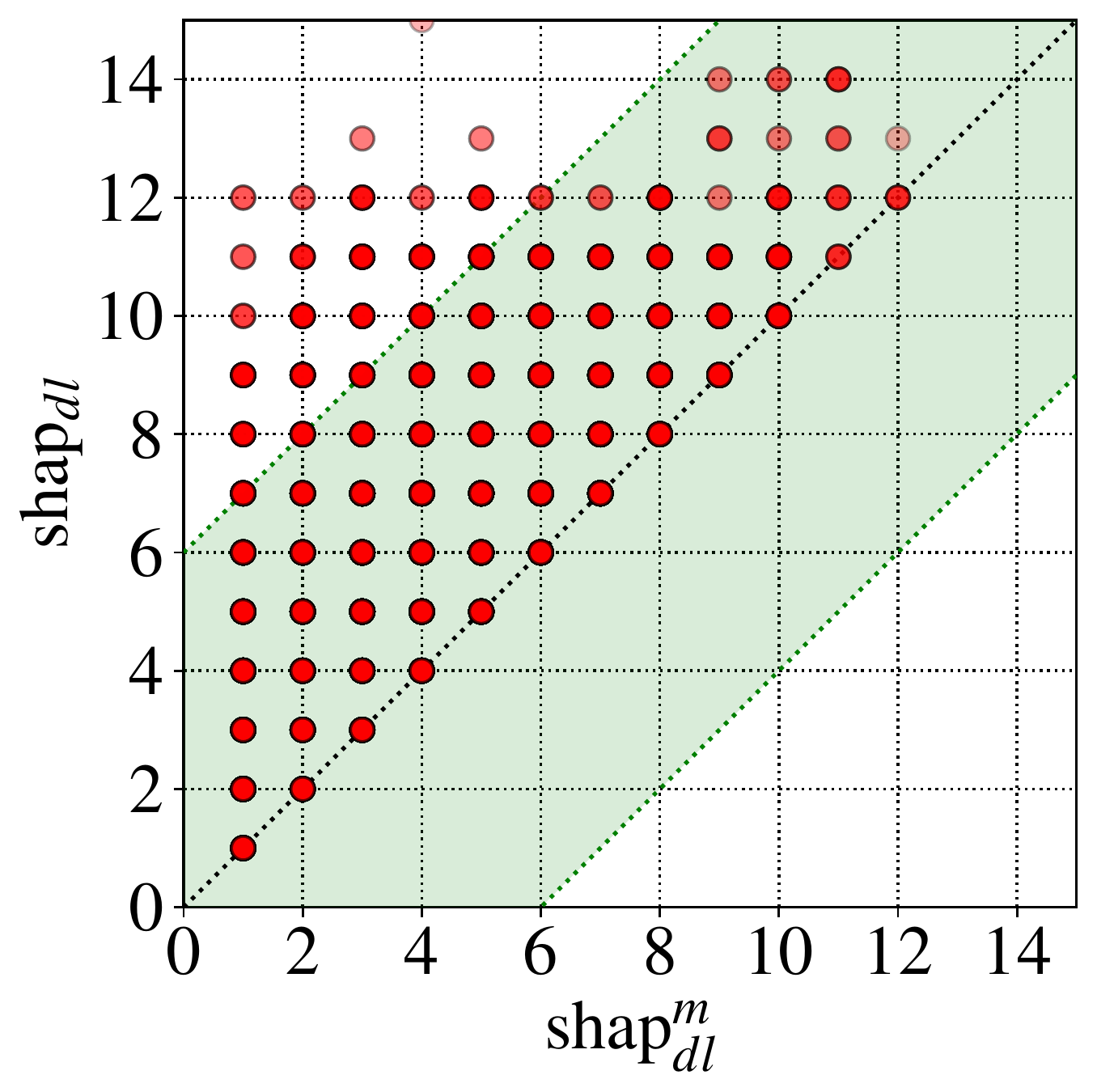}
    \caption{For DLs, without knowledge.}
    \label{fig:dlshapnobg}
  \end{subfigure}
  \hfill
  \begin{subfigure}[b]{0.32\textwidth}
    \centering
    \includegraphics[width=\textwidth]{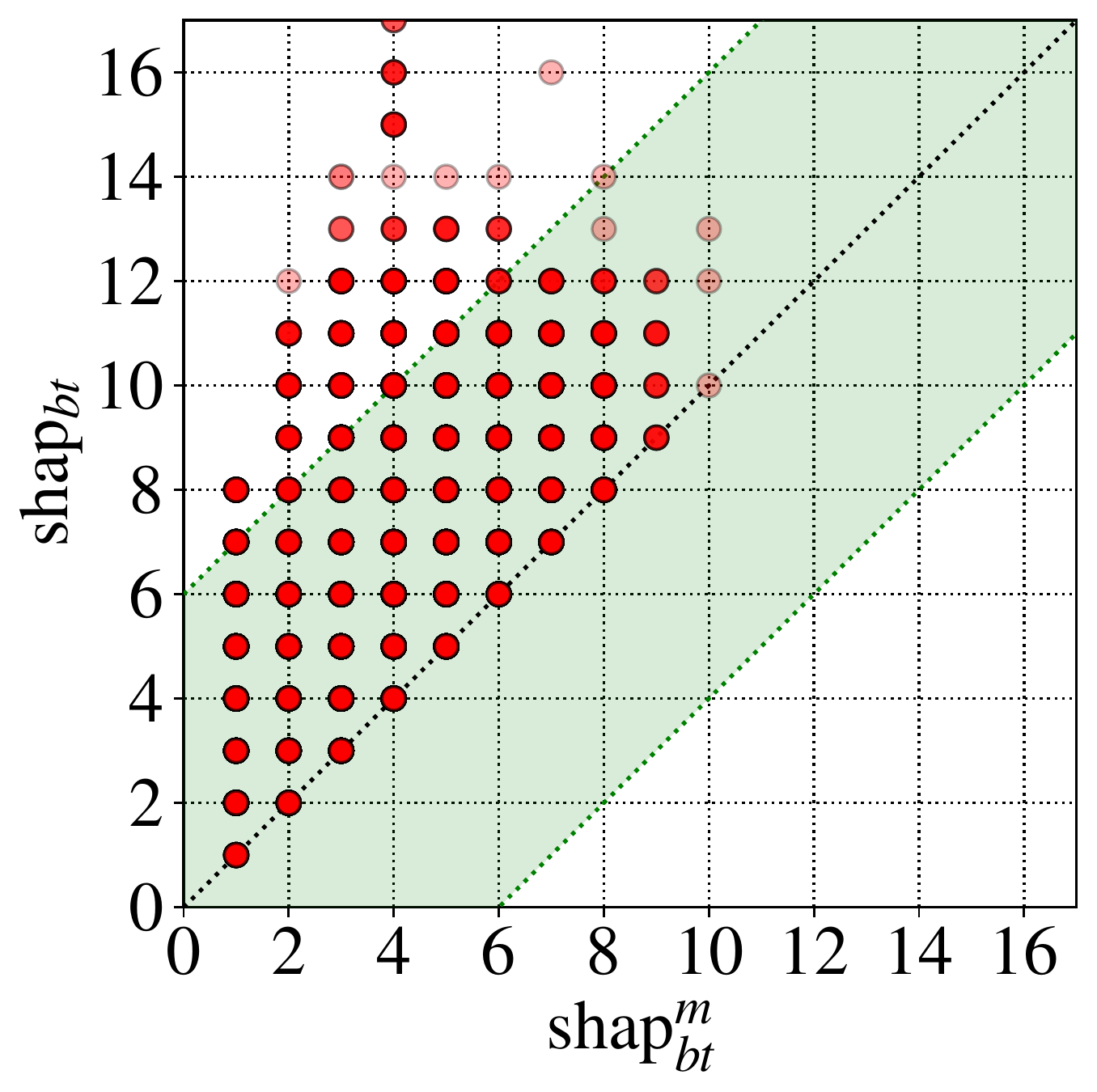}
    \caption{For BTs, without knowledge.}
    \label{fig:btshapnobg}
  \end{subfigure}
  \hfill
  \begin{subfigure}[b]{0.32\textwidth}
    \centering
    \includegraphics[width=\textwidth]{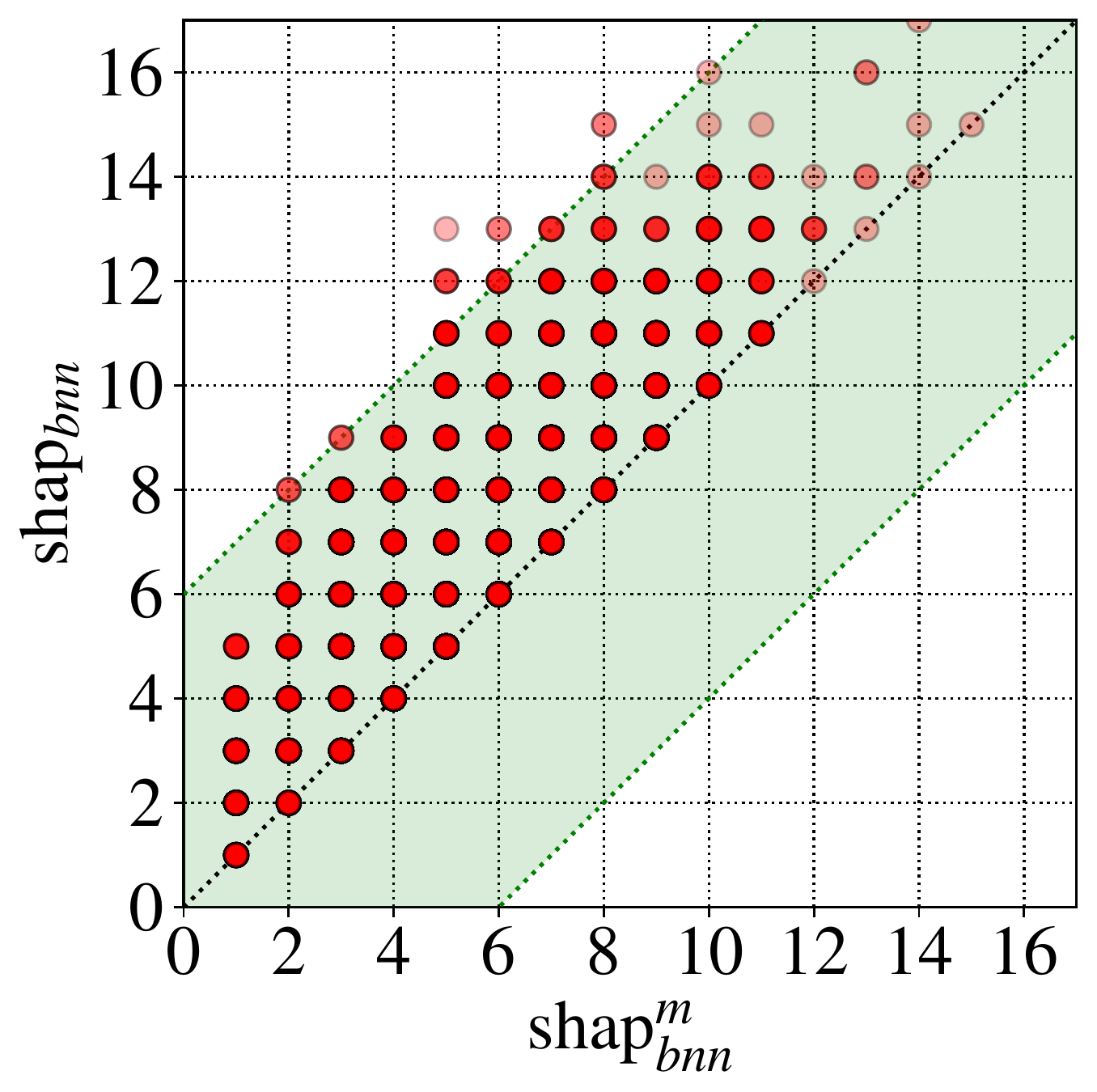}
    \caption{For BNNs, without knowledge.}
    \label{fig:bnnshapnobg}
  \end{subfigure}
  \hfill
  \begin{subfigure}[b]{0.32\textwidth}
    \centering
    \includegraphics[width=\textwidth]{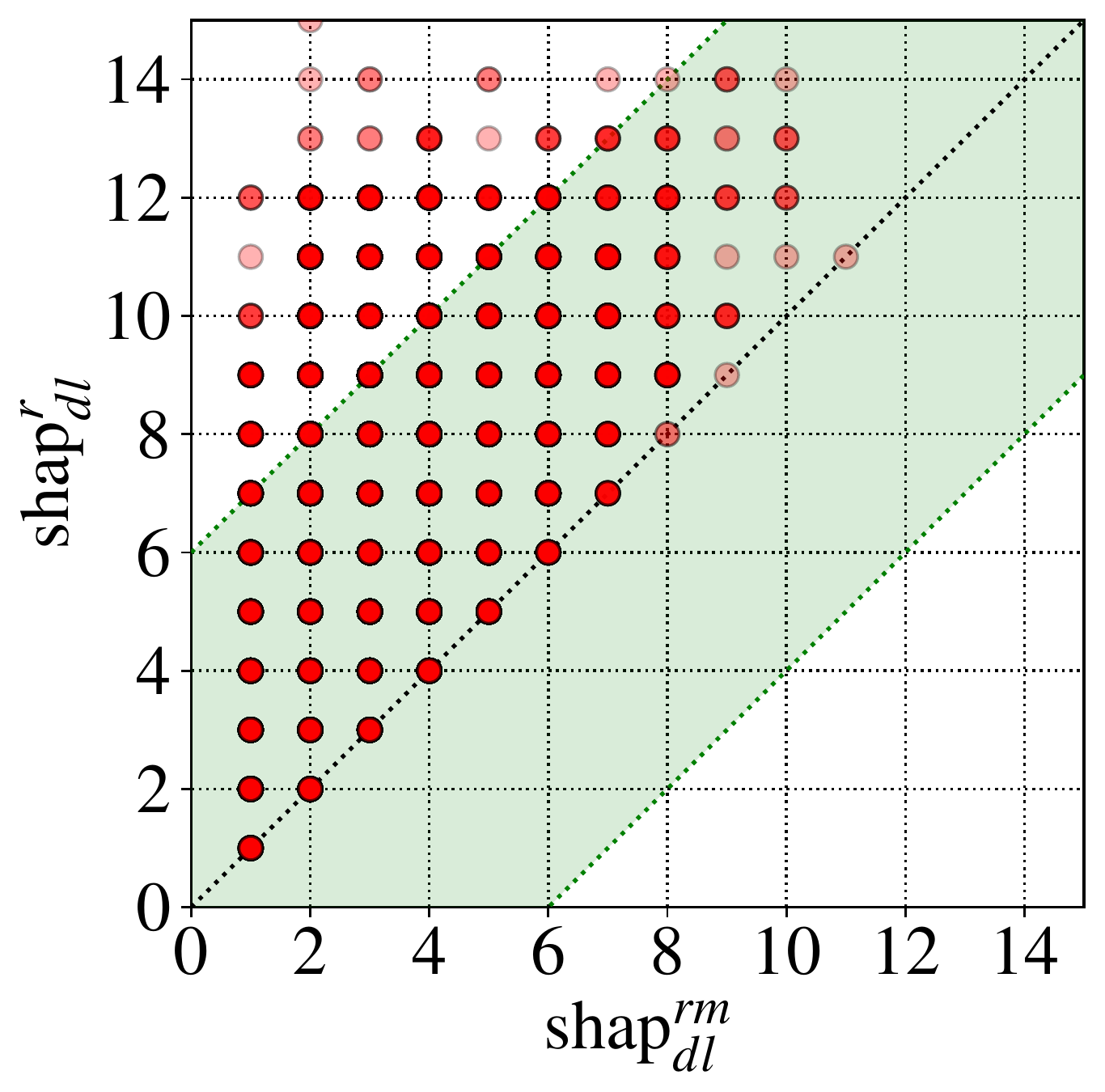}
    \caption{For DLs, with background knowledge.}
    \label{fig:dlshapbg}
  \end{subfigure}
  \hfill
  \begin{subfigure}[b]{0.32\textwidth}
    \centering
    \includegraphics[width=\textwidth]{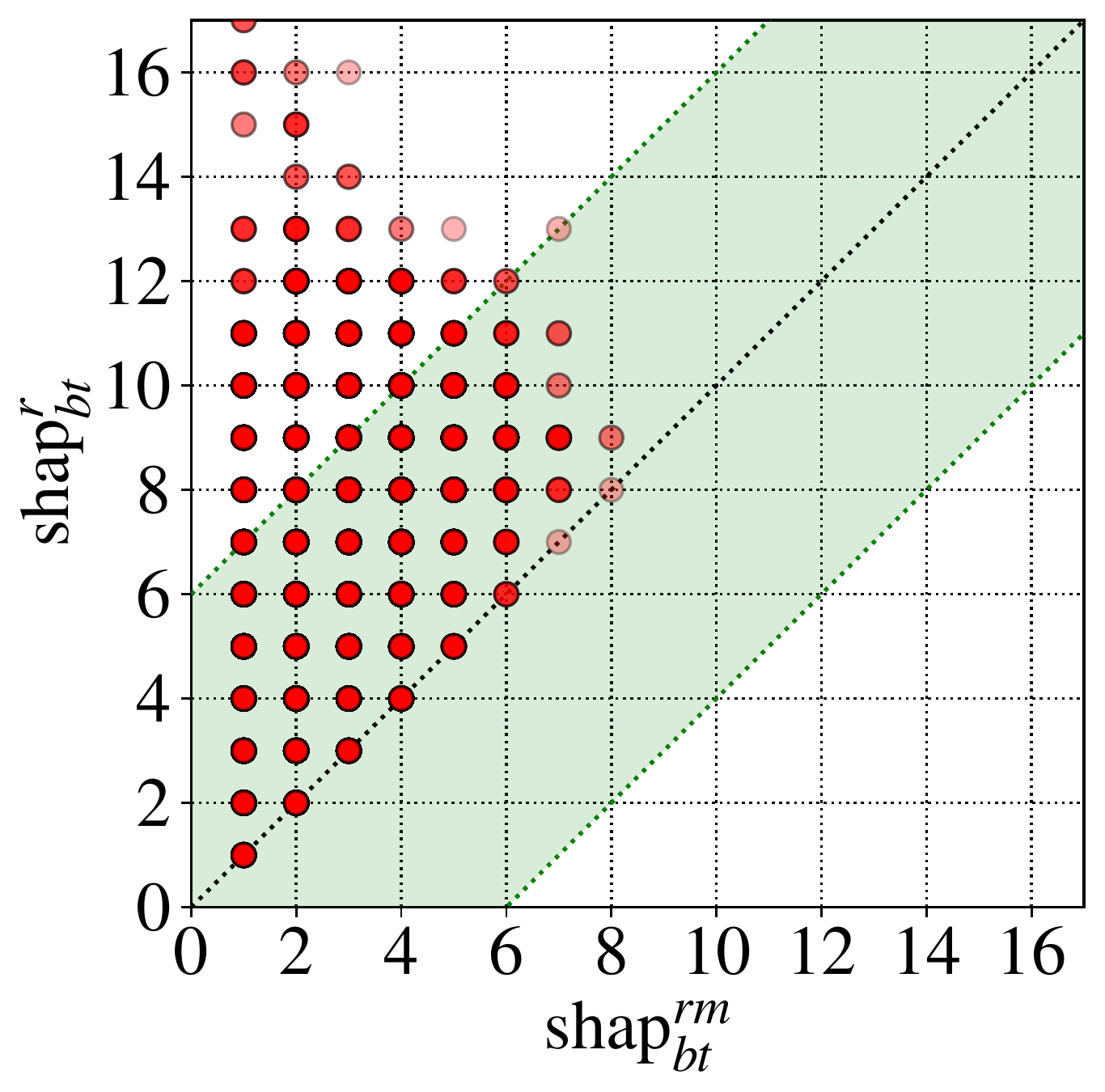}
    \caption{For BTs, with background knowledge.}
    \label{fig:btshapbg}
  \end{subfigure}
  \hfill
  \begin{subfigure}[b]{0.32\textwidth}
    \centering
    \includegraphics[width=\textwidth]{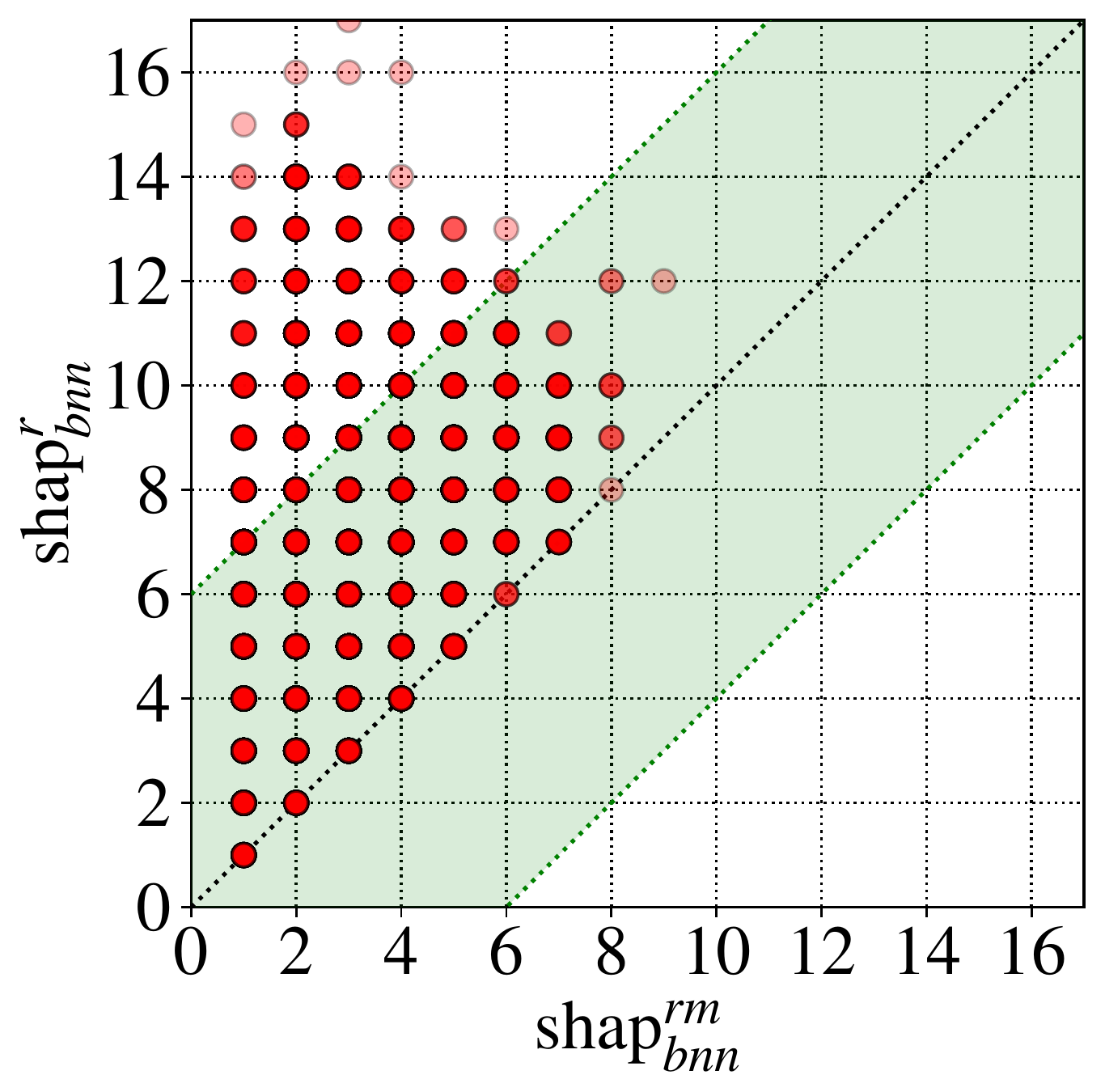}
    \caption{For BNNs, with background knowledge.}
    \label{fig:bnnshapbg}
  \end{subfigure}
  \caption{Size of SHAP explanations for DLs, BTs and BNNs.}
  \label{fig:shapsize}
\end{figure*}

\begin{figure*}[h!]
	\centering
	\begin{subfigure}[b]{0.32\textwidth}
    \centering
    \includegraphics[width=\textwidth]{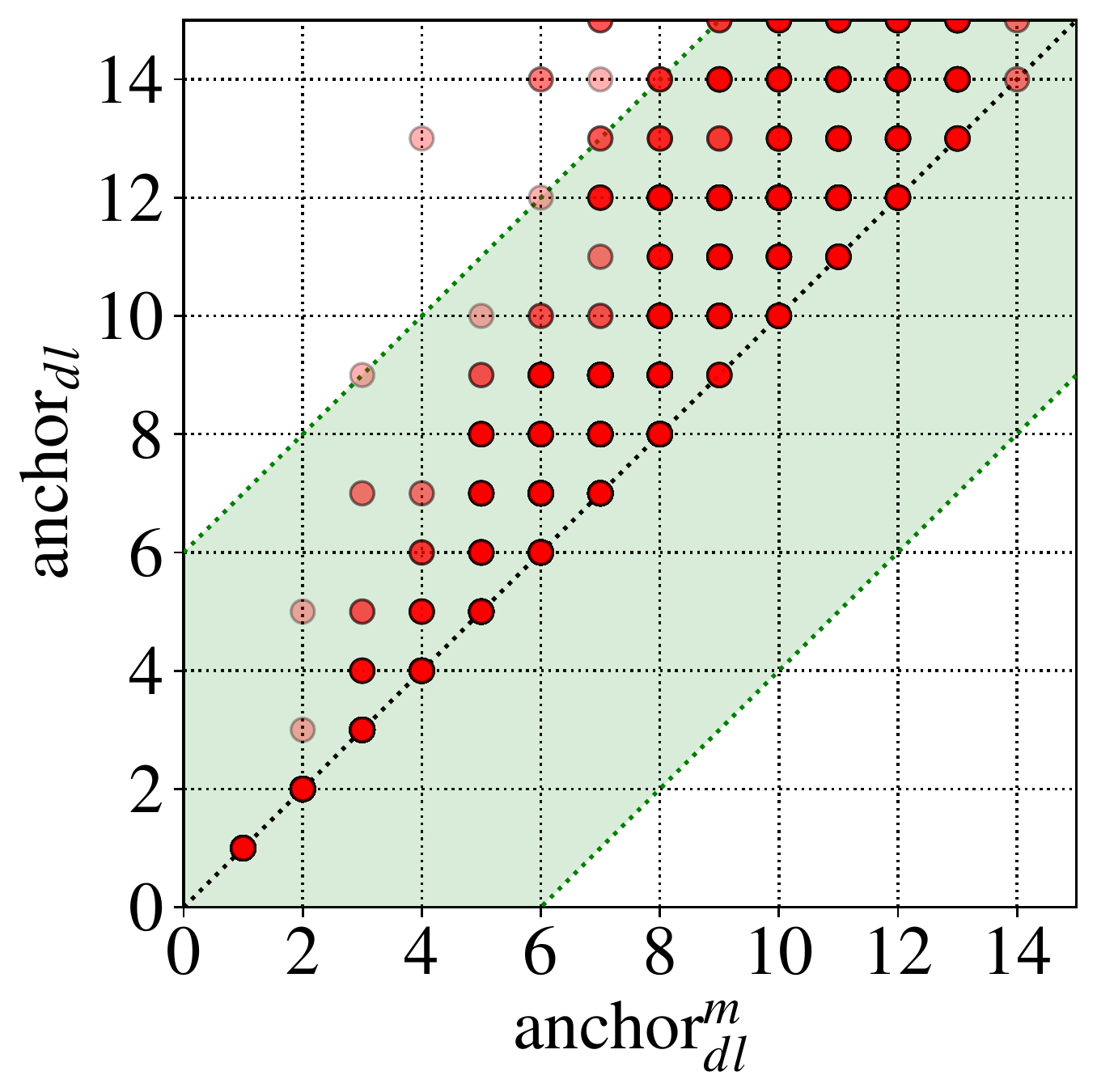}
    \caption{For DLs, without knowledge.}
    \label{fig:dlanchornobg}
  \end{subfigure}
  \hfill
  \begin{subfigure}[b]{0.32\textwidth}
    \centering
    \includegraphics[width=\textwidth]{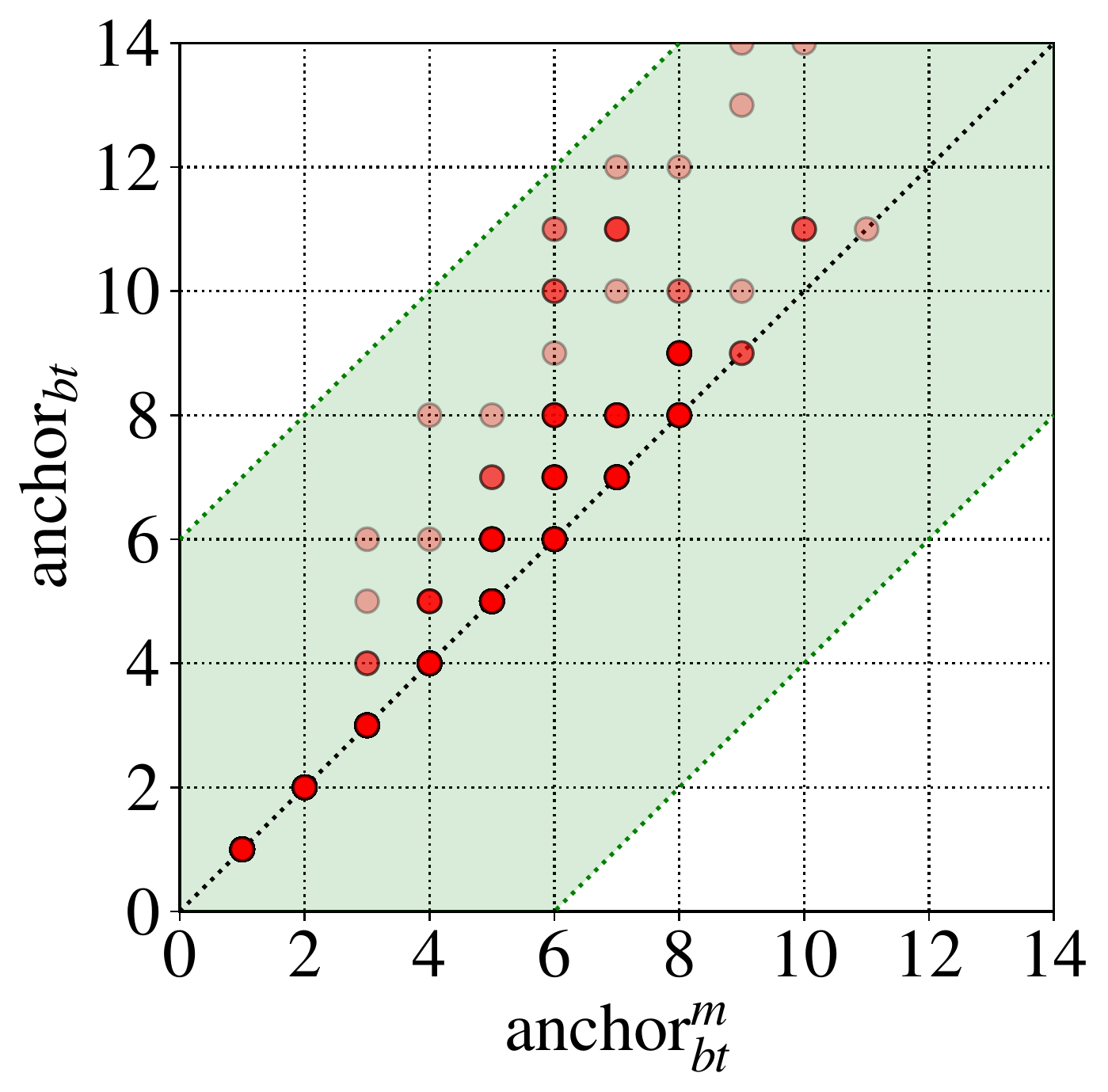}
    \caption{For BTs, without knowledge.}
    \label{fig:btanchornobg}
  \end{subfigure}
  \hfill
  \begin{subfigure}[b]{0.32\textwidth}
    \centering
    \includegraphics[width=\textwidth]{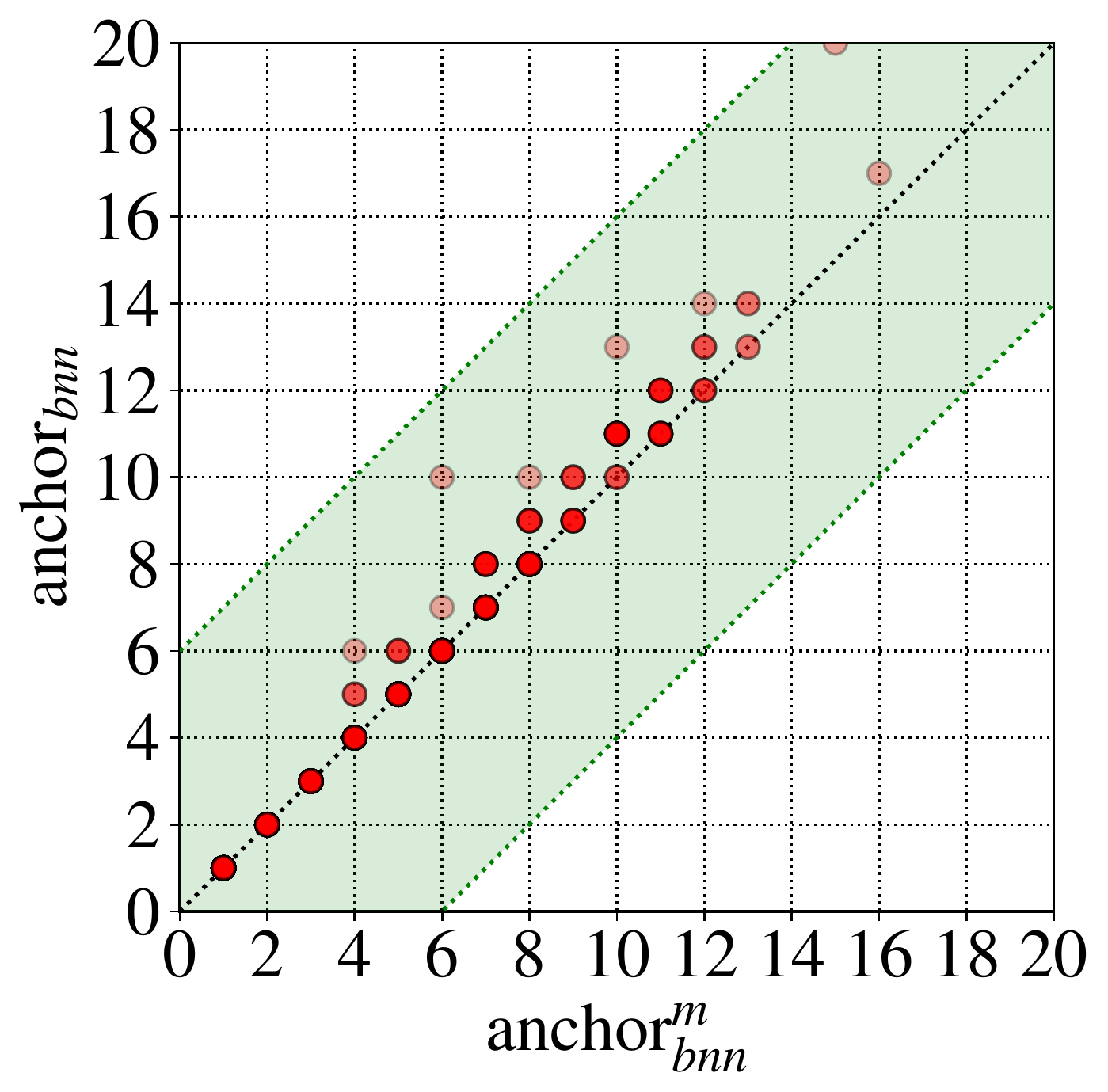}
    \caption{For BNNs, without knowledge.}
    \label{fig:bnnanchornobg}
  \end{subfigure}
  \hfill
  \begin{subfigure}[b]{0.32\textwidth}
    \centering
    \includegraphics[width=\textwidth]{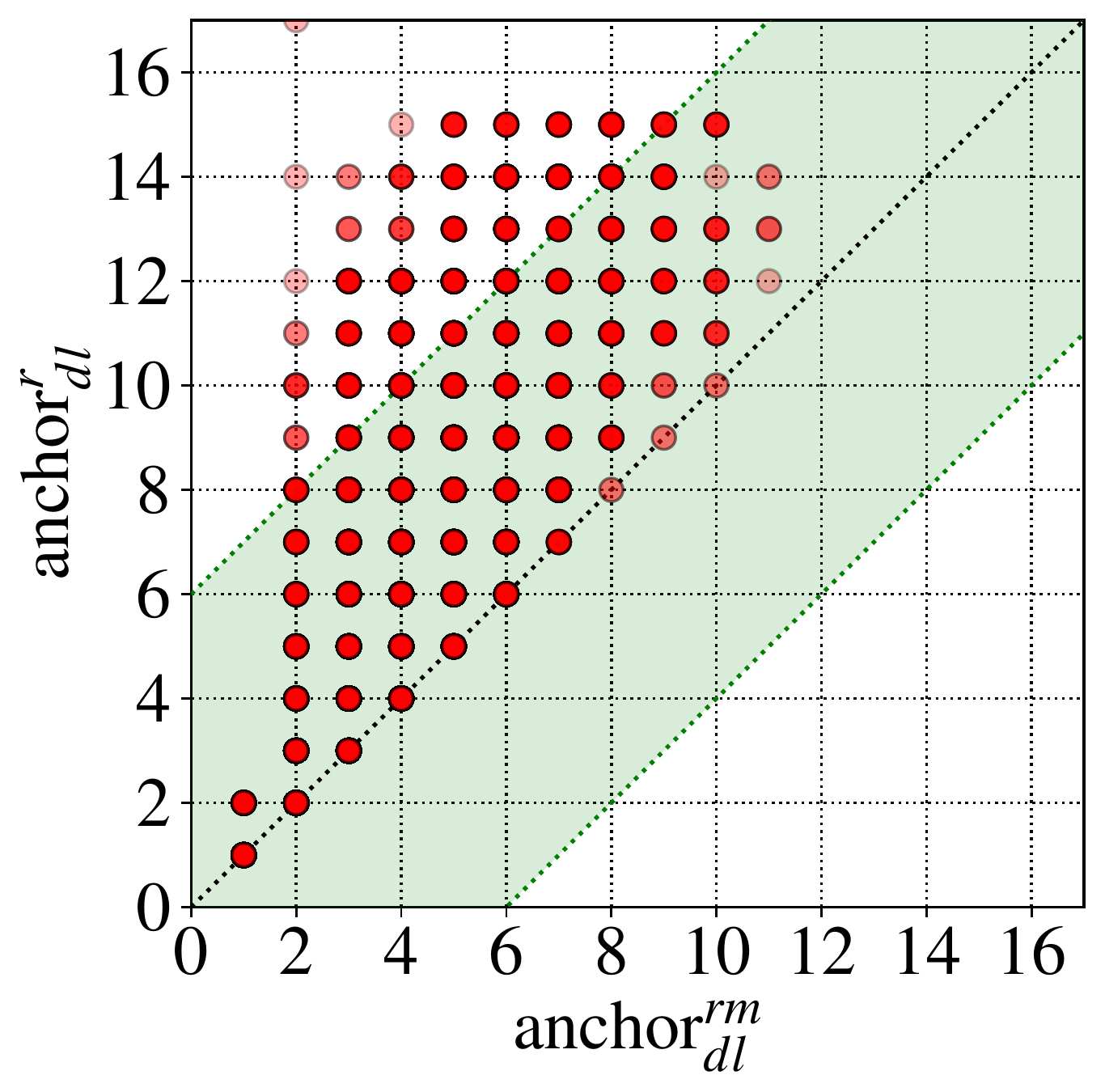}
    \caption{For DLs, with background knowledge.}
    \label{fig:dlanchorbg}
  \end{subfigure}
  \hfill
  \begin{subfigure}[b]{0.32\textwidth}
    \centering
    \includegraphics[width=\textwidth]{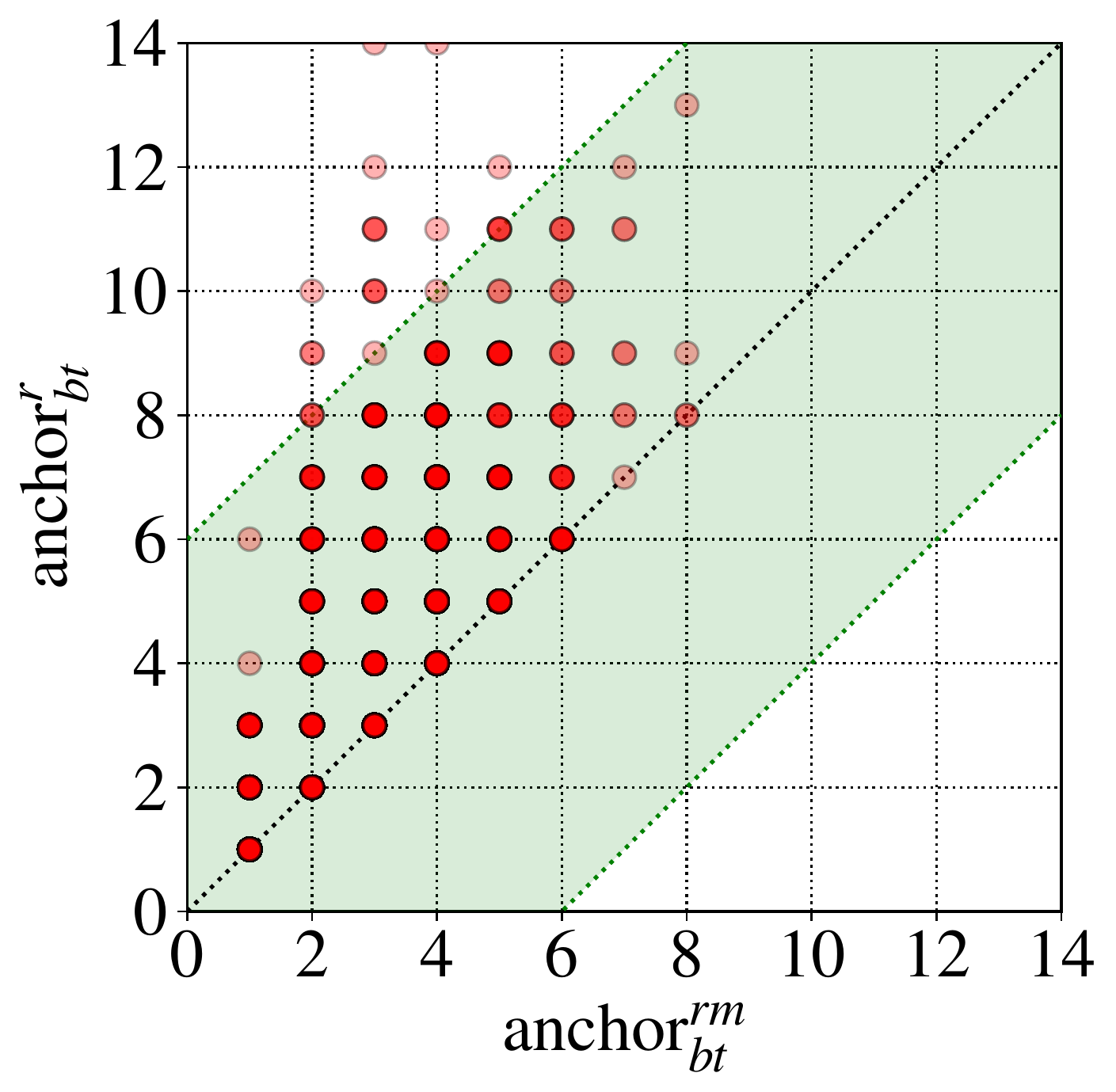}
    \caption{For BTs, with background knowledge.}
    \label{fig:btanchorbg}
  \end{subfigure}
  \hfill
  \begin{subfigure}[b]{0.32\textwidth}
    \centering
    \includegraphics[width=\textwidth]{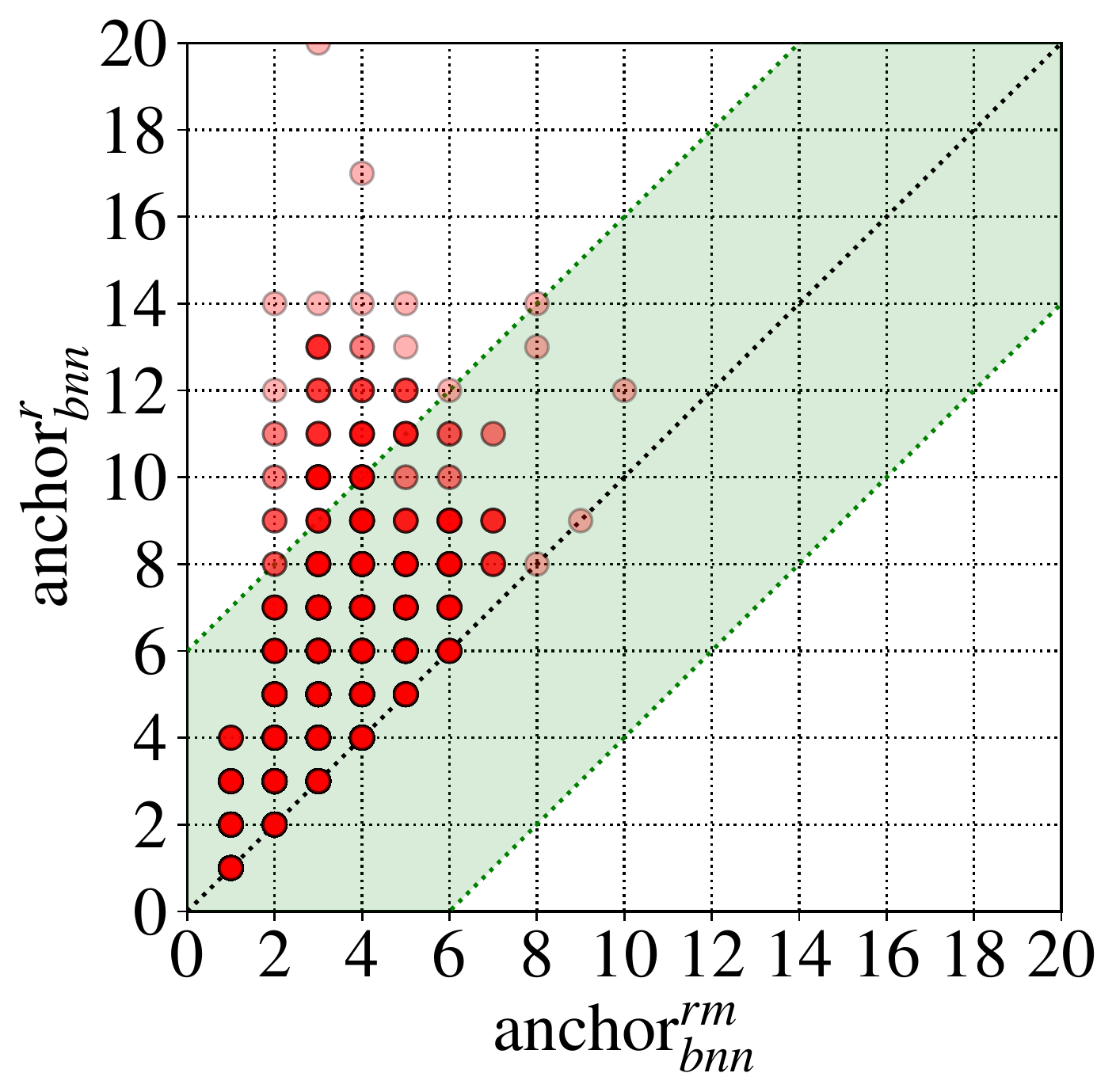}
    \caption{For BNNs, with background knowledge.}
    \label{fig:bnnanchorbg}
  \end{subfigure}
  \caption{Size of Anchor explanations for DLs, BTs and BNNs.}
  \label{fig:anchorsize}
\end{figure*}

\section{Related Work} \label{sec:relw}

Many methods for extracting knowledge from a dataset of
rules exist~\cite{hipp-kdd00,zhang-bk02,as-vlbd94,zaki-kdd97,ijrb-ijcai20,bbl-siam19}.
For use as background knowledge, we aim at very high confidence in the
rules, ideally they should be \emph{completely} valid for the feature
space.
While this is impossible to guarantee, the approach we define only
generates rules, which are valid for the entire data used for rule
generation. We can then use a validation or test set to remove rules
that are not supported by the larger data.
Traditional rule mining is more interested in rules with
high support and less focused on validity, although it can be
adapted to this case (see our experimental results above).
Although the explanation methods we apply in the presence of
background knowledge are agnostic about where it comes
from,
the motivation for our rule extraction method is twofold:
(1)~the rules are computed in a clausal form
and (2)~their high quality is guaranteed by the use of the strict
optimization problem formulation.


The most prominent approaches to post-hoc explainability are of
heuristic nature~\cite{guestrin-kdd16,lundberg-nips17,guestrin-aaai18}
and based on sampling in the vicinity of the instances being
explained.
None of these approaches can handle background knowledge.
Furthermore, they are susceptible to out-of-distribution
attacks~\cite{lakkaraju-aies20a}.
Approaches to formal explainability are represented by compilation of
classifiers into tractable representations~\cite{darwiche-ijcai18} and
reasoning-based explanation
approaches~\cite{inms-aaai19,msi-aaai22}.
%
The closest related work
is~\cite{rubin-aaai22}.
%
Based on compilation of a binary classifier into a binary decision
diagram (BDD), it
conjoins concocted background knowledge
to give
more succinct \emph{``why''} explanations for the classifier.
This approach is restricted to much smaller examples than we consider
here, since the compilation of a classifier into a BDD tends to explode with
the feature space.
The SAT and SMT based approaches to explanation we use are far more
scalable.
Finally, we consider a much broader class of classifiers, and also
examine \emph{``why not''} explanations and how they can be improved
by using background knowledge.



\section{Conclusions} \label{sec:conc}

Using background knowledge is highly advantageous for producing formal
explanations of machine learning models.
For abductive explanations (AXps), the use of background knowledge
substantially shortens explanations, making them easier to understand,
\emph{and} improves the speed of producing explanations.
For contrastive explanations (CXps), while the background knowledge
lengthens them and may increase the time required to generate an
explanation, the resulting explanations are far more useful since they
do not rely on the (usually unsupportable) assumption that all tuples
in the feature space are possible.
Furthermore and as this paper shows, background knowledge can be
applied in the context of heuristic explanations when an accurate
analysis of their correctness is required.

\ignore{
Several lines of future work can be envisioned.
First, possible ways to apply knowledge extraction approach in the
context of heuristic explainers can be investigated.
Second, the use of background knowledge may prove helpful in the
context of probabilistic generalisations of formal
explanations~\cite{kutyniok-jair21,wang-ijcai21}.
}

\section*{Acknowledgments}
This research was partially funded by the Australian Government
through the Australian Research Council Industrial Transformation
Training Centre in Optimisation Technologies, Integrated
Methodologies, and Applications (OPTIMA), Project ID IC200100009.
This work was also partially supported by the AI Interdisciplinary
Institute ANITI, funded by the French program ``Investing for the
Future -- PIA3'' under Grant agreement no.\ ANR-19-PI3A-0004, and by
the H2020-ICT38 project COALA ``Cognitive Assisted agile manufacturing
for a Labor force supported by trustworthy Artificial intelligence''.

\bibliography{paper}

\end{document}